\documentclass{article}

\usepackage{microtype}
\usepackage{graphicx}
\usepackage{subfigure}
\usepackage{booktabs} %
\usepackage{fullpage}

\usepackage[bookmarksopen=true,bookmarks=true,hidelinks,colorlinks=true,citecolor={blue}]{hyperref}
\usepackage{bookmark}

\usepackage{algorithm}
\usepackage{algorithmic}

\usepackage{amsmath}
\usepackage{amssymb}
\usepackage{mathtools}
\usepackage{amsthm}

\usepackage[capitalize]{cleveref}

\usepackage{xspace}
\usepackage{thmtools} 
\usepackage{thm-restate}
 
\allowdisplaybreaks

\theoremstyle{plain}
\newtheorem{theorem}{Theorem}[section]

\newtheorem{lemma}[theorem]{Lemma}

\theoremstyle{definition}

\newtheorem{fact}[theorem]{Fact}
\theoremstyle{remark}

\newtheorem{theoremi}{Theorem}[section]

\newtheorem{lemmai}[theorem]{Lemma}

\newtheorem{assumptioni}[theoremi]{Assumption}
\newtheorem{remarki}[theoremi]{Remark}

\usepackage[round,sectionbib]{natbib}
\newcommand{\ca}[1]{\mathcal{#1}}

\newcommand{\R}{\mathbb{R}}
\newcommand{\E}{\mathbb{E}}
\newcommand{\EE}[1]{\E\left[#1\right]}

\newcommand{\one}{\mathbf{1}}

\newcommand{\cF}{\mathcal{F}}
\newcommand{\cH}{\mathcal{H}}

\newcommand{\cI}{\mathcal{I}}
\newcommand{\cS}{\mathcal{S}}
\newcommand{\cR}{\mathcal{R}}
\newcommand{\cX}{\mathcal{X}}

\newcommand{\reg}{R}

\newcommand{\base}{\ensuremath{\tt Base}\xspace}

\newcommand{\exr}{\textsc{Exr}\xspace}
\newcommand{\ext}{\textsc{Ext}\xspace}
\newcommand{\Cinfo}{\textbf{C}_{\texttt{info}}}
\newcommand{\Cgood}{\textbf{C}_{\texttt{hit}}}
\newcommand{\Cbad}{\textbf{C}_{\texttt{miss}}}

\newcommand{\Mlogn}{\tilde{M}}

\newcommand{\mppr}{\lceil M\log N'\rceil}

\newcommand{\AdS}{\ensuremath{\tt AdSwitch}\xspace}
\newcommand{\BOG}{\texttt{OS-BASS}\xspace} %
\newcommand{\OptMoss}{\texttt{Opt-MOSS}\xspace}
\newcommand{\OGO}{${\tt OG^o}$\xspace}
\newcommand{\OG}{\texttt{OG}\xspace}
\newcommand{\UCB}{{\tt UCB}\xspace}
\newcommand{\MOSS}{{\tt MOSS}\xspace}
\newcommand{\pmml}{\texttt{E-BASS}\xspace}
\newcommand{\greedy}{\texttt{G-BASS}\xspace}
\newcommand{\bass}{\texttt{BASS}\xspace}

\newcommand{\Rcov}{R_{\text{coverage}}}

\newcommand{\A}{\mathcal{A}}

\usepackage{upgreek}

\newcommand{\gammaOne}{ \left(\frac{\Mlogn K\log K}{N'}\right)^{1/3}}

\newcommand{\gammaOnen}{\left(\frac{\Mlogn K\log K}{n}\right)^{1/3}}
\newcommand{\gammaOnenprime}{\left(\frac{\Mlogn K\log K}{n}\right)^{1/3}}

\newcommand{\Nthresh}{\frac{T^{3/5} (\log K)^{2/5}}{\Mlogn^{3/5}}}

\newcommand{\NoGapSmallt}{without \cref{assumption:identification} (no minimum gap and small task length)}

\usepackage{xcolor}

\newcommand{\setupOne}{(N,\tau,K,M)=(500,4500,30,10)\xspace}
\newcommand{\setupTwo}{(N,\tau,K,M)=(500,450,30,10)\xspace}

\DeclareMathOperator*{\argmin}{argmin}

\usepackage[textsize=tiny, disable]{todonotes}

\PassOptionsToPackage{usenames,dvipsnames}{xcolor}

\newcommand{\todoa}[2][]{\xspace\todo[size=\scriptsize,color=green!20!white,#1]{A: #2}\xspace}
\newcommand{\jtodo}[2][]{\xspace\todo[color=blue!5,size=\scriptsize,#1]{J: #2}\xspace}

\usepackage[disable]{todonotes}

\PassOptionsToPackage{usenames,dvipsnames}{xcolor}

\title{\bf
Non-stationary Bandits and Meta-Learning with a Small Set of Optimal Arms
}
\author{MohammadJavad Azizi \\ 
{\small University of Southern California}\\
{\small\tt azizim@usc.edu} 
\and
Thang Duong \\ 
{\small VinUniversity, VinAI Research}
\\{\small\tt v.thangdn3@vinai.io}
\and
Yasin Abbasi-Yadkori \\ 
{\small DeepMind}
\\{\small\tt yadkori@deepmind.com} 
\and
Andr\'{a}s Gy\"{o}rgy \\ 
{\small DeepMind}
\\{\small\tt agyorgy@deepmind.com} 
\and 
Claire Vernade \\ 
{\small DeepMind}
\\{\small\tt vernade@deepmind.com}
\and 
Mohammad Ghavamzadeh\\
{\small Google Research}
\\{\small\tt ghavamza@google.com}
}
\date{}

\begin{document}

\maketitle

\begin{abstract}
  We study a sequential decision problem where the learner faces a sequence of $K$-armed bandit tasks. %
  The task boundaries might be known (the bandit meta-learning setting), or unknown (the non-stationary bandit setting). For a given integer $M\le K$, the learner aims to compete with the best subset of arms of size $M$.  
  We design an algorithm based on a reduction to bandit submodular maximization, and show that, for $T$ rounds comprised of $N$ tasks, in the regime of large number of tasks and small number of optimal arms $M$, its regret in both settings is smaller than the simple baseline of $\tilde{O}(\sqrt{KNT})$ that can be obtained by using standard algorithms designed for non-stationary bandit problems. For the bandit meta-learning problem with fixed task length $\tau$, we show that the regret of the algorithm is bounded as $\tilde{O}(NM\sqrt{M \tau}+N^{2/3}M\tau)$. Under additional assumptions on the identifiability of the optimal arms in each task, we show a bandit meta-learning algorithm with an improved $\tilde{O}(N\sqrt{M \tau}+N^{1/2}\sqrt{M K \tau})$ regret.
\end{abstract}

\section{Introduction}

Consider a recommendation platform with a large catalog of items. 
Every day, the system gets to interact with different customers and must sequentially discover what is the preferred item of the day. But if we suppose that days correspond to recurring groups of customers, it naturally implies that only a small subset of the catalog is eventually preferred.
In this problem, like in many personalization or recommendation scenarios, the learner needs to solve a sequence of tasks.
When the tasks share some similarities, the learner may be able to utilize this by sequentially learning the underlying common structure in the tasks and hence perform better on the sequence than solving each task in isolation. In this paper we study a special class of these \emph{meta-learning} problems, where each task is an instance of a multi-armed bandit problem.

Formally, we consider a problem where a learner faces $N$ instances of $K$-armed bandit tasks sequentially. At the beginning of task $n\in [N]$,\footnote{For any integer $k$, we let $[k]=\{1,2,\dots,k\}$, and for any (multi-)set $S$, $|S|$ denotes the number of distinct elements in $S$.} the adversary  chooses the mean reward function $r_{n}\in \cR=[0,1]^K$. 
Then, the learner interacts with the $K$-armed bandit task specified by the reward function $r_n$ for $\tau_n$ time steps: for each $n\in [N]$, when action $a \in [K]$ is played, it receives reward $r_n(a)+\eta_{n,t}(a)$ where $\eta_{n,t}(a)$ are independent zero-mean and $[-1/2,1/2]$-valued noise variables for all $n,t$ and $a$, so that the expected reward is $r_n(a)$.

We consider two settings: the \emph{meta-learning setting} where the learner knows the length $\tau_n$ of each task $n$ when it starts, and the more general \emph{non-stationary setting} where the task boundaries are unknown~\citep{WL-2021,HKZCAGB-2020,Russac-Vernade-Cappe-2019,Auer-2019}.

We use $T$ to denote the total number of rounds, $T=\sum_{n=1}^N \tau_n$. Let $\cH_{n,t}$ be the history of the learner's actions and the corresponding observed rewards up to, but not including, time step $t$ in task $n$. For this time step, the learner chooses a distribution $\pi_{n,t}$ over the actions as a function of $\cH_{n,t}$ and other problem data (such as $T$, and $\tau_n$ in bandit meta-learning), and samples an action $A_{n,t}$ from $\pi_{n,t}$ to be used in this time step. The learner's policy $\pi$ is a collection of the mappings $\pi_{n,t}$ (formally, $\pi$ is a mapping from the set of possible histories to distributions over the action set $[K]$). Its regret on a sequence of reward functions $(r_n)_{n=1}^N$ with respect to a sequence of actions $(a_n)_{n=1}^N$ is defined as
\begin{align*}
R (\pi,T, (r_n)_{n=1}^N, (a_n)_{n=1}^N) = 
\E\left[\sum_{n=1}^N \sum_{t=1}^{\tau_n} r_{n}(a_n) - \sum_{n=1}^N \sum_{t=1}^{\tau_n} r_{n}(A_{n,t}) \right] \,,
\end{align*}
where the expectation is over the random actions of the learner, which may depend on the realization of the noise in the rewards as well as some internal randomness used by the learner. Note that, in the notation, we suppressed the dependence of the regret on the segment lengths $(\tau_n)_{n=1}^N$ and only indicated their sum $T=\sum_{n=1}^N \tau_n$.
The learner's worst-case expected regret relative to the best set of $M$ arms (in the sense of achieving the highest expected reward) is defined as
\begin{align*}
    R_T \doteq R (\pi,T, N,M)
    \doteq \sup_{(r_n)_{n=1}^N}  \max_{\substack{(a_n)_{n=1}^N
    : |\{a_n\}_{n=1}^N|\le M}} R(\pi, T, (r_n)_{n=1}^N, (a_n)_{n=1}^N) \;.
\end{align*}
Note that, as usual in bandit problems, we use the mean reward functions $r_n$ in the regret definition rather than the noisy rewards.
Importantly, the learner competes with a sequence of arms that has at most $M$ distinct elements. We call such problems \emph{sparse} non-stationary bandit problems \citep{Kwon-Perchet-Vernade-2017}. We say a sequence of reward functions $r_{1},\dots,r_{N}$ is \emph{realizable} if there exists a set of arms of size at most $M \le K$ which contains an optimal arm (achieving reward $\max_a r_n(a)$) for every task $n$. Otherwise, the problem is called agnostic. All our results, apart from those in Section~\ref{sec:greedy}, hold in the agnostic setting. In the meta-learning problems, our regret guarantees in fact hold even in the more general adversarial setting where the reward vector can change within each task and regret is defined as 
\begin{align*}
R_T = \sup_{(r_{n,t})_{n=1,t=1}^{N,\tau_n}}  \max_{\substack{(a_n)_{n=1}^N : |\{a_n\}_{n=1}^N|\le M}}\E\left[\sum_{n=1}^N \sum_{t=1}^{\tau_n} r_{n,t}(a_n) - \sum_{n=1}^N \sum_{t=1}^{\tau_n} r_{n,t}(A_{n,t}) \right] \;.
\end{align*}

Our goal in this paper is to provide algorithms achieving small worst-case expected regret for both the above problems. In the bandit meta-learning problem, the task lengths (and therefore, the change points) are known at the beginning of each task, while in the sparse non-stationary bandit problems, the start and duration of tasks are unknown, but the number of tasks $N$ is known at the beginning of the game. 

\subsection{Approach: reduction to bandit subset selection}
\label{sec:approach}

Our approach is based on a reduction to the bandit subset selection problem. We describe it first in the simpler bandit meta-learning problem. A natural solution to this problem is to use an appropriate base bandit algorithm
restricted to the optimal subset of $M$ arms at the beginning of each task. Unfortunately, this set is unknown. Nevertheless, it can be learned, and if the learner identifies the correct subset, a restricted version of the base bandit algorithm can be applied to the next task. For any $K$, let $B_{\tau,K}$ be a worst-case regret upper bound in $\tau$ steps for the \textit{base} $K$-armed bandit algorithm. 
Note that $B_{\tau_n,K}=\widetilde O(\sqrt{K \tau_n})$ for the standard bandit algorithms (such as \UCB and Thompson sampling in stochastic tasks, or EXP3 in adversarial tasks, see, e.g., \citealp{LS-2020}
)
;\footnote{The notation $\widetilde O$ hides polylogarithmic terms.}  accordingly, we assume throughout that $B_{\tau,K}$ is a non-decreasing concave function of both $\tau$ and $K$. When restricted to the correct $M$-subset (i.e, a subset of size $M$ containing the best arm) of the $K$ arms, the regret of the base algorithm is bounded by  $B_{\tau_n,M}\leq B_{\tau_n,K}$.
Otherwise, if the chosen subset does not contain the best arm, the regret on this task can be linear in $\tau_n$. Therefore, it is important to identify the correct subset quickly.

This way the bandit meta-learning problem can be reduced to a bandit subset-selection problem, where the decision space is the set of $M$-subsets of $[K]$ and the reward in round $n$ is the maximum cumulative reward in task $n$ over the chosen subset. For example, when all tasks have length $\tau$, and $\tau\ll N$, an $o(N)$-regret for the bandit subset-selection problem translates to an $O(N B_{\tau,M}) + o(N)$ bound on $R(\pi,T,N,M)$ for the bandit meta-learning problem. Notice that the leading term in the regret is $O(N B_{\tau,M})$, which is an upper bound on the regret of a bandit strategy that runs the base algorithm on the correct subset in each task. The subset-selection problem is a bandit submodular maximization problem, for which the online greedy approximation (\OGO) algorithm\footnote{The superscript $^o$ stands for "opaque" feedback model in the cited paper.}of \citet{streeter2007online} achieves regret $O(N^{2/3})$. 

The approach to the sparse non-stationary bandit problem is similar to the one described above, with the difference that pre-determined segments are used in place of task boundaries, and we require the base algorithm to have a guarantee for a notion of \textit{dynamic regret}: for any $1\le i\le j\le N$, and letting $L=\sum_{n=i}^j \tau_n$,
\vspace{-0.2cm}
\begin{equation}
\label{eq:basedynamicregret}
\sup_{(r_n)_{n=i}^j} \EE{\sum_{n=i}^j \sum_{t=1}^{\tau_n} (\max_{a\in [K]} r_{n}(a) - r_{n}(A_{n,t}))} \le B_{L,j-i,K},
\end{equation}
where we naturally extend the upper-bound notation to indicate the number of task changes in the segment of length $L$, and typically $B_{L,j-i,K}=\widetilde O(\sqrt{L(j-i+1)K})$. An example of such a base algorithm, which does not need to know the task boundaries and lengths, is \AdS~\citep{ADSWITCH-auer19a}.

\subsection{Contributions}\label{sec:contribution}
In \cref{sec:unknown}, we present an algorithm for our sparse non-stationary bandit problem, 
and show that its dynamic regret scales as $\widetilde O(M {N}^{-1/3} T (\log K)^{1/3} + M \sqrt{M N T})$.
This algorithm is based on the aforementioned reduction to bandit subset selection problem and uses a bandit submodular optimization method, and needs to know the number of change points in advance.
Without any restrictions on the set of optimal arms, the optimal rate for the non-stationary problem is $\widetilde{O}(\sqrt{KNT})$~\citep{ADSWITCH-auer19a,CLLW-2019,WL-2021}.\footnote{This rate is achieved by algorithms such as \AdS of \citet{ADSWITCH-auer19a} without any prior knowledge of $T$ or $N$.} For small $N$, this baseline rate cannot be improved. %
Therefore, we are mainly interested in the regime of a large number of tasks and a small number of optimal arms, for which our regret upper bound is smaller than the baseline regret. 

The same reduction to submodular optimization also applies to the bandit meta-learning problem, and leads to a regret bound that is better than the simple baseline of $O(\sum_{n=1}^N B_{\tau_n,K})$; this is shown in \cref{sec:known}. 
Importantly, in this setting, our regret guarantees hold even in the more general adversarial problems where the reward vector can change within each task.

In Section~\ref{sec:greedy}, we analyze the bandit meta-learning problem in the realizable setting with stochastic tasks, and provide two algorithms which provide sample-efficient solutions to the problem under some assumptions on the identifiability of the optimal arms in each task. 
When all tasks have the same length $\tau$, our computationally efficient algorithm \greedy achieves an $O(N B_{\tau,M (1+\log N)}) + \widetilde O(N^{1/2})$ regret, paying a small $\log N$ factor in the regret for computational efficiency. To remove this extra factor we design \pmml, which achieves the desired $O(N B_{\tau,M})+\widetilde O(N^{1/2})$ regret, at the cost of a computational complexity that is potentially exponential in $M$.
Although \greedy is similar to the greedy solution in offline submodular maximization, the proof technique is quite different and exploits the special structure of the problem and results in an improved regret. A naive reduction to bandit submodular maximization would yield only an $\widetilde O(N B_{\tau,M})+\widetilde O(N^{2/3})$ regret.

\section{Reduction to subset-selection and bandit submodular maximization}
\label{sec:subset}

As explained in Section~\ref{sec:approach}, we reduce our problems to a bandit subset-selection problem. However, instead of learning the subset of the best arms task by task (whose boundaries are unknown in non-stationary problems), it will be convenient to divide the time horizon $T$ into $N'$ segments of equal length 
$\tau'=T/N'$ (in meta-learning problems when tasks are of equal length, we can choose $N'=N$ making segments and tasks coincide), and learn the optimal arms from segment to segment. Consider segment $n$, $[(n-1)\tau'+1, n\tau')$. Without loss of generality, we assume that a new task starts at the beginning of each segment (we can redefine the task boundaries accordingly), so we have at most $N+N'-1$ tasks. That is, denoting by $N_n$ the number of tasks in segment $n$, we have 
\begin{align}
    \label{eq:segnum}
    \sum_{n=1}^{N'} N_n  \le N+N'-1\; .
\end{align} 
We also denote $\tau_{n,u}$ and $r_{n,u}$ respectively the length and mean reward of $u$'th task in segment $n$.

In segment $n$, the learner chooses a subset of arms $S_n\in \cS=\{S:S\in 2^{[K]}, |S|\le M\}$ and runs a base bandit algorithm on that subset for $\tau'$ steps and receives a pseudo-reward\footnote{We call it pseudo-reward since it considers the mean reward functions not the actual random actions. 
}
\begin{align*}
\sum_{u=1}^{N_n}\sum_{t=1}^{\tau_{n,u}} r_{n,u}(A_{n,u,t})
&= \sum_{u=1}^{N_n}\bigg( \tau_{n,u} \max_{a\in S_n} r_{n,u}(a)
- \sum_{t=1}^{\tau_{n,u}}\left(\max_{a\in S_n} r_{n,u}(a) - r_{n,u}(A_{n,u,t}) \right)\bigg)
\\
&\doteq f_n(S_n) - \tau' \varepsilon_n\;,
\end{align*}
where $A_{n,u,t}$ is the learner's action in time step $t$ of task $u$ of segment $n$, $f_n(S) \doteq \sum_{u=1}^{N_n} \tau_{n,u} f_{n,u}(S)$ and $f_{n,u}(S)\doteq \underset{a\in S}{\max}~ r_{n,u}(a)$ is the max-reward function for a set of arms $S$ and reward function $r_{n,u}$, and $\varepsilon_n$ is the average ``noise'' per time step observed by the agent. We assume the base algorithm satisfies \eqref{eq:basedynamicregret},
$\tau' \E[\varepsilon_n] \le B_{\tau',N_n,M}$.
We prove in \cref{app:submod} that $f_n\in \cF$, the family of submodular functions, since it is an affine combination of finitely many submodular functions. Now we can bound the regret in the sparse non-stationary bandit problem as follows (proof in \cref{app:metalearning2subset}):
\begin{lemma}
\label{lemma:metalearning2subset}
For any policy running a base algorithm satisfying \eqref{eq:basedynamicregret} in each segment $n$ on a selected subset of arms $S_n$, we have
$
R_T 
\le~\!\! 
\underset{f_1,\ldots,f_{N'}{\in \cF}}{\sup}\!~ \underset{S\in \cS}{\max} \; \E\!\left[\sum_{n=1}^{N'} \!\Big(f_n(S)\! - \!f_n (S_{n})\! +\! B_{\tau',N_n,M} \Big)\! \right]\!\!.
$
\end{lemma}
Thus, the sparse non-stationary bandit problem can be reduced to minimizing a notion of regret where in segment $n$, the learner chooses a subset $S_n\in\cS$ and observes $f_n(S_n)-\tau'\varepsilon_n$. 
Since $f_n$ is submodular, this reduction allows us to leverage the literature on bandit submodular optimization to get a bound on $R_T$. %
\subsection{Bandit submodular maximization}
\label{sec:submodular-disc}

As mentioned above, the problem of subset selection %
is an instance of online submodular maximization. \citet{streeter2007online} studied the online submodular maximization problem in four different settings: the full-information setting where the function $f_n$ is fully observed at the end of round $n$; the priced feedback model where the learner can observe $f_n$ by paying a price and the price is added to the total regret; a partially transparent model where values of $f_n$ for some subsets are revealed, and the bandit setting where only $f_n(S_n)$ is observed. For the full-information and partially transparent models, an $O(\sqrt{N})$ regret is shown, while for the priced feedback and bandit model, an $O(N^{2/3})$ regret is shown. \citet{RKJ-2008} provide a similar algorithm to that of \citet{streeter2007online} for a particular ranking problem, with the more informative partially transparent feedback model, and also obtain an $O(N^{1/2})$ regret bound.
The priced feedback model is similar to problems where the best arms can be identified in every task, which we study in Section~\ref{sec:greedy}. 

\section{A general solution based on submodular maximization}
\label{sec:unknown}

In this section we present our first algorithm, \BOG, for ``online submodular bandit subset selection'', based on the above reduction and the general submodular maximization algorithm \OGO of \citet{streeter2007online}. The algorithm is applicable in the agnostic setting and its pseudo-code is given in \cref{alg:BanditOG}. \BOG applies $\Mlogn=\mppr$ expert algorithms (i.e, regret minimization algorithms in the full information setting) over the $K$ arms, denoted by $\mathcal{E}_1,\ldots,\mathcal{E}_{\Mlogn}$, where the role of $\mathcal{E}_i$ is to learn the $i$th ``best'' action.

In segment $n$, each expert picks an action. The algorithm \emph{exploits} (\ext) the set of $\Mlogn$ actions picked by the experts with probability $1-\gamma_n$, 
i.e, a base algorithm \base is executed on this set for $\tau'$ steps. With probability $\gamma_n$, the algorithm \emph{explores} (\exr), meaning that first a random index $i\in [\Mlogn]$ is chosen uniformly at random, and then the set of $i-1$ actions picked by the experts $\{\ca{E}_1,\cdots,\ca{E}_{i-1}\}$ is chosen, and together with a random action $a'_i$, \base is executed on this set of $i$ actions, $S_{n:i}$, for $\tau'$ steps.
When exploring, the expert whose action is replaced with a random action ($\mathcal{E}_i$) is updated with a reward function where the reward for the chosen action $a'_i$ is the average reward of \base in the segment (approximately $f_n(S_{n:i})/\tau'$ up to error $\varepsilon_n/\tau'$), while for all other actions the reward is zero, and all other experts are updated with an all-zero reward function.\footnote{This is simply to construct an unbiased estimator and use importance sampling. We provide no feedback to the experts during the exploitation phase. It is reasonable to expect that by using such feedback we can improve the empirical performance of the algorithm, but using the feedback is not going to improve the theoretical guarantees.}
In this way, $\mathcal{E}_1,\ldots,\mathcal{E}_i$ jointly learn the identity of the best $i$-set given the (approximate) reward function $\sum_{n=1}^N f_n$, hence approximating a greedy solution for finding the top $M$ arms (essentially, the ``real'' reward of expert $\mathcal{E}_i$ for action $a$ is $f_n(S_{n:i-1}\cup\{a\})-f_n(S_{n:i-1})$). \BOG should be viewed as a simulation of the off-line greedy procedure, which constructs its solution incrementally~\citep{streeter2007online}. When an index $i$ is chosen, the algorithm is learning the $i$th choice of the off-line greedy procedure and therefore only a subset of size $i$ is played.

Our requirement about the expert algorithm is that it achieves an $O(\sqrt{v \log(K)})$ regret over $v$ time steps relative to the best action selected in hindsight; this is achieved by all standard expert algorithms, such as exponential weights \citep{PLG}.
The choice of \base depends on which setting we solve.

\begin{algorithm}[tb]
\caption{\BOG: Online Submodular Bandit Subset Selection for Non-Stationary Bandits and Meta-Learning}
\label{alg:BanditOG}
\begin{algorithmic}[1]
\STATE {\bfseries Input}: Subset size $M$, segment length $\tau'$, {\tt Expert} algorithms $\mathcal{E}_1,\cdots,\mathcal{E}_{\Mlogn}$;
\FOR{$n \in [N']$}
    \STATE For each $i\in[\Mlogn]$, use $\mathcal{E}_i$ to choose an action $a_i$ 
    \STATE Set $S_{n}=\{a_1,\cdots,a_{\Mlogn}\}$
    \STATE With prob. $\gamma_n$, $E_{n} = \exr$, otherwise $E_{n} = \ext$
        \IF{$E_{n} = \ext$}
        \STATE Run \base on $S_{n}$ for $\tau'$ steps
        \ELSE
        \STATE Choose $i\in[\Mlogn]$ uniformly at random
        \STATE Choose a new action $a'_{i}$ uniformly at random 
        \STATE Replace $i$'th element of $S_{n}$ with $a'_{i}$: 
        \STATE \quad $S_{n:i}\gets\{a_1,\cdots,a_{i-1},a'_i\}$
        \STATE Run \base on $S_{n:i}$ for $\tau'$ steps
        \STATE Feed back the average reward over the segment as a reward to $\mathcal{E}_i$ for $a_i'$, and zero reward for all other actions and experts
        \ENDIF
\ENDFOR
\end{algorithmic}
\end{algorithm}

\subsection{The sparse non-stationary bandit setting}

To choose \base, notice that in the sparse non-stationary bandit problem, every segment may contain multiple tasks. Therefore, we need to choose an algorithm that is able to solve non-stationary bandit problems, without the knowledge change points. Hence, we use the \AdS algorithm of \citet{ADSWITCH-auer19a} as \base in \cref{alg:BanditOG}, giving $B_{\tau',N_n,\Mlogn}=\sqrt{\Mlogn N_n \tau'}$ for segment $n$ in Lemma~\ref{lemma:metalearning2subset}. 
The performance of the algorithm is analyzed in the next theorem (proved in \cref{sec:BOGpfs}).
Note that the algorithm does not need to know when the tasks change and the tasks can be of different lengths.
\jtodo{how to use unknown N results}
\begin{restatable}[]{theoremi}{BOGNonStatOpt}
\label{thm:BOG-NonStat-Opt}
The regret of \BOG using \AdS as \base, with $N'$ segments each of size $\tau'=T/N'$, and using exploration probability $\gamma_n=\gammaOne
$ is $R_T = \tilde O(
T {N'}^{-1/3} (\Mlogn^4 K \log K)^{1/3} + \Mlogn \sqrt{\Mlogn N' T} )$. 

\end{restatable}
We are mainly interested in the regime of large number of switches and small number of optimal arms. 
If $N \ge \Nthresh$ and $M \le K^{1/3}$, with the choice of $N'=N$, we have 
$\Mlogn (\log K)^{1/3} N^{-1/3} T \le \Mlogn^{3/2}\sqrt{N T}$, and the regret of \BOG improves upon the $\tilde{O}(\sqrt{KTN})$ bound of standard non-stationary bandit algorithms (such as \AdS).
In the case of a small number of switches, however, the simple baseline of $\tilde{O}(\sqrt{KTN})$ can be smaller than our bound, and the learner should simply play a standard non-stationary bandit algorithm. %
At the extreme that $N=O(1)$, i.e. only a few changes in the environment, the regret bound $\tilde{O}(\sqrt{KNT})$ cannot be improved. On the other hand, when $N$ is large compared to $T$, it is easy to establish a $O(\sqrt{MNT})$ lower bound, and therefore our $\tilde{O}(M\sqrt{MNT})$ bound is optimal up to a factor of $M$. Closing this gap remains an open question. We can improve the bounds in certain regimes by further tuning of the parameters, which is discussed in more details in \cref{sec:BOGpfs}.

The algorithm needs to know the number of changes $N$ in advance: the segment length $\tau'$ depends on $N'$, and as the proof of \cref{thm:BOG-NonStat-Opt} reveals, we need $N'\ge N$ to have the above regret upper bound. In the next section, we consider a meta-learning setting where segment lengths are fixed in advance. In that case, we can have a version of the algorithm that does not need to have $N$ as input.

\subsection{Bandit meta-learning}\label{sec:known}

In this section we consider the bandit meta-learning problem with tasks of equal and known length $\tau$. Again, we apply \BOG, instantiated now with different parameters and a different \base algorithm. Since we know $\tau$, we can always introduce a segmentation such that each segment contains a single task only (possibly only a part of a task), and choose \base to be an algorithm designed for a multi-armed bandit problem such as the well-known \UCB algorithm \citep{ACFS-2002} when tasks are stochastic bandits, or EXP3 algorithm \citep{ACFS-1995} when tasks are adversarial bandits. 

\begin{restatable}{corollaryi}{thmBG}
\label{thm:BOG}
Consider \BOG with \UCB as \base in a meta-learning problem with $N$ stochastic bandit tasks of equal length $\tau$. %
By the choice of $\tau'=\tau$ and $\gamma_n=\gammaOnenprime$, the regret is $R_T = \tilde O(
\tau {N}^{2/3} (\Mlogn^4 K \log K)^{1/3}
+\Mlogn N\sqrt{\Mlogn \tau})$. 
\end{restatable}
The proof is omitted as it follows the proof argument of \cref{thm:BOG-NonStat-Opt}. Again, in the regime of large number of tasks and small number of optimal arms, our bound improves upon the $\tilde{O}(N\sqrt{K\tau})$ bound of the trivial solution of running an independent UCB algorithm in each task.

We can extend the above result to an adversarial setting, where the reward vector can change in every time step, by using an EXP3 algorithm instead of UCB. Such an extension is  not possible in the sparse non-stationary setting as currently we are not aware of a base algorithm with \textit{adaptive} dynamic regret guarantee in the adversarial setting.

\begin{remarki}[Variable task lengths]
Consider problems with non-equal task lengths where the learner only gets to know the length of each task when they begin. To deal with this situation, we construct an exponential grid for the task lengths with $b:=\log(\max_n\tau_n)\leq \log T$ buckets where bucket $i$ is $[\underbar{$\tau$}_i:=2^{i-1}, \overline{\tau}_i:=2^i]$. %
Now we run a copy of \BOG on the tasks falling in each bucket as they arrive. Let $N^{(i)}$ denote the number of tasks that fall in bucket $i$. Then by \cref{thm:BOG}, the total regret satisfies
$%
    R_T\leq \sum_{i=1}^{b}\tilde{O}(
    \Mlogn^{4/3} (N^{(i)})^{2/3}K^{1/3}\overline{\tau}_i
    +N^{(i)} \Mlogn^{3/2}\sqrt{\overline{\tau}_i}).~ 
$%
\end{remarki}

\section{Bandit meta-learning under an identifiability condition}
\label{sec:greedy}

In this section, we study the bandit meta-learning problem in the \emph{realizable} setting under the assumption that the learner has access to an exploration method that reveals optimal actions. We further assume that the tasks are of equal length $\tau$.

\begin{assumptioni}[Efficient Identification]
\label{assumption:identification}
There exists a set of $M$ arms that has a non-empty intersection with the set of optimal arms in each task. Also, the learner has access to a best-arm-identification (BAI) procedure that for some $\delta \in [0,1]$, with probability at least $1-\delta/N$, identifies the set of optimal arms if executed in a task (for at most $\tau$ steps)\footnote{Note that the Efficient Identification assumption requires the BAI procedure to return only optimal arms. This choice is for simplicity and could easily be relaxed to allow the returned set to be all arms with sub-optimality gap smaller than $\Theta(\sqrt{M\log (N/\delta)/\tau})$; we discuss this in more details in the analysis.}.
\end{assumptioni}
The Efficient Identification assumption is a special case of the priced feedback model of \citet{streeter2007online}. 
If for any task $n$ with optimal arms $S^*_n \subset [K]$, we have $r_n(a^*_n) - \max_{a\not\in S^*_n}r_n(a) \ge \Delta$ for all $a^*_n \in S^*_n$ (note that $r_n(a^*_n)$ is the same for all $a^*_n\in S^*_n$) for some $\Delta=\Theta(\sqrt{K\log (N/\delta)/\tau})$, a properly tuned \emph{phased elimination} (PE)\footnote{Improved \UCB of \citet{AO-2010} runs exponentially growing phases and maintains an active set of arms with gaps twice smaller in every step.} procedure~\citep{AO-2010} returns the set of optimal arms with probability at least $1-\delta/N$. 
The cumulative worst-case regret of PE in a task with $K$ arms is $B'_{\tau,K}=\Theta(B_{\tau,K})$ \citep{AO-2010}, see 
\citet{lattimore-Bandit} for details. With a slight abuse of notation, in this section we use $B_{\tau,K}$ to denote $\max\{B_{\tau,K},B'_{\tau,K}\}$.

Akin to \BOG, we disentangle exploration (\exr) and exploitation (\ext) at a meta-level.
In \exr mode, the learner executes a BAI on all arms, and (by Assumption~\ref{assumption:identification}) observes, with high probability, the set of optimal actions, $S_n^*$. 
The price of this information is a large regret denoted by $\Cinfo$, which for a properly tuned PE, we know $\Cinfo=B_{\tau,K}>B'_{\tau,M}$.
So, since we aim for $\reg_T \le \widetilde O(N B_{\tau,M}) + o(N)$, we should keep the number of \exr~calls small. In \ext mode, the learner executes a base bandit algorithm on a chosen subset $S_n$, constructed using the previously identified optimal actions $\cI_n=\bigcup_{j<n: E_j=\exr} \{S^*_j\}$.

Let $s_n$ be the size of $S_n$. If $S_n\cap S_n^*\neq \emptyset$, the regret of the base algorithm is bounded by $\Cgood=B_{\tau,s_n}$. Otherwise, since the performance gap between the optimal arms and the arms in $S_n$ can be arbitrary, the regret in the task can be as large as $\Cbad=\tau$. Note that to keep $\Cgood$ small, the subset $S_n$ should be as small as possible. Ideally, $S_n$ should be a subset of size $M$ that has non-empty overlap with all members of $\cI_n$. However, the problem of finding such $S_n$ 
is the so-called \emph{hitting set} problem, which is known to be NP-Complete \citep{feige2004approximating}.

A simple greedy algorithm can be used to get an approximate solution efficiently (see, e.g, \citealp{streeter2007online}):\todoa{Explain what the greedy algorithm does} it has polynomial computation complexity and 
it finds a subset of size at most $M(1+\log N)$ that contains an optimal action for each task. We say an action $a\in [K]$ covers task $j$ if $a\in S_j^*$. The greedy method, denoted by \textsc{Greedy}, starts with an empty set and at each stage, it adds the action that covers the largest number of uncovered tasks in $\cI_n$, until all tasks are covered.

We also propose \pmml, that is based on an elimination procedure; 
the learner maintains an active set of possible $M$-subsets compatible with the \exr history, and all subsets that are inconsistent with $\cI_n$ are eliminated. In $\ext$ mode, a subset is selected uniformly at random from the set of active subsets. As we will show, this algorithm improves the regret by a factor of $\log N$, although it is not computationally efficient.

\begin{algorithm}[tb]
 \caption{\bass: Bandit Subset Selection for Meta-Learning} \label{alg:ossem}
 \begin{algorithmic}[1]
\STATE {\bfseries Option:} Greedy \greedy~{(G)}, Elimination-based \pmml~{(E)}\;

\STATE {\bfseries Input:} \base (efficient $K$-armed bandit algorithm), best arm identification algorithm \texttt{BAI}, \exr~probabilities $p_n$, (E) subset size $M$\;

\STATE  \textbf{Initialize:} Let
 {(G)} ${\cI_0}=\emptyset$; {(E)} $\cX_0$ be the set of all $M$-subsets of $[K]$.\;
 \FOR{$n=1,2,\dots,N$}{
  
\STATE  With prob. $p_n$, let $E_n=\textsc{Exr}$; otherwise $E_n=\textsc{Ext}$\;
  \IF{$E_n=\textsc{Exr}$ or $n=1$}{
  
\STATE  Run \texttt{BAI} on all arms and observe the best arms $S^*_n$ of this task\;
  
\STATE  {(G)} Update $\cI_n = \cI_{n-1} \cup \{S^*_n\}$\;

\STATE  {(E)} Let $\cX_n=\{S \in \cX_{n-1}: S\cap S^*_n \neq \emptyset\}$ be elements of $\cX_{n-1}$ with non-empty overlap with $S^*_n$; 
  }\ELSE{
  
\STATE  {(G)} Find $S_n$ by \textsc{Greedy} s.t. $\forall S\in\cI_n, S_n\cap S \neq \emptyset$\;
  
\STATE  {(E)} Sample $S_n$ uniformly at random from $\cX_n$\;
 
\STATE Run \base algorithm on $S_n$ \;
  }\ENDIF
 }\ENDFOR
\end{algorithmic}
\end{algorithm}

The analysis of \greedy depends on the \emph{cost-to-go} function of the following game between the learner and environment. At round $n$, the learner may choose $E_n\in\{\exr,\ext\}$ with distribution $p_n=P(E_n=\exr)$ and cost $\Cinfo$. The environment may choose a best arm $a^*_n$ that the learner already knows about which costs $\Cgood$ (i.e, $a^*_n \in S_n$
) or choose an optimal arm set $S_n^*$ so that $S_n^*\cap S_n=\emptyset$, with cost $\Cbad$. Let $q_n=P(S_n^*\cap S_n=\emptyset)$.
The regret of \greedy is bounded by the cost of the learner in this simple game, if we assume  $\delta=0$ in \cref{assumption:identification}.
The learner is a (randomized) function of 
$\cI$, hence we can easily define and write the minimax cost-to-go function as 
\begin{align}
\begin{split}
    V_N(\cI)  &= 0~~\text{and for } n<N, 
    \\
     V_n(\cI) 
     &= \min_{p}\max_{q} \{ p \Cinfo + q(1-p)\Cbad 
     +(1-q)(1-p)\Cgood 
     \\&\qquad
      +(1-pq)V_{n+1}(\cI) 
     +pq V_{n+1}(\cI\cup \{S_n^*\}) \}.  \label{eq:cost-to-go-def}
\end{split}
\end{align}
For the last equality note that when the environment reveals a new action (happens with probability $q$) and the learner explores (with probability $p$), its current knowledge set $\cI$ is incremented.
The optimal cost-to-go function $V_n$ above corresponds to the case of $\delta=0$ in Assumption~\ref{assumption:identification}, and $V_0(\emptyset)$ gives the minimax regret for the family of algorithms with the limited choice described.
Therefore, when the BAI algorithm is successful, almost surely, $\reg_T \le V_0(\emptyset)$. For $\delta>0$, using a union bound, we can show
$\reg_T \le V_0(\emptyset) + \delta N\tau$. 
Setting $\delta=1/(N\tau)$ ensures that $\delta N\tau$ is negligible. 
Finally, if the BAI algorithm only returns a set of approximately optimal arms satisfying $r_n(a) \ge r_n(a^*_n)-\Delta$ for all arms $a$ selected, the meta-regret can be bounded trivially as 
$\reg_T \le V_0(\emptyset) + (\delta + \Delta)N\tau$.

Before deriving $V_0(\emptyset)$ in the general case, we consider the more restricted setting where each task has a unique optimum.

\paragraph{Characterizing an optimal policy in the case of unique optimal arms:}
Assume there is a unique and identifiable optimal arm in each task. 
\begin{assumptioni}[Unique Identification]
\label{assumption:unique-identification}
Assumption~\ref{assumption:identification} holds, and each task has a unique best arm.
\end{assumptioni}

\begin{theoremi}
\label{thm:unique-identification}
In the simple game, 
\[
V_0(\emptyset) \le N B_{\tau,M} + M \sqrt{2(\Cinfo - \Cgood)(\Cbad - \Cgood) N}\;.
\]
Therefore, under \cref{assumption:unique-identification}, the regret of \greedy can be bounded, for an appropriate selection of the parameters
$p_n=\Theta(1/\sqrt{N-n})$, as
\[
\reg_T  \le N B_{\tau,M} + M \sqrt{B_{\tau,K} N\tau} + \delta N\tau\;.
\]
\end{theoremi}

We prove this in \cref{sec:thm:unique-identification-pf} by solving the min-max problem in \eqref{eq:cost-to-go-def} for $V_n$.
Interestingly, the exploration probability $p_n$ increases as $\Theta(1/\sqrt{N-n})$. This might seem counter-intuitive at first as typically the exploration rate decreases in most online learning algorithms. The intuition is that as $n$ gets closer to $N$, if $s<M$, the adversary does not have much time left to use the remaining budget to make the learner suffer a big cost. Therefore, the adversary needs to increase its probability of choosing a new optimal arm.

\paragraph{The more general case:}
Now we consider the more general case with potentially multiple optimal arms in each task. 
\begin{theoremi}
\label{thm:greedy}
Let $M'=M(1+\log N)$. 
Under \cref{assumption:identification}, the regret of the \greedy~algorithm with exploration probability $p_n=\sqrt{\frac{|S_n| K \tau}{N B_{\tau,K}}}$ for all $n$ is bounded as
\begin{align*}
\reg_T \le N B_{\tau,M'} + M B_{\tau,K} +\sqrt{M K B_{\tau,K} N \tau} + \delta N\tau \;.
\end{align*}
\end{theoremi}
The proof is similar in spirit to that of Theorem~\ref{thm:unique-identification} and is deferred to Appendix~\ref{ap:proof-thm-greedy}. The regret guarantee holds in the realizable setting, but the \greedy~algorithm does not need $M$ as input. The next theorem shows that the regret of \pmml~is bounded as $N B_{\tau,M} + o(N)$, which is smaller than the regret of the \greedy~ algorithm by a factor of $\log N$; note, however, that \pmml is not computationally efficient and also requires $M$ as input. The proof of the theorem is in Appendix~\ref{app:partial-monitoring}.
\begin{theoremi}
\label{thm:PM}
Under \cref{assumption:identification}, with constant exploration probability $p_n=\left(\frac{\tau}{K}\right)^{1/4}\sqrt{\frac{\log K}{N}}$, the regret of the \pmml~ algorithm is bounded as $\reg_T \le N B_{\tau,M} + O(\tau^{3/4}K^{1/4}\sqrt{N M \log K})$.
\end{theoremi}

\begin{remarki}[Connections with partial monitoring games]
The setting of this section can be viewed more generally as a partial monitoring game. Partial monitoring is a general framework in online learning that disentangles rewards and observations (information). 
In our bandit meta-learning problem, different actions of the meta-learner (\textsc{Exr} and \textsc{Ext}) provide different levels of information and have different costs, and the problem can be reduced to a partial monitoring game on $\cX$, the set of $M$-subsets of $[K]$. More details are in Appendix~\ref{app:partial-monitoring}.
\end{remarki}

\section{Related work}
\label{sec:related}

Our bandit solution is based on subset selection, which connects it to many branches of online learning and bandit literature. %

\noindent\textbf{Non-stationary experts/bandits with long-term memory. } The sparse non-stationary bandit problem is the bandit variant of experts problem with small set of optimal arms whose study goes back to \citet{Bousquet-Warmuth-2002}. %
The only result in the bandit setting that we are aware of is the work of \citet{ZLDW-2019}, who show an algorithm for competing against a small set of optimal arms in an adversarial setting with sparse reward vectors. The solution of \citet{ZLDW-2019} has similarities with our approach, as both solutions are reduction-based and employ a meta-learner that plays with base algorithms. Similar to our approach, the meta-learner of \citet{ZLDW-2019} needs to satisfy a static regret guarantee, while the base algorithm needs to satisfy a dynamic regret guarantee. The algorithm of \citet{ZLDW-2019} also requires the number of change points as input and its dynamic regret is
$%
\Tilde O ((s N)^{1/3} (MT)^{2/3} + M \sqrt{s T \log K} + M K^3 \log T) \,,~
$%
where $s=\max_n \|r_n\|_0$ is the number of non-zero elements in reward vectors. Notice that without this sparsity condition, the above dynamic regret is worse than $\Tilde O(\sqrt{KNT})$ regret of EXP3.S. When $s$ is a small constant, the above bound improves upon the regret of EXP3.S when $M^{4/3}(T/N)^{1/3} < K < (TN/M^2)^{1/5}$.

\noindent\textbf{Meta-learning in adversarial bandits.} \citet{balcan2022meta} studies bandit meta-learning problems with adversarial bandit tasks. They introduce a meta-learning algorithm that tunes the initialization and step-size of the online mirror decent base algorithm. \citet{balcan2022meta} show that their algorithm achieves  $\min_{\beta\in(0,1]}\tilde{O}(N\sqrt{H_\beta K^\beta \tau/\beta}+N^{1-\beta/6})$ total regret, where $H_\beta$ is a notion of entropy. When a subset of size $M$ contains the optimal arms of most tasks, then by the choice of $\beta=1/\log(K)$ the regret bound simplifies to $\tilde{O}(N\sqrt{M\tau} + N^{1-1/(6\log K)})$.  
Compared to our results in Section~\ref{sec:known}, their convergence rate of $\tilde{O}(1/N^{1/(6\log K)})$ is much slower than our convergence rate of $\tilde{O}(1/N^{1/3})$. However, their task-averaged regret is $O(\sqrt{M \tau})$, whereas it is $O(M\sqrt{M \tau})$ in our case. Finally, the regret bound of \citet{balcan2022meta} adaptively holds for the best value of $M$, while $M$ is an input to our algorithm.
\citet{osadchiy2022online} further improves the results of \citet{balcan2022meta} under extra identifiability assumptions.  

\noindent\textbf{Slate bandits. } The reduction in Section~\ref{sec:subset} is an instance of slate bandit problems with a non-separable cost function~\citep{DVJ-2019,RAZK-2020,KRS-2010}.
\citet{RAZK-2020} study the problem in the stochastic setting, where the reward parameter is fixed throughout the game. \citet{MM-2019} study a problem that includes the probabilistic maximum coverage as a special case. They obtain problem-dependent logarithmic and problem-independent $O(\sqrt{N})$ regret bounds. However, the feedback structure in this work is richer than our setting. Applied to our problem, they assume that in each round, for each item and task pair, a random variable is observed whose expected value is the probability that the item is the optimal arm in that task.

\noindent\textbf{Meta-learning and bandit meta-learning. }
Meta, Multi-Task, and Transfer Learning \citep{baxter2000model,caruana1997multitask,thrun1996learning} are related machine learning problems concerned with learning a lower dimensional subspace across tasks. In that sense, our work is connected to other theoretical studies \citep{franceschi2018bilevel,denevi2018incremental,DCSP-2018,DCGP-2019,KSSKO-2020,khodak2019provable,TJJ-2021} though indeed we focus on the bandit learning setting. Various other ways of modelling structure have been proposed and studied in bandit meta-learning. A special case of our problem was studied by \citet{ALM-2013} where $K$-armed bandit problems are sampled from a prior over a finite set of tasks. \citet{PSJO-2021} consider a continual learning setting where  the bandit environment changes under a Lipschitz condition. \citet{kveton-2020} observe that the hyperparameters of bandit algorithms can be learned by gradient descent across tasks. Learning regularization for bandit algorithms \citep{KKZHMBS-2021,CLP-2020} are also proposed, building on the biased regularization ideas from \citet{baxter2000model}. Interestingly, these contextual problems are also connected with latent and clustering of bandit models \citep{MM-2014,GLZ-2014,HKZCAB-2020,HKZCAGB-2020}.

\noindent\textbf{Bandits with very large action spaces. } 
As $K$ grows very large, our bandit meta-learning problem is akin to infinitely many armed bandits \citep{BCZHS-1997,WAM-2008,BP-2013,CV-2015,chan2020infinite} and countable-armed bandits~\citep{Kalvit-Zeevi-2020} though these settings do not have a meta-learning aspect. 

\section{Experiments}

In this section, we study the performance of our algorithms on synthetic environments. 
The experiments include: 1) \greedy
~\footnote{\pmml~is computationally too expensive so we only run it on smaller settings in \cref{sec:FurtherExperiments}.}
from \cref{sec:greedy}, 2) \cref{alg:BanditOG}, \BOG, 3) \MOSS which is agnostic MOSS \citep{AB-2009} running independently on the tasks without any knowledge of the optimal $M$-subset, 4) \OptMoss, an oracle MOSS that plays only the arms in the optimal $M$-subset, and its performance constitutes an empirical lower bound on the achievable regret, and 5) \OGO, which is \BOG with the choice of $\tau'=\tau$ and optimized $\gamma$. Error bars are $\pm1$ standard deviation computed over $5$ independent runs.

We study the impact of four variables on the regret: number of tasks $N$, length of each task $\tau$, number of arms in each task $K$, and the optimal subset size $M$. To do so, we fix a default setting of $(N,\tau,K,M)$ and for each experiment we let one of these parameters vary. The problems are generated by an \emph{oblivious} adversary (see \cref{sec:FurtherExperiments} for further details). 

\cref{fig:all} demonstrates the impact of $N$ and $M$. Further experiments in \cref{sec:FurtherExperiments} illustrate the effect of all four variables including $\tau$ and $K$. 
Under \cref{assumption:identification} (left two plots), \greedy~outperforms all methods with a regret close to that of the oracle \OptMoss. When this assumption does not hold (right two plots), \BOG outperforms the other algorithms, while \greedy naturally has high variance.

\begin{figure*}[htb!]
    \centering
    \begin{minipage}{.24\textwidth}
        \includegraphics[width=\textwidth]{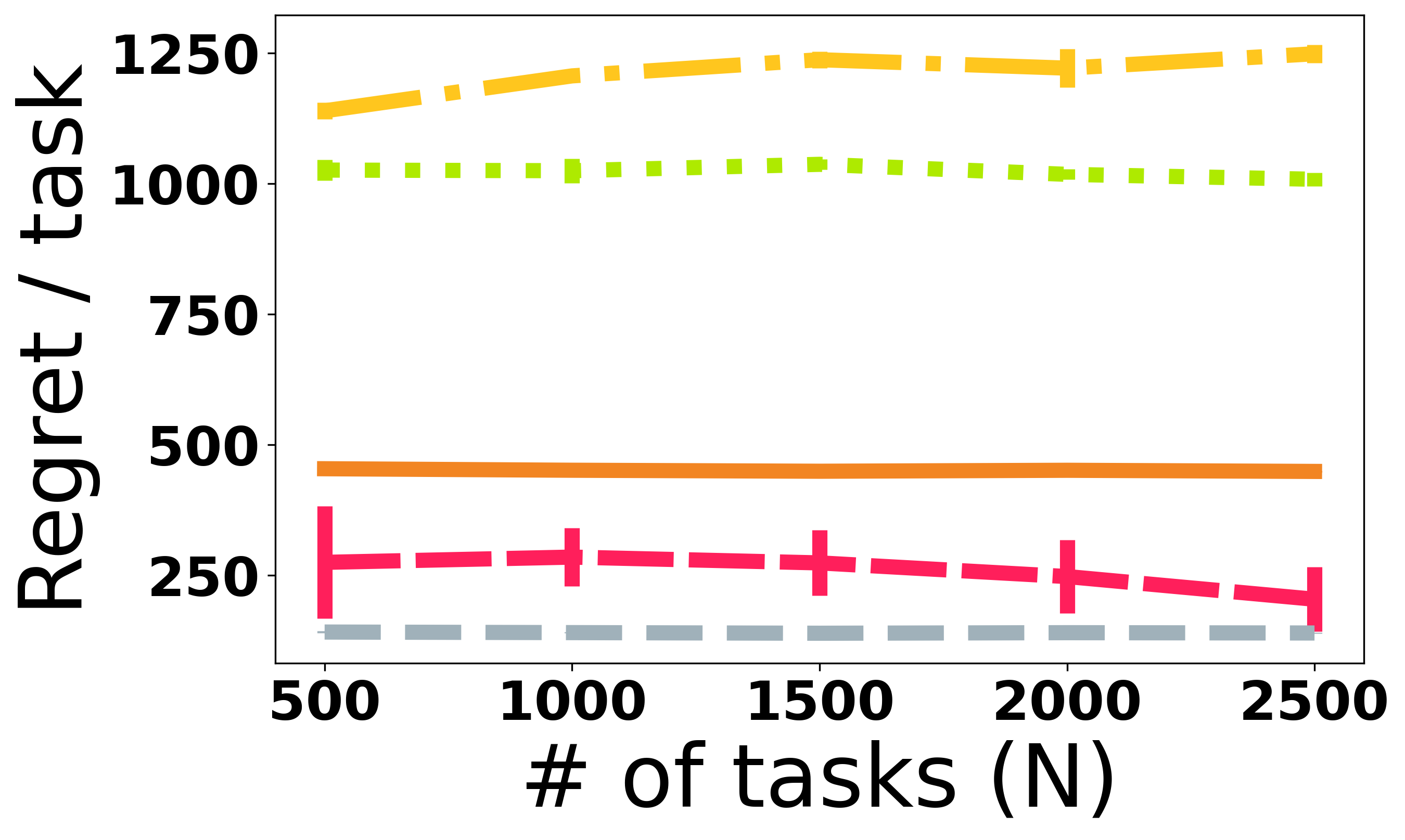}
    \end{minipage}
    \begin{minipage}{.24\textwidth}
        \includegraphics[width=\textwidth]{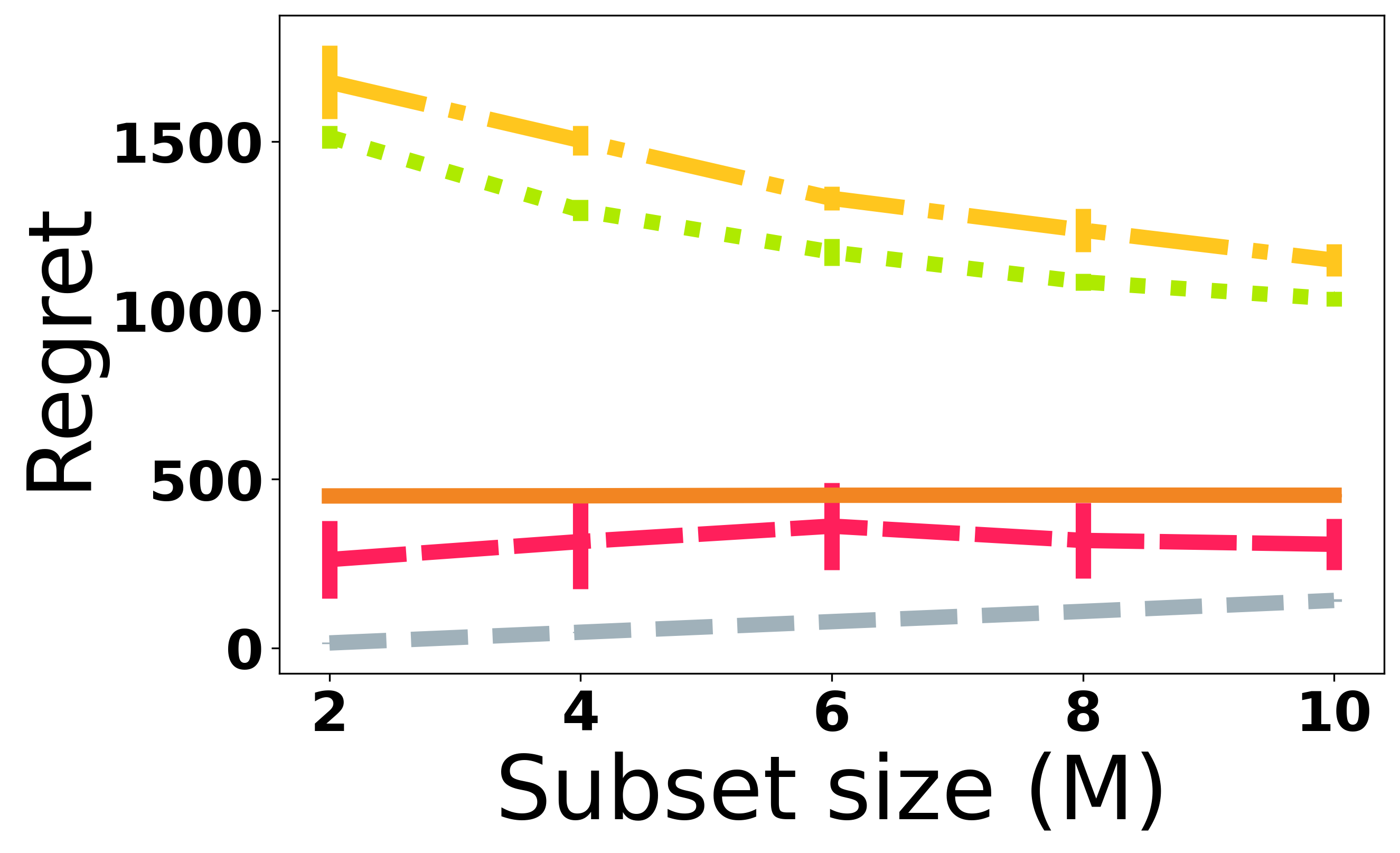}
    \end{minipage}
    \begin{minipage}{.24\textwidth}
        \includegraphics[width=\textwidth]{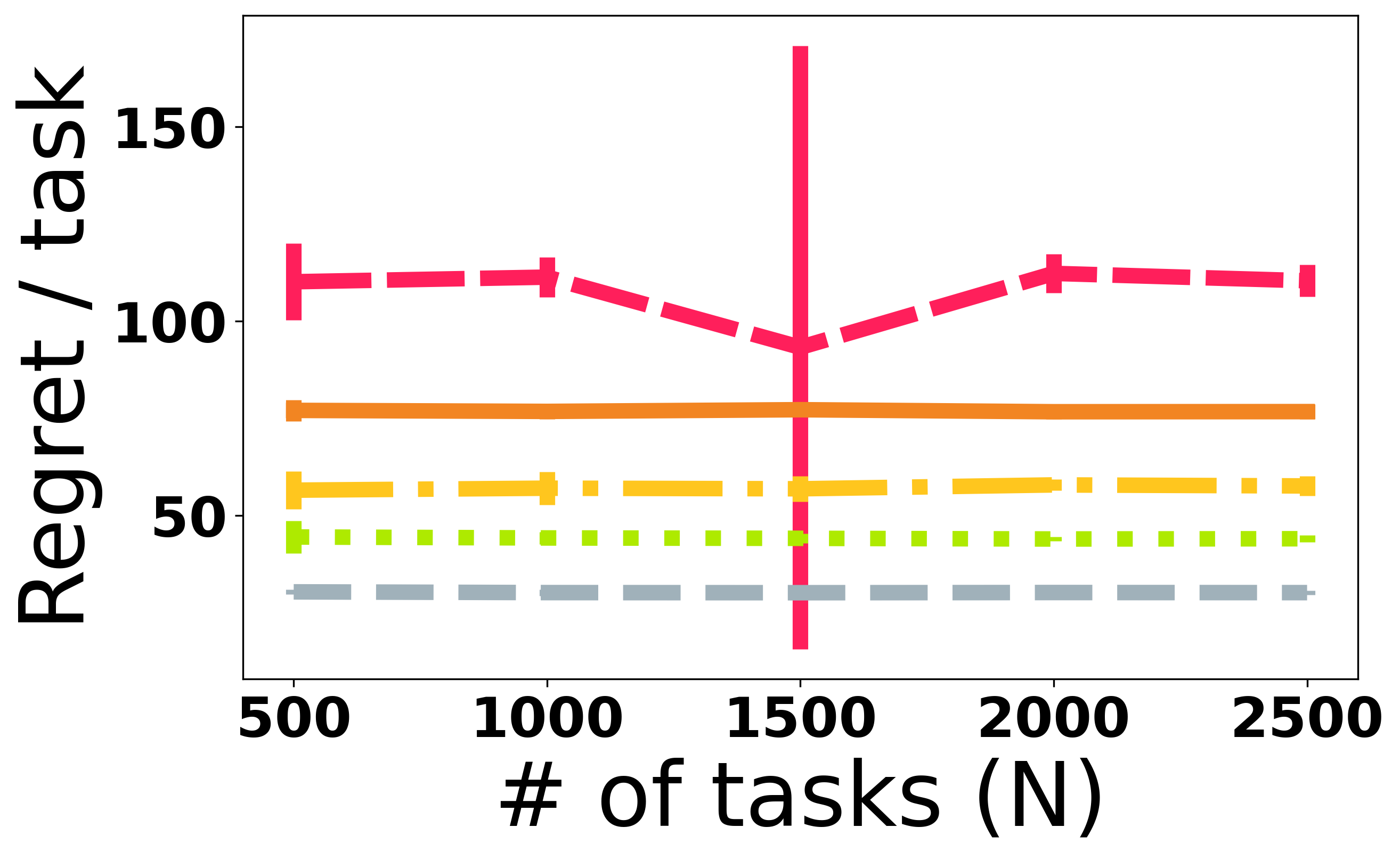}
    \end{minipage}
    \begin{minipage}{.24\textwidth}
        \includegraphics[width=\textwidth]{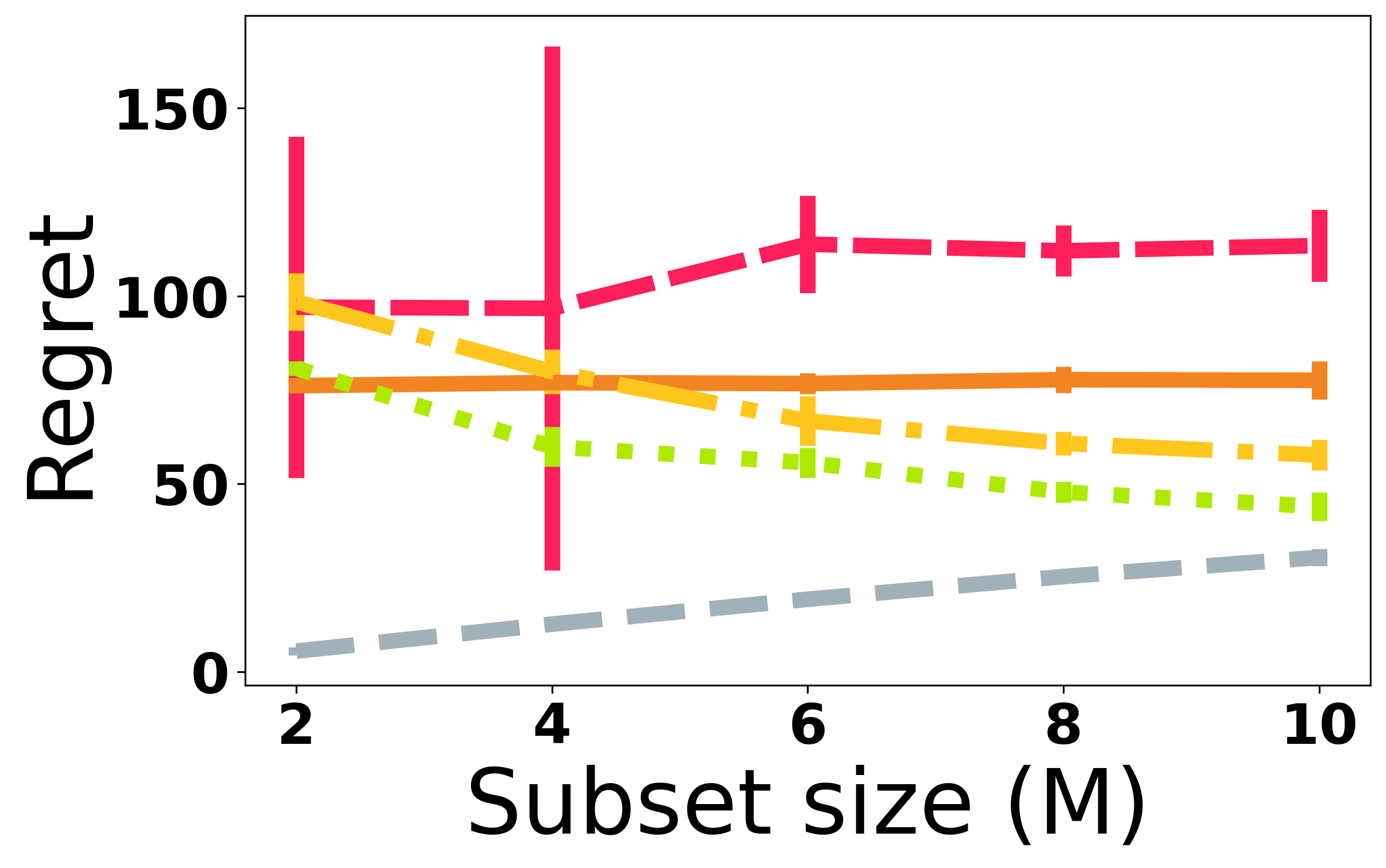}
    \end{minipage}
        \includegraphics[width=\textwidth]{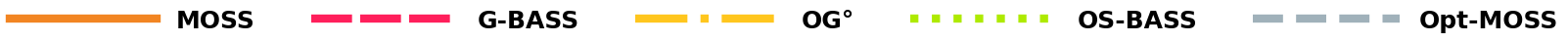}
    \vspace{-0.8cm}
    \caption{
    Default setting: $\setupOne$.
    In the right two plots $\tau=450$.
    Left to Right: Regret as a function of $N$ and $M$ under \cref{assumption:identification}. Regret as a function of $N$ and $M$ without \cref{assumption:identification}.
    }
    \label{fig:all}
\end{figure*}

\section{Discussion and Future Work}\label{sec:conclude}
We analyze a problem of $N$ tasks of $K$-armed bandits arriving sequentially, where the adversary is constrained to choose the optimal arm of each task in a smaller subset of $M$ arms. We consider two settings: bandit meta-learning and non-stationary bandits. We design an algorithm based on a reduction to bandit submodular optimization, and prove that its regret is $\tilde{O}(M N^{-1/3}T (\log K)^{1/3} + M^{3/2} \sqrt{N T})$. Under additional identifiability assumptions, we develop a meta-learning algorithm with an improved regret.

For the general non-stationary setting, we assume $N$ is known. Relaxing this assumption adaptively is left for further work. %

\bibliography{refs}

\begin{thebibliography}{50}
\providecommand{\natexlab}[1]{#1}
\providecommand{\url}[1]{\texttt{#1}}
\expandafter\ifx\csname urlstyle\endcsname\relax
  \providecommand{\doi}[1]{doi: #1}\else
  \providecommand{\doi}{doi: \begingroup \urlstyle{rm}\Url}\fi

\bibitem[Audibert and Bubeck(2009)]{AB-2009}
Jean-Yves Audibert and S\'{e}bastien Bubeck.
\newblock Minimax policies for adversarial and stochastic bandits.
\newblock In \emph{Proceedings of the 22nd Annual Conference on Learning Theory
  (COLT)}, 2009.

\bibitem[Auer and Ortner(2010)]{AO-2010}
P.~Auer and R.~Ortner.
\newblock Ucb revisited: Improved regret bounds for the stochastic multi-armed
  bandit problem.
\newblock \emph{Periodica Mathematica Hungarica}, pages 55–--65, 2010.

\bibitem[Auer et~al.(1995)Auer, Cesa-Bianchi, Freund, and Schapire]{ACFS-1995}
P.~Auer, N.~Cesa-Bianchi, Y.~Freund, and R.~E. Schapire.
\newblock Gambling in a rigged casino: The adversarial multi-armed bandit
  problem.
\newblock In \emph{36th Annual Symposium on Foundations of Computer Science},
  1995.

\bibitem[Auer et~al.(2002)Auer, Cesa-Bianchi, Freund, and Schapire]{ACFS-2002}
P.~Auer, N.~Cesa-Bianchi, Y.~Freund, and R.~E. Schapire.
\newblock The nonstochastic multiarmed bandit problem.
\newblock \emph{SIAM Journal on Computing}, 2002.

\bibitem[Auer et~al.(2019{\natexlab{a}})Auer, Chen, Gajane, Lee, Luo, Ortner,
  and Wei]{Auer-2019}
Peter Auer, Yifang Chen, Pratik Gajane, Chung-Wei Lee, Haipeng Luo, Ronald
  Ortner, and Chen-Yu Wei.
\newblock Achieving optimal dynamic regret for non-stationary bandits without
  prior information.
\newblock In \emph{COLT}, 2019{\natexlab{a}}.

\bibitem[Auer et~al.(2019{\natexlab{b}})Auer, Gajane, and
  Ortner]{ADSWITCH-auer19a}
Peter Auer, Pratik Gajane, and Ronald Ortner.
\newblock Adaptively tracking the best bandit arm with an unknown number of
  distribution changes.
\newblock In \emph{Proceedings of the Thirty-Second Conference on Learning
  Theory}, pages 138--158. PMLR, 25--28 Jun 2019{\natexlab{b}}.
\newblock URL \url{http://proceedings.mlr.press/v99/auer19a.html}.

\bibitem[Azar et~al.(2013)Azar, Lazaric, and Brunskill]{ALM-2013}
Mohammad~Gheshlaghi Azar, Alessandro Lazaric, and Emma Brunskill.
\newblock Sequential transfer in multi-armed bandit with finite set of models.
\newblock In \emph{NIPS}, 2013.

\bibitem[Balcan et~al.(2022)Balcan, Harris, Khodak, and Wu]{balcan2022meta}
Maria-Florina Balcan, Keegan Harris, Mikhail Khodak, and Zhiwei~Steven Wu.
\newblock Meta-learning adversarial bandits.
\newblock \emph{arXiv preprint arXiv:2205.14128}, 2022.

\bibitem[Baxter(2000)]{baxter2000model}
Jonathan Baxter.
\newblock A model of inductive bias learning.
\newblock \emph{Journal of artificial intelligence research}, 12:\penalty0
  149--198, 2000.

\bibitem[Berry et~al.(1997)Berry, Chen, Zame, Heath, and Shepp]{BCZHS-1997}
Donald~A. Berry, Robert~W. Chen, Alan Zame, David~C. Heath, and Larry~A. Shepp.
\newblock Bandit problems with infinitely many arms.
\newblock \emph{The Annals of Statistics}, 1997.

\bibitem[Bonald and Proutiere(2013)]{BP-2013}
T.~Bonald and A.~Proutiere.
\newblock Two-target algorithms for infinite-armed bandits with {B}ernoulli
  rewards.
\newblock In \emph{In Advances in Neural Information Processing Systems}, 2013.

\bibitem[Bousquet and Warmuth(2002)]{Bousquet-Warmuth-2002}
Olivier Bousquet and Manfred~K. Warmuth.
\newblock Tracking a small set of experts by mixing past posteriors.
\newblock \emph{Journal of Machine Learning Research}, 2002.

\bibitem[Carpentier and Valko(2015)]{CV-2015}
Alexandra Carpentier and Michal Valko.
\newblock Simple regret for infinitely many armed bandits.
\newblock In \emph{ICML}, 2015.

\bibitem[Caruana(1997)]{caruana1997multitask}
Rich Caruana.
\newblock Multitask learning.
\newblock \emph{Machine learning}, 28\penalty0 (1):\penalty0 41--75, 1997.

\bibitem[Cella et~al.(2020)Cella, Lazaric, and Pontil]{CLP-2020}
Leonardo Cella, Alessandro Lazaric, and Massimiliano Pontil.
\newblock Meta-learning with stochastic linear bandits.
\newblock \emph{Arxiv}, 2020.

\bibitem[Cesa-Bianchi and Lugosi(2006)]{PLG}
Nicolo Cesa-Bianchi and G\'abor Lugosi.
\newblock \emph{Prediction, learning, and games}.
\newblock Cambridge University Press, 2006.

\bibitem[Chan and Hu(2020)]{chan2020infinite}
Hock~Peng Chan and Shouri Hu.
\newblock Infinite arms bandit: Optimality via confidence bounds, 2020.

\bibitem[Chen et~al.(2019)Chen, Lee, Luo, and Wei]{CLLW-2019}
Yifang Chen, Chung-Wei Lee, Haipeng Luo, and Chen-Yu Wei.
\newblock A new algorithm for non-stationary contextual bandits: Efficient,
  optimal and parameter-free.
\newblock In \emph{COLT}, 2019.

\bibitem[Denevi et~al.(2018{\natexlab{a}})Denevi, Ciliberto, Stamos, and
  Pontil]{DCSP-2018}
Giulia Denevi, Carlo Ciliberto, Dimitris Stamos, and Massimiliano Pontil.
\newblock Learning to learn around a common mean.
\newblock In \emph{Advances in Neural Information Processing Systems},
  volume~31, 2018{\natexlab{a}}.

\bibitem[Denevi et~al.(2018{\natexlab{b}})Denevi, Ciliberto, Stamos, and
  Pontil]{denevi2018incremental}
Giulia Denevi, Carlo Ciliberto, Dimitris Stamos, and Massimiliano Pontil.
\newblock Incremental learning-to-learn with statistical guarantees,
  2018{\natexlab{b}}.

\bibitem[Denevi et~al.(2019)Denevi, Ciliberto, Grazzi, and Pontil]{DCGP-2019}
Giulia Denevi, Carlo Ciliberto, Riccardo Grazzi, and Massimiliano Pontil.
\newblock Learning-to-learn stochastic gradient descent with biased
  regularization.
\newblock In \emph{Proceedings of the 36th International Conference on Machine
  Learning}, 2019.

\bibitem[Dimakopoulou et~al.(2019)Dimakopoulou, Vlassis, and Jebara]{DVJ-2019}
M.~Dimakopoulou, N.~Vlassis, and T.~Jebara.
\newblock Marginal posterior sampling for slate bandits.
\newblock In \emph{IJCAI}, 2019.

\bibitem[Feige et~al.(2004)Feige, Lov{\'a}sz, and
  Tetali]{feige2004approximating}
Uriel Feige, L{\'a}szl{\'o} Lov{\'a}sz, and Prasad Tetali.
\newblock Approximating min sum set cover.
\newblock \emph{Algorithmica}, 40\penalty0 (4):\penalty0 219--234, 2004.

\bibitem[Franceschi et~al.(2018)Franceschi, Frasconi, Salzo, Grazzi, and
  Pontil]{franceschi2018bilevel}
Luca Franceschi, Paolo Frasconi, Saverio Salzo, Riccardo Grazzi, and
  Massimilano Pontil.
\newblock Bilevel programming for hyperparameter optimization and
  meta-learning, 2018.

\bibitem[Gentile et~al.(2014)Gentile, Li, and Zappella]{GLZ-2014}
Claudio Gentile, Shuai Li, and Giovanni Zappella.
\newblock Online clustering of bandits.
\newblock In \emph{Proceedings of the 31st International Conference on Machine
  Learning}, 2014.

\bibitem[Hong et~al.(2020{\natexlab{a}})Hong, Kveton, Zaheer, Chow, Ahmed, and
  Boutilier]{HKZCAB-2020}
Joey Hong, Branislav Kveton, Manzil Zaheer, Yinlam Chow, Amr Ahmed, and Craig
  Boutilier.
\newblock Latent bandits revisited.
\newblock In \emph{NeurIPS}, 2020{\natexlab{a}}.

\bibitem[Hong et~al.(2020{\natexlab{b}})Hong, Kveton, Zaheer, Chow, Ahmed,
  Ghavamzadeh, and Boutilier]{HKZCAGB-2020}
Joey Hong, Branislav Kveton, Manzil Zaheer, Yinlam Chow, Amr Ahmed, Mohammad
  Ghavamzadeh, and Craig Boutilier.
\newblock Non-stationary latent bandits.
\newblock \emph{arXiv}, 2020{\natexlab{b}}.

\bibitem[Kale et~al.(2010)Kale, Reyzin, and Schapire]{KRS-2010}
S.~Kale, L.~Reyzin, and R.~E. Schapire.
\newblock Non-stochastic bandit slate problems.
\newblock In \emph{NIPS}, 2010.

\bibitem[Kalvit and Zeevi(2020)]{Kalvit-Zeevi-2020}
Anand Kalvit and Assaf Zeevi.
\newblock From finite to countable-armed bandits.
\newblock In \emph{Conference on Neural Information Processing Systems}, 2020.

\bibitem[Khodak et~al.(2019)Khodak, Balcan, and Talwalkar]{khodak2019provable}
Mikhail Khodak, Maria-Florina Balcan, and Ameet Talwalkar.
\newblock Provable guarantees for gradient-based meta-learning.
\newblock In \emph{Proceedings of the 36th International Conference on Machine
  Learning}, pages 424--433, 2019.

\bibitem[Kong et~al.(2020)Kong, Somani, Song, Kakade, and Oh]{KSSKO-2020}
Weihao Kong, Raghav Somani, Zhao Song, Sham Kakade, and Sewoong Oh.
\newblock Meta-learning for mixed linear regression.
\newblock In \emph{Proceedings of the 37th International Conference on Machine
  Learning}, 2020.

\bibitem[Kveton et~al.(2020)Kveton, Mladenov, Hsu, Zaheer, Szepesv\'{a}ri, and
  Boutilier]{kveton-2020}
Branislav Kveton, Martin Mladenov, Chih-Wei Hsu, Manzil Zaheer, Csaba
  Szepesv\'{a}ri, and Craig Boutilier.
\newblock Differentiable meta-learning in contextual bandits.
\newblock \emph{arXiv:2006.05094v1}, 2020.

\bibitem[Kveton et~al.(2021)Kveton, Konobeev, Zaheer, wei Hsu, Mladenov,
  Boutilier, and Szepesvari]{KKZHMBS-2021}
Branislav Kveton, Mikhail Konobeev, Manzil Zaheer, Chih wei Hsu, Martin
  Mladenov, Craig Boutilier, and Csaba Szepesvari.
\newblock Meta-thompson sampling.
\newblock \emph{Arxiv}, 2021.

\bibitem[Kwon et~al.(2017)Kwon, Perchet, and
  Vernade]{Kwon-Perchet-Vernade-2017}
Joon Kwon, Vianney Perchet, and Claire Vernade.
\newblock Sparse stochastic bandits.
\newblock In \emph{2017 Conference on Learning Theory (COLT)}, volume~65, 2017.

\bibitem[Lattimore and Szepesv{\'a}ri(2020)]{lattimore-Bandit}
T.~Lattimore and C.~Szepesv{\'a}ri.
\newblock \emph{Bandit Algorithms}.
\newblock Cambridge University Press, 2020.
\newblock ISBN 9781108687492.
\newblock URL \url{https://books.google.com/books?id=xe3vDwAAQBAJ}.

\bibitem[Lattimore and Szepesvari(2020)]{LS-2020}
Tor Lattimore and Csaba Szepesvari.
\newblock \emph{Bandit Algorithms}.
\newblock Cambridge University Press, 2020.

\bibitem[Littlestone and Warmuth(1994)]{LITTLESTONE-1994-RMW}
N.~Littlestone and M.K. Warmuth.
\newblock The weighted majority algorithm.
\newblock \emph{Information and Computation}, 108\penalty0 (2):\penalty0
  212--261, 1994.
\newblock ISSN 0890-5401.
\newblock \doi{https://doi.org/10.1006/inco.1994.1009}.
\newblock URL
  \url{https://www.sciencedirect.com/science/article/pii/S0890540184710091}.

\bibitem[Maillard and Mannor(2014)]{MM-2014}
Odalric-Ambrym Maillard and Shie Mannor.
\newblock Latent bandits.
\newblock In \emph{ICML}, 2014.

\bibitem[Merlis and Mannor(2019)]{MM-2019}
Nadav Merlis and Shie Mannor.
\newblock Batch-size independent regret bounds for the combinatorial
  multi-armed bandit problem.
\newblock In \emph{COLT}, 2019.

\bibitem[Osadchiy et~al.(2022)Osadchiy, Levy, and Meir]{osadchiy2022online}
Ilya Osadchiy, Kfir~Y Levy, and Ron Meir.
\newblock Online meta-learning in adversarial multi-armed bandits.
\newblock \emph{arXiv preprint arXiv:2205.15921}, 2022.

\bibitem[Park et~al.(2021)Park, Shin, Jun, and Ok]{PSJO-2021}
Hyejin Park, Seiyun Shin, Kwang-Sung Jun, and Jungseul Ok.
\newblock Transfer learning in bandits with latent continuity.
\newblock \emph{arXiv}, 2021.

\bibitem[Radlinski et~al.(2008)Radlinski, Kleinberg, and Joachims]{RKJ-2008}
Filip Radlinski, Robert Kleinberg, and Thorsten Joachims.
\newblock Learning diverse rankings with multi-armed bandits.
\newblock In \emph{ICML}, 2008.

\bibitem[Rhuggenaath et~al.(2020)Rhuggenaath, Akcay, Zhang, and
  Kayma]{RAZK-2020}
Jason Rhuggenaath, Alp Akcay, Yingqian Zhang, and Uzay Kayma.
\newblock Algorithms for slate bandits with non-separable reward functions.
\newblock \emph{arXiv}, 2020.

\bibitem[Russac et~al.(2019)Russac, Vernade, and
  Capp\'{e}]{Russac-Vernade-Cappe-2019}
Yoan Russac, Claire Vernade, and Olivier Capp\'{e}.
\newblock Weighted linear bandits for non-stationary environments.
\newblock In \emph{NeurIPS}, 2019.

\bibitem[Streeter and Golovin(2007)]{streeter2007online}
Matthew Streeter and Daniel Golovin.
\newblock An online algorithm for maximizing submodular functions.
\newblock Technical report, Carnegie Mellon University, Pittsburgh, PA, School
  of computer science, 2007.
\newblock URL \url{https://apps.dtic.mil/sti/pdfs/ADA476871.pdf}.
\newblock Available at \url{https://apps.dtic.mil/sti/pdfs/ADA476871.pdf}.

\bibitem[Thrun(1996)]{thrun1996learning}
Sebastian Thrun.
\newblock Is learning the n-th thing any easier than learning the first?
\newblock In \emph{Advances in neural information processing systems}, pages
  640--646. MORGAN KAUFMANN PUBLISHERS, 1996.

\bibitem[Tripuraneni et~al.(2021)Tripuraneni, Jin, and Jordan]{TJJ-2021}
Nilesh Tripuraneni, Chi Jin, and Michael~I. Jordan.
\newblock Provable meta-learning of linear representations.
\newblock \emph{Arxiv}, 2021.

\bibitem[Wang et~al.(2008)Wang, Audibert, and Munos]{WAM-2008}
Yizao Wang, Jean-Yves Audibert, and R\'{e}mi Munos.
\newblock Algorithms for infinitely many-armed bandits.
\newblock In \emph{NIPS}, 2008.

\bibitem[Wei and Luo(2021)]{WL-2021}
Chen-Yu Wei and Haipeng Luo.
\newblock Non-stationary reinforcement learning without prior knowledge: An
  optimal black-box approach.
\newblock \emph{COLT}, 2021.

\bibitem[Zheng et~al.(2019)Zheng, Luo, Diakonikolas, and Wang]{ZLDW-2019}
Kai Zheng, Haipeng Luo, Ilias Diakonikolas, and Liwei Wang.
\newblock Equipping experts/bandits with long-term memory.
\newblock In \emph{NeurIPS}, 2019.

\end{thebibliography}

\newpage
\appendix
\onecolumn

\newpage

\section{Submodular functions}\label{app:submod}

In this section, we verify that $f(S)\doteq f(r,S)=\max_{a\in S}r(a)$ is a submodular function. It is obviously monotone. We have
\begin{align*}
    f(S_1\cup S_2\cup \{a\})-f(S_1\cup S_2)\leq f(S_1\cup \{a\})-f(S_1)
\end{align*}
This is because (1) if $r(a)\leq f(S_1)$ then the inequality holds as $f(S_1\cup \{a\})=f(S_1)$ and $f(S_1\cup S_2\cup \{a\})=f(S_1\cup S_2)$. (2) if $r(a) > f(S_1)$ then (2.i) if $r(a)>f(S_1\cup S_2)$ the inequality holds as $r(a)-f(S_1\cup S_2)\leq r(a)-f(S_1)\iff f(S_1\cup S_2)\geq f(S_1)$, which holds by monotonicity of $f$,  (2.ii) if $r(a)\leq f(S_1\cup S_2)$ then $f(S_1\cup S_2\cup \{a\})=f(S_1\cup S_2)$ and the inequality simplifies to $0\leq r(a)-f(S_1)$ which holds as we assumed $r(a) > f(S_1)$ in (2).

\section{Proofs of Section~\ref{sec:subset}}
\label{app:metalearning2subset}

\begin{proof}[Proof of Lemma~\ref{lemma:metalearning2subset}]
We have that
\begin{align*}
R_T &= \sup_{\substack{r_{1},\dots,r_{N}\\ \in [0,1]^{K}}} \max_{\substack{a_1,\dots, a_N\in [K],\\ |\{a_1,\dots,a_N\}|\le M}}\; \E\left[\sum_{n=1}^N \sum_{t=1}^{\tau_n} (r_n(a_n) - r_{n}(A_{n,t}))\right] 
\\
&= \sup_{\substack{r_{1},\dots,r_{N}\\ \in [0,1]^{K}}} \max_{S \in\cS}\; \E\left[ \sum_{n=1}^{N'} \sum_{u=1}^{N_{n}} \sum_{t=1}^{\tau_{n,u}} (\max_{a\in S}(r_{n,u}(a) - r_{n,u}(A_{n,u,t})) \right] 
\\
&= \sup_{\substack{r_{1},\dots,r_{N}\\ \in [0,1]^{K}}} \max_{S \in\cS}\; \E\left[ \sum_{n=1}^{N'} \sum_{u=1}^{N'_n} \left(\sum_{t=1}^{\tau_{n,u}} (\max_{a\in S}r_{n,u}(a) - \max_{a\in S_{n}}r_{n,u}(a))+ \sum_{t=1}^{\tau_{n,u}} (\max_{a\in S_{n}}r_{n,u}(a) - r_{n,u}(A_{n,u,t}))
\right)\right] 
\\
&\le \sup_{(f_{n,u}\in \cF)_{n\in[N'],u\in[N_n]} } \max_{S\in \cS}\; \E\left[ \sum_{n=1}^{N'} \left( \sum_{u=1}^{N_n}  \tau_{n,u} (f_{n,u}(S) - f_{n,u}(S_{n})) + \tau'\varepsilon_{n}\right)\right]
\\
&= \sup_{f_1,\cdots,f_{N'} \in \cF} \max_{S\in \cS}\; \E\left[\sum_{n=1}^{N'} \left(f_n(S) - f_n (S_{n}) + B_{\tau',N_n,M} \right) \right] \;.
\end{align*}
\end{proof}

\section{Proofs of Section~\ref{sec:unknown}}

First, we present the relevant results from \citet{streeter2007online} with appropriate modifications. We start with the regret analysis of the \OG algorithm, which is designed to solve the online submodular maximization in the full feedback model. 

\subsection{The \OG algorithm}\label{sec:OGBase}

For a submodular function $g$, consider an ordered set of actions $\bar{S}=\langle\bar{a}_1,\bar{a}_2,\cdots\rangle$ that satisfies the following \emph{greedy} condition for any $j$: %
\begin{align}
g(\bar{S}_j\cup \bar{a}_j)-g(\bar{S}_j) \geq \max_{a\in[K]}\{g(\bar{S}_j\cup a)-g(\bar{S}_j) \}-\alpha_j,
 \label{eq:grd}
\end{align}
where $\alpha_j$ are some positive error terms. Let $\bar{S}_0=\emptyset$, $\bar{S}_j=\langle\bar{a}_1,\bar{a}_2,\cdots,\bar{a}_{j-1}\rangle$,
and for a sequence of actions $S \subset [K]$, let $S_{\langle M\rangle}$ be the set of actions in $S$ truncated at the $M$'th action.
The following result shows near-optimality of $\bar{S}$ as constructed above. Recall that $\Mlogn=\mppr$.

\begin{theoremi}[Based on \citet{streeter2007online}, Theorem~6]
\label{thm:6}
Consider the greedy solution in \cref{eq:grd}. Then
\[
g(\bar{S}_{\langle \Mlogn\rangle})> \left(1-\frac{1}{N'} \right)\max_{S\in\ca{S}}g(S_{\langle M\rangle})-\sum_{j=1}^{\Mlogn}\alpha_j \;.
\]
\end{theoremi}
\begin{proof}
Let $C^*=\max_{S\in\ca{S}}g(S_{\langle M\rangle})$ and for any $j$ let $\Delta_j=C^*-g(\bar{S}_j)$. Then, by \Cref{fact2} below, we have $C^*\leq g(\bar{S}_j)+M(s_j+\alpha_j)$. Therefore, $\Delta_j\leq M(s_j+\alpha_j)=M(\Delta_j-\Delta_{j+1}+\alpha_j)$ which means $\Delta_{j+1}\leq \Delta_j(1-\tfrac{1}{M})+\alpha_j$. Unrolling this $\Mlogn$ times (and noting $1-\tfrac{1}{N'}<1$) gives
\begin{align*}
    \Delta_{\Mlogn+1} &\leq \Delta_1\left(1-\frac{1}{M}\right)^{\Mlogn}+\sum_{j=1}^{\Mlogn}\alpha_j
    \\
    &<\Delta_1\tfrac{1}{N'}+\sum_{j=1}^{\Mlogn}\alpha_j
    \leq C^*\tfrac{1}{N'}+\sum_{j=1}^{\Mlogn}\alpha_j.
\end{align*}
This concludes the proof since $C^*-g(\bar{S}_{\Mlogn+1})=\Delta_{\Mlogn+1}$ and $g(\bar{S}_{\Mlogn+1})=g(\bar{S}_{\langle \Mlogn\rangle})$. 
\end{proof}

\begin{fact}[\citet{streeter2007online}, Fact 1]\label{fact2}
For any subset of arms $S$, and any positive integer $j$, and any $t>0$, we have $g(S_{\langle t\rangle})\leq g(\bar{S}_j)+t(s_j+\alpha_j)$ where $s_j=g(\bar{S}_{j+1})-g(\bar{S}_j)$.
\end{fact}
\begin{proof} The proof is akin to Fact 1 of \citet{streeter2007online} and it goes
$
g(S_{\langle t\rangle})\leq g(\bar{S}_j\cup S_{\langle t\rangle}) \leq g(\bar{S}_{j})+t (s_j+\alpha_j)
$. The first inequality holds because $g$ is a monotone function. The second inequality is by definition of $s_j$ and Condition 1, \citep[Lemma 1]{streeter2007online} -- for any submodular function $g$, and any $S_1,S\in \ca{S}$, $\frac{g(S_1\cup S)-g(S_1)}{|S|}\leq \max_{a\in\A}g(S_1\cup \{a\})-g(S_1)$
and we replace $S_1\gets \bar{S}_j$, $S\gets S_{\langle t\rangle}$, so $|S|=t$. 
\end{proof}

\begin{algorithm}[tb]
    \caption{\OG algorithm}
    \label{alg:OG}
    \begin{algorithmic}[1]
        \STATE {\bfseries Input}: Subset size $M$, {\tt Expert} algorithms $\mathcal{E}_1,\cdots,\mathcal{E}_{\Mlogn}$;
        \FOR{$n \in[N']$}%
        {
            \STATE Let $S_{n,0}=\emptyset$
            \FOR{$j \in\{1,\cdots,\Mlogn\}$}{
            \STATE Let action $a_j^n\in[K]$ be the choice of expert $\ca{E}_j$
            \STATE Let $S_{n,j}\gets S_{n,j-1}\cup \{a_j^n\}$
            }\ENDFOR
            \STATE Play subset $S_n$ and observe function $g_n$.
            \FOR{$j \in\{1,\cdots,\Mlogn\}$ and any $a\in [K]$}{
            \STATE Let $x_{j,a}^n\gets g_n(S_{n,j-1}\cup\{a\})-g_n(S_{n,j-1})$ 
            \STATE Expert $\ca{E}_j$ receives payoff vector $(x_{j,a}^n)_{a\in [K]}$ 
            }\ENDFOR
        }\ENDFOR
  \end{algorithmic}
\end{algorithm}

Consider a sequence of submodular functions $g_1,\cdots,g_{N'}$ for a fixed $N'\in\mathbb{N}$. Define the coverage regret of a submodular maximization policy by
\begin{align*}
\Rcov(N'):= \left(1-\frac{1}{N'} \right)\underset{S\in\ca{S}}{\max}\sum_{n=1}^{N'}g_n(S_{\langle M\rangle})-\sum_{n=1}^{N'}g_n(S_n)\;.
\end{align*}

\cref{alg:OG} is the \OG algorithm of \citet{streeter2007online} for the full feedback model modified for our setting with $\Mlogn$ experts. In this algorithm, $N'$ is the number of rounds (analogous to segments/tasks). The algorithm uses a set of experts and each expert is a \emph{randomized weighted majority} (RWM) algorithm~\citep{LITTLESTONE-1994-RMW}. See Chapter 4.2 of \citet{PLG} for more information.
The following lemma connects the coverage regret of the \OG algorithm and the regret of the experts.
\begin{lemmai}[Lemma~3 of \citet{streeter2007online}]
\label{lem:RcovR-1}
Let $G_j(N')$ be the cumulative regret of expert $\ca{E}_j$ in \OG algorithm, and
let $G(N')=\sum_{j=1}^{\Mlogn} G_j(N')$.
Then, $\Rcov(N')\leq G(N')$.
\end{lemmai}
\begin{proof}
As we will show, the \OG algorithm is an approximate version of the offline greedy subset selection, defined by \cref{eq:grd}, for function $g =\frac{1}{N'}\sum_{n=1}^{N'} g_n$. First, let's view the sequence of actions selected by $\ca{E}_j$ as a single “batch-action” $\Tilde{a}_j$, and extend the domain of each $g_n$ to include the batch-actions by defining $g_n(S \cup \{\Tilde{a}_j\}) = g_n(S \cup \{a_j^n\})$ for all $S \in \ca{S}$. Thus, the online algorithm produces a single set $\Tilde{S} =\{\Tilde{a}_1, \Tilde{a}_2,\cdots,\Tilde{a}_{\Mlogn}\}$. By definition we have
\begin{align*}
    \frac{G_j(N')}{N'}=\max_{a\in[K]}\left(g(\Tilde{S}_{\langle j-1\rangle}\cup \{a\})-g(\Tilde{S}_{\langle j-1\rangle})\right)-\left(g(\Tilde{S}_{\langle j-1\rangle}\cup \{\Tilde{a}_j\})-g(\Tilde{S}_{\langle j-1\rangle})\right)\,,
\end{align*}
where $\Tilde{S}_{\langle j\rangle}$ is $\Tilde{S}$ truncated at $j$'th action. Thus the \OG algorithm  simulates the greedy schedule \eqref{eq:grd} for function $g$, where the $j$'th decision is made with additive error $\alpha_j=\frac{G_j(N')}{N'}$. By \cref{thm:6} and the fact that function $g$ is submodular, we get that $\Rcov(N')\leq\sum_{j=1}^{\Mlogn} G_j(N') = G(N')$.
\end{proof}
By Lemma 4 of \citet{streeter2007online}, $\E[G(N')] =O(\sqrt{\Mlogn N'\log(K)})$.

\subsection{The \OGO algorithm}\label{sec:OGpfs}

\cref{alg:OGo} is based on the \OGO algorithm of \citet{streeter2007online} for the bandit (opaque) feedback model. This algorithm is very similar to \BOG algorithm so we omit the description. The difference is that in the \BOG algorithm, the meta-learner observes the value of the submodular function up to a noise term $\varepsilon_n=(1/\tau') B_{\tau',N_n,\Mlogn}$. So we extend the analysis of \citet{streeter2007online} to the case that the observation of the submodular function is corrupted by a noise term.

\begin{algorithm}[tb]
    \caption{\OGO algorithm}
    \label{alg:OGo}
    \begin{algorithmic}[1]
    \STATE {\bfseries Input}: Subset size $M$, {\tt Expert} algorithms $\mathcal{E}_1,\cdots,\mathcal{E}_{\Mlogn}$, Probabilities of exploration $\{\gamma_n\}_{n=1}^{N'}$;
    \FOR{$n \in [N']$}%
    \STATE Observe $g_n$
    \STATE For $i\in[\Mlogn]$, let $a_i$ be the choice of $\mathcal{E}_i$
    \STATE Set $S_{n}=\{a_1,\cdots,a_{\Mlogn}\}$
    \STATE With prob. $\gamma_n$, $E_{n} = \exr$, otherwise $E_{n} = \ext$
        \IF{$E_{n} = \ext$}
            \STATE All experts receive the zero vector as the payoff vector
        \ELSE
            \STATE Choose $i\in[\Mlogn]$ uniformly at random
            \STATE Choose a new action $a'_{i}$ uniformly at random 
            \STATE Replace $i$'th element of $S_{n}$ with $a'_{i}$: $S_{n:i}\gets\{a_1,\cdots,a_{i-1},a'_i\}$
            \STATE Expert $\ca{E}_i$ receives a payoff vector that is zero everywhere except at position $a'_{i}$ that has the value of $g_n(S_{n:i})$ 
            \STATE All other experts receive the zero vector as the payoff vector
        \ENDIF
    \ENDFOR
    \end{algorithmic}
\end{algorithm}

\begin{lemmai}

\label{lem:5}
Consider an expert prediction problem with $K$ actions, and let $x_a^n$ be the payoff for action $a\in [K]$ in round $n$. Let $\ca{E}$ be a 
an expert algorithm that gets payoff vector $(x_a^n)_{a\in [K]}$ in round $n$, let $e_n$ be its action in round $n$, and let $R(N')$ be its worst-case expected regret over $N'$ rounds: $R(N')=\max_{a} \sum_{n=1}^{N'} (x_a^n - x_{e_n}^n)$. Let $\ca{E}'$ and $\Tilde{\ca{E}}$ be the same algorithm but with payoff vectors $(\hat{x}_a^n)_{a\in [K]}$ and $(\Tilde{x}_a^n)_{a\in [K]}$ instead of $(x_a^n)_{a\in [K]}$. These feedbacks are such that $\E[\hat{x}_a^n]=\gamma_n x_a^n+\delta_n$ for some constant $\gamma_n\in[0,1]$ and $\delta_n$, and
\[
\E[\Tilde{x}_a^n]\in[ \E[\hat{x}_a^n-\gamma_n\varepsilon'_n], \E[\hat{x}_a^n]]
\]
for some $\varepsilon'_n\geq 0$. Let $u_n$ be the action of algorithm $\Tilde{\ca{E}}$ in round $n$. 
Then the worst-case expected regret of $\Tilde{\ca{E}}$ is bounded as
\[
\max_{a} \sum_{n=1}^{N'} (x_a^n - x_{u_n}^n) \le \frac{1}{\min_n\gamma_n} R(N')+\sum_{n=1}^{N'}\E[\varepsilon'_n] \;.
\]
\end{lemmai}
\begin{proof}

By the regret guarantee of the expert algorithm,
\begin{align*}
R(N') \geq \underset{a\in[K]}{\max} \sum_{n=1}^{N'} (\Tilde{x}_{a}^n-\Tilde{x}_{u_n}^n) \;.
\end{align*}
Thus, for any $a$, 
\begin{align*}
R(N')  &\geq \E\left[\sum_{n=1}^{N'}\Tilde{x}_a^n- \Tilde{x}_{u_n}^{n}|\Tilde{x}\right] \\
&\geq \sum_{n=1}^{N'}\E[\hat{x}_a^n-\gamma_n\varepsilon'_n- \hat{x}_{u_n}^{n}] \\
&= \sum_{n=1}^{N'}\E[\gamma_n x_a^n+\delta_n-\gamma_n\varepsilon'_n- \gamma_n x_{u_n}^{n}-\delta_n] 
\\
&=  \sum_{n=1}^{N'} \gamma_n \E[x_a^n-\varepsilon'_n- x_{u_n}^{n}] 
\\
&\geq \min_n \gamma_n \sum_{n=1}^{N'} \E[x_a^n-\varepsilon'_n- x_{u_n}^{n}] 
\;.
\end{align*}
Therefore,
\[
\frac{1}{\min_n\gamma_n}\, R(N') +\sum_{n=1}^{N'}\E[\varepsilon'_n] \geq \sum_{n=1}^{N'}\E[x_a^n- x_{u_n}^{n}] \,,
\]   
and the result follows as the above inequality holds for all $a$.
\end{proof}

The next lemma bounds the coverage regret of the \OGO algorithm when $N'$ (i.e. $N$) is known.
\begin{lemmai}[Coverage Regret for known $N'$]%
\label{lem:Rcov}
Let $\gamma_n=\gammaOne$ for all $n$. Assume the $j$th expert $\Tilde{\ca{E}}_j$ in the \OGO algorithm gets a payoff vector $(\Tilde{x}_a^n)_{a\in [K]}$ in round $n$ such that the following holds:
\[
 \gamma'\left(g_n(S_{n:j-1}\cup\{a\})-\varepsilon'_n\right)\leq \E[\Tilde{x}_a^n]\leq \gamma g_n(S_{n:j-1}\cup\{a\}) \;.
\]
where $\gamma = \gammaOne$ and $\gamma'= \frac{\gamma}{\Mlogn K}$. Then for the sequence of subsets $(S_n)_{n=1}^{N'}$ chosen by the \OGO algorithm,
\[
R_{\text{coverage}}(N')\leq 
(\Mlogn^4 N'^2 K\log k)^{1/3}
+\Mlogn\sum_{n=1}^{N'}\E[\varepsilon'_n] \;.
\]
\end{lemmai}
\begin{proof}
We start with another expert $\ca{E}'$ that gets payoff vector $\hat{x}^n$ such that $\E[\hat{x}_a^n]=\gamma' g_n(S_{n-1:j}\cup\{a\})$ for any action $a$. Then we can write
\begin{align*}
    \E[\hat{x}_a^n]=\gamma' x_a^n+\delta_n
\end{align*}
for $x_a^n= (g_n(S_{n-1:j}\cup\{a\})-g_n(S_{n-1:j}))$ and $\delta_n=\gamma'  g_n(S_{n-1:j})$, where $\gamma'= \frac{\gamma}{\Mlogn K}$. Let $N'_{\exr}$ be the number of exploration rounds. Let $G'_j(N')$ be the total regret of expert $\ca{E}'_j$. By \cref{lem:5} and the regret guarantee of the expert algorithm, the total regret of expert $\ca{E}'_j$ is bounded as
\begin{align*}
\E[G'_j(N')]&\leq \frac{1}{\gamma'} \E \sqrt{\left(\max_{a} \sum_{n=1}^{N'} \hat{x}_a^n \right) \log K}
\\
&\leq \frac{1}{\gamma'} \E \sqrt{N'_{\exr} \log K} 
\\
&\leq \sqrt{\frac{N'}{\gamma'} \log K}
\,,
\end{align*}
where we used Jensen's inequality and $\E[N'_{\exr}]=\gamma' N'$ in the last step. 
Let $\Tilde{G}_j(N')$ be regret of expert $\Tilde{\ca{E}}_j$. 
We observe that $\E[\Tilde{x}_a^n]\in[\E[\hat{x}_a^n-\gamma'\varepsilon'_n],\E[\hat{x}_a^n]]$. Given that the \OGO algorithm takes random actions in the exploration rounds, it incurs an extra $\gamma'N'$ regret, and therefore together with \cref{lem:5}, we have $\E[\Tilde{G}_j(N')] \le \E[G'_j(N')] +\sum_{n=1}^{N'}\E[\varepsilon'_n] +  \gamma N' $. By summing over $j\in[\Mlogn]$, 
\[
\E\left[\sum_{j=1}^{\Mlogn} \Tilde{G}_j(N') \right] \le {\Mlogn}\sqrt{\frac{N'}{\gamma'}\log K}+{\Mlogn}\sum_{n=1}^{N'}\E[\varepsilon'_n] + \gamma {\Mlogn} N' \;.
\]
By \cref{lem:RcovR-1} we get
\begin{align*}
    \E[\Rcov(N')] &\leq \Mlogn \sqrt{\frac{N'}{\gamma'}\log K}+\Mlogn\sum_{n=1}^{N'}\E[\varepsilon'_n] + \gamma \Mlogn N' 
    \\
    &= \Mlogn \sqrt{\frac{N'}{\gamma} \Mlogn K\log K}+\Mlogn\sum_{n=1}^{N'}
    \E[\varepsilon'_n] + \gamma \Mlogn N'\;.
\end{align*}

Finally, choosing $\gamma=\gammaOne$\footnote{To be more precise, $\gamma=(3/2)\gammaOne$.} yields
\begin{align*}
    \Mlogn \sqrt{\frac{N'}{\gamma} \Mlogn K\log K} \! +\! \gamma \Mlogn N'
    = (\Mlogn^4 N'^2 K\log k)^{1/3}
    \;. 
\end{align*}
Therefore
\[
\E[\Rcov(N')] \le  
(\Mlogn^4 N'^2 K\log k)^{1/3}
\! +\!\Mlogn\sum_{n=1}^{N'}\E[\varepsilon'_n] \;.
\]
\end{proof}\jtodo{fix gammaOne in icml}

A similar result holds for the case of unknown $N$ by using a time-varying exploration probability. %
\begin{lemmai}[Coverage regret for unknown $N$]
\label{lem:Rcov-N}
Let $\gamma_n=
\gammaOnen
$ and assume the $j$th expert $\Tilde{\ca{E}}_j$ in the \OGO algorithm gets a payoff vector $(\Tilde{x}_a^n)_{a\in [K]}$ in round $n$ such that the following holds:
\[
 \gamma'_n\left(g_n(S_{n:j-1}\cup\{a\})-\varepsilon'_n\right)\leq \E[\Tilde{x}_a^n]\leq \gamma'_n g_n(S_{n:j-1}\cup\{a\}) \,,
\]
where $\gamma_n=\gammaOnen$ and $\gamma'_n= \frac{\gamma_n}{\Mlogn K}$. Then for the sequence of subsets $(S_n)_{n=1}^{N'}$ chosen by the \OGO algorithm,
\[
R_{\text{coverage}}(N')\le \Mlogn(\log K)^{1/3}N'^{2/3}+\Mlogn\sum_{n=1}^{N'}\E[\varepsilon'_n] \;.
\]
\end{lemmai}
\begin{proof}
The proof is analogous to \cref{lem:Rcov}. Let expert $\ca{E}'$ be an expert that gets payoff vector $\hat{x}^n$ such that $\E[\hat{x}_a^n]=\gamma'_n g_n(S_{n-1:j}\cup\{a\})$ for any action $a$, then $\E[\hat{x}_a^n]=\gamma'_n x_a^n+\delta_n$ for $x_a^n= (g_n(S_{n-1:j}\cup\{a\})-g_n(S_{n-1:j}))$ and $\delta_n=\gamma'_n  g_n(S_{n-1:j})$, where $\gamma'_n= \frac{\gamma_n}{\Mlogn K}$. Let $N'_{\exr}$ and $G'_j(N')$ be the same as in \cref{lem:Rcov}. Akin to \cref{lem:Rcov}, by \cref{lem:5} and the regret guarantee of the expert algorithm, we know
\begin{align*}
\E[G'_j(N')]&\leq \frac{1}{\min_n\gamma'_n} \E \sqrt{N'_{\exr} \log K} 
\\
&\leq \frac{1}{\min_n\gamma'_n}\sqrt{\sum_{n=1}^{N'} \gamma'_n \log K}
\,,
\end{align*}
where we used Jensen's inequality and $\E[N'_{\exr}]= \sum_{n=1}^{N'} \gamma_n$ in the last step. 
Again $\Tilde{G}_j(N')$ is the regret of expert $\Tilde{\ca{E}}_j$. 
It is easy to see that $\E[\Tilde{x}_a^n]\in[\E[\hat{x}_a^n-\gamma'\varepsilon'_n],\E[\hat{x}_a^n]]$. The exploration rounds regret is $\leq \sum_{n=1}^{N'} \gamma'_n$. Therefore, by \cref{lem:5} it holds that $\E[\Tilde{G}_j(N')] \le \E[G'_j(N')] +  \sum_{n=1}^{N'}\gamma_n+\sum_{n=1}^{N'}\E[\varepsilon'_n]$. Summing over $j\in[\Mlogn]$ gives 
\[
\E\left[\sum_{j=1}^{\Mlogn} \Tilde{G}_j(N') \right] \le 
\frac{\Mlogn}{\min_n \gamma'_n} \sqrt{\sum_{n=1}^{N'}\gamma'_n\log K}+\Mlogn\sum_{n=1}^{N'}\E[\varepsilon'_n] +  \Mlogn \sum_{n=1}^{N'} \gamma_n \;.
\]
We know $\E[\Rcov(N')]\le \E\left[\sum_{j=1}^{\Mlogn} \Tilde{G}_j(N') \right] $ by \cref{lem:RcovR-1}, i.e., using the definition of $\gamma'_n$ gives
\begin{align*}
    \E[\Rcov(N')] &\leq \frac{\Mlogn^2 K}{\min_n \gamma_n} \sqrt{\frac{\sum_{n=1}^{N'}\gamma_n}{\Mlogn K} \log K }+\Mlogn\sum_{n=1}^{N'}\E[\varepsilon'_n] + \Mlogn \sum_{n=1}^{N'} \gamma_n \;.
\end{align*}
Now, $\gamma_n =\gammaOnen$ gives the same result as in \cref{lem:Rcov} noting that $\sum_{n=1}^{N'}(1/n)^{1/3}\leq \int_{x=1}^{N'}(1/x)^{1/3}=(N')^{2/3}$. Therefore
\[
\E[\Rcov(N')] \le  (\Mlogn^4 N'^2 K\log k)^{1/3}
+ \Mlogn\sum_{n=1}^{N'}\E[\varepsilon'_n] \;.
\]
\end{proof}

\subsection{The \BOG algorithm for non-stationary bandits}\label{sec:BOGpfs}

Now we are ready to bound the regret of the \BOG algorithm (shown in \cref{alg:BanditOG}).

\BOGNonStatOpt*
\begin{proof}
Fix a sequence of $N$ tasks with unknown and potentially variable task lengths $\{\tau_n\}_{n\in[N]}$. Let $\pi_{\text{OS}}$ be the policy used by \cref{alg:BanditOG}.
By the decomposition in \cref{lemma:metalearning2subset} the regret of \cref{alg:BanditOG} when updated every $\tau'$ steps (with $N'=T/\tau'$ updates) satisfies the following, 
\begin{align}
\reg(\pi_{\text{OS}},T, N',\Mlogn)&=\sup_{f_n \in \cF} \max_{S\in \cS}\; \EE{\sum_{n=1}^{N'} \big(f_n(S) - f_n (S_{n})) + B_{\tau',N_n,\Mlogn}}\nonumber
    \\
    &\leq \sup_{f_n \in \cF} \max_{S\in \cS}\; \E\bigg[\sum_{n=1}^{N'} \tfrac{1}{N'}f_n(S) \nonumber
    \\
    &\qquad\qquad +\sum_{n=1}^{N'}(1-\tfrac{1}{N'})f_n(S) -\sum_{n=1}^{N'}f_n (S_{n}) +B_{\tau',N_n,\Mlogn}\bigg]\nonumber
    \\
    & \leq \frac{N'\tau'}{N'}+ \tau' \E[\Rcov]+\sum_{n=1}^{N'}\sqrt{\Mlogn \tau' N_n}\nonumber%
    \\
    &\le \tau' + \tau' \E[\Rcov]+\sqrt{N'  \Mlogn \tau' \sum_{n=1}^{N'} N_n}\label{eq:cs}
    \\
    &\le \tau' + \tau' \E[\Rcov]+\sqrt{N'  \Mlogn \tau' \max\{2T/\tau', 2N\}}\label{eq:sumNn}
    \;,
\end{align}
where in \cref{eq:cs} we use the Cauchy-Schwarz inequality. For \cref{eq:sumNn} we used the fact that if $N'\geq N$ then $\sum_{n=1}^{N'}N_n\leq 2N'=2T/\tau'$, otherwise $\sum_{n=1}^{N'}N_n\leq 2N$ by \cref{eq:segnum}. Let $\varepsilon_n=B_{\tau',N_n,\Mlogn}/\tau'$. 
By \cref{lem:Rcov} we can set $\gamma=\gammaOne$ %
and bound $\E[\Rcov]$ to get
\begin{align*}
    \reg(\pi_{\text{OS}},T,N',\Mlogn) &
    \le \tau' + \tau' \bigg(
    (\Mlogn^4 N'^2 K\log k)^{1/3}
    + \Mlogn \sum_{n=1}^{N'}\varepsilon_n \bigg)+\sqrt{N'  \Mlogn \tau' \max\{2T/\tau', 2N\}}
    \\
    &\leq \tau' + \tau' 
    (\Mlogn^4 N'^2 K\log k)^{1/3}
    +\Mlogn \sqrt{N'  \Mlogn \tau' \max\{2T/\tau', 2N\}} \\
    &= T {N'}^{-1} +
     T {N'}^{-1/3} (\Mlogn^4 K \log K)^{1/3}
    + \Mlogn \sqrt{2 \Mlogn N' T} \;.
\end{align*}
Here, the second inequality follows from Cauchy-Schwarz inequality and the same argument as above for bounding $\sum_{n=1}^{N'}N_n$, and the last step holds as the bound is minimized with $N'\ge N$.

\end{proof}

If $N \ge N_1 \doteq \bigg(\frac{T^{3}(K\log K)^{2}}{\Mlogn}\bigg)^{1/5}$ and $M \le (K\log K)^{1/3}
$ (large number of changes and small number of optimal arms), then our regret upper bound is $\tilde O(\Mlogn^{3/2}\sqrt{N T})$, and the regret of \BOG improves upon the $\tilde{O}(\sqrt{KTN})$ bound of standard non-stationary bandit algorithms (such as \AdS).

If $N \le N_1$ and $M \le (K\log K)^{1/3}
$, and we can obtain an improved bound by using a larger number of segments. By choosing $N'= \bigg(\frac{T^{3}(K\log K)^{2}}{\Mlogn}\bigg)^{1/5}
$ and $M \le (K\log K)^{1/3}$ segments, each of size 
$\tau'=T/N'$, the bound improves to $\tilde O(\Mlogn^{7/5} (K \log K)^{1/5} T^{4/5})$.

If $N \le N_2 \doteq \Mlogn^{14/5}(T/K)^{3/5}(\log K)^{2/5}
$ (even small number of changes), then $\sqrt{KNT} \le \Mlogn^{7/5} (K \log K)^{1/5} T^{4/5}
$. In this case, the simple baseline of $\tilde{O}(\sqrt{KTN})$ is smaller than our bound, and the learner should simply play a standard non-stationary bandit algorithm. Notice that $N_2 \le N_1$ as long as $M \le K^{1/3}$.

\section{Proofs of Section~\ref{sec:greedy}}

\subsection{Complement to the proof of Theorem~\ref{thm:unique-identification}}
\label{ap:induction-proof-th-unique-udentification}

We are left to prove that 
\begin{align}
    \label{eq:recursive2}
G_n(s) \le G_{n+1}(s) + \frac{(\Cinfo - \Cgood)(\Cbad - \Cgood)}{\Cbad - \Cgood + G_{n+1}(s)},
\end{align}
given in \eqref{eq:recursive} implies that for any $n\leq N$,
\begin{equation}
    \label{eq:Gbound}
G_{n}(s) \le \sqrt{2(\Cinfo - \Cgood)(\Cbad - \Cgood) (N-n)}\;.
\end{equation}

\begin{proof}
We proceed by (backward) induction. First, by definition, $G_{N}(s)=V_N(s)-V_N(s+1)=0$ for all $s$, thus \eqref{eq:Gbound} holds for $n=N$. Next, assume that \eqref{eq:Gbound} holds for $\{N,N-1,\dots,n+1\}$, and we show that it also holds for $n$.

Consider positive constants $b \ge a$ and consider the function $h(z) = z + \frac{ab}{b + z}$ defined on $[0,c]$ for some $c>0$. Then $h'(z) = 1-ab/(b+z)^2 \ge 0$. Therefore, $h$ is maximized at $z=c$. Since the right-hand side of \eqref{eq:recursive2} is of the form $h(G_{n+1}(s))$ with $a=\Cinfo - \Cgood$ and $b=\Cbad - \Cgood$, which indeed satisfy $b\ge a$. By this argument, the induction assumption, and $0\le G_{n+1}(s) \le \sqrt{ab(N-n-1)}$ by the induction hypothesis, we obtain that
\begin{align}
G_n(s) &\le \sqrt{2(\Cinfo - \Cgood)(\Cbad - \Cgood) (N-n-1)} \nonumber \\ 
&\qquad+ \frac{(\Cinfo - \Cgood)(\Cbad - \Cgood)}{\Cbad - \Cgood + \sqrt{2(\Cinfo - \Cgood)(\Cbad - \Cgood) (N-n-1)}} \nonumber \\
& = \sqrt{2ab(N-n-1)} + \frac{ab}{b+\sqrt{2ab(N-n-1)}}
\label{eq:recursive3}
\end{align}
It remains to show that the right-hand side above is bounded from above by $\sqrt{2ab(N-n)}$. This follows since
\begin{align*}
   \lefteqn{\sqrt{2ab(N-n)} - \sqrt{2ab(N-n-1)}
   = \frac{\sqrt{2ab}}{\sqrt{N-n} + \sqrt{N-n-1}}} \\
   & = \frac{ab}{\sqrt{ab(N-n)/2} + \sqrt{ab(N-n-1)/2}}
   \ge \frac{ab}{b+\sqrt{2ab(N-n-1})}
\end{align*}
where the last inequality holds because
\begin{align*}
b+\sqrt{2ab(N-n-1}) & \ge 
\big[\sqrt{ab} + \sqrt{ab(N-n-1)/2}\big] + \sqrt{ab(N-n-1)/2} \\
& \ge \sqrt{ab(N-n)/2} + \sqrt{ab(N-n-1)/2}
\end{align*}
(where we used that $1+\sqrt{z} \ge \sqrt{z+1}$ for $z \ge 0$).
Thus, $G_n(s) \le \sqrt{2ab(N-n)}$, proving the induction hypothesis \eqref{eq:Gbound} for $n$.

\end{proof}

\subsection{Proof of \cref{thm:unique-identification}}
\label{sec:thm:unique-identification-pf}

The proof relies on solving the min-max problem in \eqref{eq:cost-to-go-def}.
First, we consider the case that the best-arm-identification can be performed with probability $1$ (i.e., $\delta=0$ in the efficient identification assumption). From symmetry, it is easy to see that $V_n(\cI_n)$ only depends on the size of $\cI_n$, and not the actual arms in $\cI_n$. Therefore, to simplify notation and emphasize the dependence on the number of discovered optimal arms, we use below $V_n(|\cI_n|):=V_n(\cI_n)$. Let  $n_M=\argmin_n\{|\cI_n|=M\}$ be the first round when all optimal arms have been discovered. Then from any $n > n_M$, the adversary no longer can reveal new arms ($q=0$), and the learner should no longer explore ($p=0$), and so
\[
\forall n\geq n_M, \; V_n(M)=(N-n)\Cgood.
\]

Denoting $s=|\cI_n|$, the min-max optimization objective in \eqref{eq:cost-to-go-def} can be written as 
\begin{align*}
L(q,p) &= \Cgood + p (\Cinfo - \Cgood) + V_{n+1}(s) +\\& q (1-p) (\Cbad - \Cgood)
- p \left[q^1(V_{n+1}(s)\right.\\&\left. - V_{n+1}(s+1)) + \dots\right.\\&\qquad\qquad\left. + q^{M-s} (V_{n+1}(s) - V_{n+1}(M)) \right]\;,
\end{align*}
where $q^i$ denotes the probability that the environment reveals $i$ optimal arms in the round, and $q=\sum_{i=1}^{M-s} q^i$.
Given that $V_{n+1}(s) - V_{n+1}(s+1) < \cdots < V_{n+1}(s) - V_{n+1}(M)$,\todoa{explain why this is true} the maximizing $q$ is such that $q^i=0$ for $i>1$ and $q=q^1$. Using this, the saddle point can be obtained by solving $\partial L(q,p)/\partial q=0$ and $\partial L(q,p)/\partial p=0$:
\begin{align}
p= p_n = \frac{\Cbad - \Cgood}{\Cbad - \Cgood + V_{n+1}(s) - V_{n+1}(s+1)} \nonumber
\\
q_n = \frac{\Cinfo - \Cgood}{\Cbad - \Cgood + V_{n+1}(s) - V_{n+1}(s+1)} \label{eq:qn-}\;.
\end{align}
Plugging these values in \eqref{eq:cost-to-go-def}, we get
\begin{align*}
V_n(s) &= V_{n+1}(s) + \Cgood 
+ \frac{(\Cinfo - \Cgood)(\Cbad - \Cgood)}{\Cbad - \Cgood + V_{n+1}(s) - V_{n+1}(s+1)} \;.
\end{align*}
Given $N$ and $M$, the policy of the learner and the adversary can be computed by solving the above recursive equation. Given that for any $s<M$,  $V_n(s+1) \ge V_{n+1}(s+1) + \Cgood$, 
\begin{align*}
    V_n(s) - V_n(s+1) \le V_{n+1}(s) - V_{n+1}(s+1) + \frac{(\Cinfo - \Cgood)(\Cbad - \Cgood)}{\Cbad - \Cgood + V_{n+1}(s) - V_{n+1}(s+1)} \;.
\end{align*}
Let $G_n(s) = V_n(s) - V_n(s+1)\geq 0$ be the 
cost difference in state $s$ relative to state $s+1$.
We have
\begin{equation}
\label{eq:recursive}
G_n(s) \le G_{n+1}(s) + \frac{(\Cinfo - \Cgood)(\Cbad - \Cgood)}{\Cbad - \Cgood + G_{n+1}(s)} \;,
\end{equation}
and indeed by a telescopic argument,
\begin{align*}
    \reg_T - N B_{\tau,M} &= V_0(0) - V_0(M) =\sum_{s=0}^{M-1} (V_0(s) - V_0(s+1)) = \sum_{s=0}^{M-1} G_0(s) \;.
\end{align*}
The proof is completed by bounding $G_0(s)$ by backward induction on $n\leq N$:
\[
G_{N-n}(s) \le \sqrt{2(\Cinfo - \Cgood)(\Cbad - \Cgood) n}\;.
\]
The proof of this inequality relies on standard algebraic manipulations that can be found in Appendix~\ref{ap:induction-proof-th-unique-udentification}.
When the BAI routine returns the best arm with probability at least $1-\delta/N$, with a simple union bound argument, the probability that $\cI_n$ ever contains wrong elements is bounded by $\delta$ and the above derivations again hold.

\subsection{Proof of Theorem~\ref{thm:greedy}}
\label{ap:proof-thm-greedy}

\begin{proof}
The proof relies on the analysis of the optimization problem defined as in Eq.~\eqref{eq:cost-to-go-def} with a fixed $p_n=\hat{p}_{|S_n|}=\sqrt{|S_n|K\tau/(NB_{\tau,K})}$ (no minimization over the exploration probability of the learner in task $n$). As in the proof of Theorem~\ref{thm:unique-identification}, we assume that the best arm identification is successful, and the extension to $\delta\neq0$ can be done the same way.
After some algebraic manipulation, similarly to the proof of Theorem~\ref{thm:unique-identification}, the optimization objective can be written as 
\begin{align*}
L(q) &= V_{n+1}(S_n) + B_{\tau,|S_n|} + \hat{p}_{|S_n|} (B_{\tau,K} - B_{\tau,|S_n|})\\
& \qquad + q \left\{ (1-\hat{p}_{|S_n|}) (\tau - B_{\tau,|S_n|}) - \hat{p}_{|S_n|} (V_{n+1}(S_n) - V_{n+1}(S'_{n})) \right\}
\end{align*}
where $S'_n$ is the new greedy subset selected by the learner in time step $n+1$ if $S^*_n \cap S_n=\emptyset$ and the learner chooses to explore at time $n$ (we use the notation $S'_n$ instead of $S_{n+1}$ to emphasize that this corresponds to the aforementioned choices of the learner and the adversary).
Given that $L$ is linear in $q$, the optimal adversary choice is either $q=0$ or $q=1$ (similarly as in Theorem~\ref{thm:unique-identification}, it is suboptimal for the adversary to reveal multiple optimal arms). We have
\[ 
q =
  \begin{cases}
    0       & \quad \text{if } (1-\hat{p}_{|S_n|}) (\tau - B_{\tau,|S_n|}) - \hat{p}_{|S_n|} (V_{n+1}(S_n) - V_{n+1}(S'_n)) \le 0\,, \\
    1  & \quad \text{otherwise } 
  \end{cases}
\]
When $q=0$, $V_n(S_n) = V_{n+1}(S_n) + B_{\tau,|S_n|} + \hat{p}_{|S_n|} (B_{\tau,K} - B_{\tau,|S_n|})$, and given that $|S_n|\le M'$, the total contribution of these rounds to the regret is bounded by 
\[
N B_{\tau,M'} + \hat{p}_{M'} N B_{\tau,K} \;.
\]
Next, consider the rounds where $q=1$. Among these rounds, consider rounds where the adversary chooses a particular arm $a\in S^*$ and the learner chooses to explore (\textsc{Exr}). This arm is not added to the future \textsc{Ext} subset of the learner if instead another arm is used to cover this round. This means that after at most $K$ such rounds, the learner adds $a$ to the \textsc{Ext} subset. Since the learner's regret in the exploration rounds is $B_{\tau,K}$, in these rounds the cumulative regret is bounded  by $MKB_{\tau,K}$. Since the random choices made by the learner and the adversary are independent in the same round, we discover the first arm in $K/\hat{p}_1$ tasks in expectation, the second in $K/\hat{p}_2$ tasks in expectation, and so on. Thus, since the size of our cover is at most $M'$, we get 
\begin{align*}
K\sum_{s=1}^{M'} \frac{1}{\hat{p}_s} &= \sqrt{\frac{KNB_{\tau,K}}{\tau}} \sum_{s=1}^{M'} \frac{1}{\sqrt{s}} 
\le 2\sqrt{\frac{M' KNB_{\tau,K}}{\tau}} \;.
\end{align*}
The adversary reveals all positions after $2\sqrt{\frac{M' KNB_{\tau,K}}{\tau}}$ such tasks in expectation where the adversary's choice is $q=1$. If in these tasks the learner chooses to exploit, it can suffer a regret $\tau$, leading to a total expected regret of at most $2\sqrt{M' KNB_{\tau,K}\tau}$. Thus, the total regret of rounds with $q=1$ is bounded by
\[
2\sqrt{M' KNB_{\tau,K}\tau} + M K B_{\tau,K} \;.   
\]
Therefore,
\[
\reg_T = V_0(\emptyset) \le N B_{\tau,M'} + M K B_{\tau,K} + 3\sqrt{M K \tau B_{\tau,K} N} \;.
\]

\end{proof}

\section{Partial monitoring and bandit meta-learning}
\label{app:partial-monitoring}

Partial monitoring is a general framework in online learning that disentangles rewards and observations (information). It is a game where the learner has $Z$ actions and the adversary has $D$ actions, and it is characterized by two $Z\times D$ matrices (not observed): matrix $C$ maps the learner's action to its cost given the adversary's choice, and matrix $X$ maps the learner's action to its observation given the adversary's choice. 
In all generality, we consider bandit meta-learning problems with $Z+1$ learner actions: an $\exr$ action that provides information for a cost $\Cinfo$, and $Z$ other actions that do not provide information but have a hidden cost $\Cgood$ or $\Cbad$ depending on whether the chosen action had low or high cost respectively. 

As defined in the introduction, a bandit subset selection problem is realizable when there is a subset of size $M$ that contains an optimal arm in all rounds. Otherwise, the problem is called agnostic. 

In our bandit subset selection problem, $Z= {K \choose M }\leq K^M$ and the adversary can have up to $2^K$ choices depending on the realizable or agnostic nature of the problem. We have $D=M$ if the problem is realizable and if the adversary is constrained to picking a unique optimal arm in each round.
For example, let $M=2$ and $K=4$. There are $Z+1={4 \choose 2} +1 = 7$ learner actions and only $D=2$ possible choices for the adversary
\[
\begin{pmatrix}
\exr \\
\{1,2\}=x^* \\
\{1,3\} \\
\{1,4\} \\
\{2,3\} \\
\{2,4\} \\
\{3,4\}
\end{pmatrix} 
\to C = 
\begin{pmatrix}
\Cinfo & \Cinfo \\
\Cgood & \Cgood \\
\Cgood & \Cbad \\ 
\Cgood & \Cbad \\ 
\Cbad & \Cgood \\
\Cbad & \Cgood \\
\Cbad & \Cbad
\end{pmatrix}
\,,\qquad
X = 
\begin{pmatrix}
1 & 2 \\
\perp & \perp  \\
\vdots & \vdots  \\ 
\perp & \perp 
\end{pmatrix}
\;.
\]
The symbol $\perp$ is used to denote no observations. We use $C_{i,y}$ to denote the cost of action $i\in \{\exr,x_1,\ldots,x_Z\}$ when adversary chooses $a\in [D]$. Thanks to this reduction, we can leverage the partial monitoring literature to obtain an algorithm and the corresponding bounds for our problem as well. We detail this process below. Note that using the vocabulary of online learning, the learner's actions are referred to as "experts".

Next, we describe an algorithm based on the Exponentially Weighted Average (EWA) forecaster. The learner estimates the cost matrix by importance sampling when action $\exr$ is chosen. When $\ext$ is chosen, the learner samples an expert according to EWA weights that depend on the estimated cost matrix. The pseudo-code of the method is shown in Algorithm~\ref{alg:EWA}.   

\begin{algorithm}[tb]
\caption{The partial monitoring algorithm}
 \label{alg:EWA}
 \begin{algorithmic}[1]
 
\STATE Exploration probability $p\in (0,1)$, learning rate $\eta>0$, base costs $\Cinfo, \Cgood,\Cbad$ \;
     \FOR{$n=1,2,\dots,N$}{
              
            \STATE With probability $p$, let $E_n=\textsc{Exr}$ and otherwise $E_n=\textsc{Ext}$\;
      \IF{$E_n=\textsc{Exr}$}{
            \STATE Observe the best arms $S^*_{n}$ of this round and for all $i\in\textsc{Ext}$ experts, observe cost $C_{i,S^*_n}$ and let $\widehat C_n(i) = (C_{i,S^*_n}-\Cgood)/p$\;
               
            \STATE Update exponential weights $Q_{n,i} \propto \exp(-\eta \sum_{\tau=1}^{n} \widehat C_n(i))$\;
               Suffer cost $\Cinfo$\;
       }
       \ELSE{
            \STATE Sample $S_n\sim Q_{n-1}$\;
               
            \STATE Suffer (but do not observe) cost $\Cgood$ if $S^*_n \cap S_n \neq \emptyset$ and suffer cost $\Cbad$ otherwise\;
      }\ENDIF
 }\ENDFOR
 \end{algorithmic}
\end{algorithm}

To analyze the algorithm, we consider the realizable and agnostic cases. In the realizable case, there a subset of size $M$ that contains an optimal arm in all rounds. In this case, the exponential weights distribution reduces to a uniform distribution over the subsets that satisfy this condition. %
\begin{theoremi}
\label{thm:PM-app}
Consider the partial monitoring algorithm shown in Algorithm~\ref{alg:EWA}. In the agnostic case, with the choice of $p=O\left(\left(\frac{\Cbad^2 \log Z}{\Cinfo^2 N}\right)^{1/3} \right)$ and $\eta=O\left(\left(\frac{\log^{2} Z}{\Cinfo  \Cbad^{2}N^{2}}\right)^{1/3}\right)$, the regret of the algorithm is bounded as $O((\Cinfo \Cbad^{2} N^{2} \log Z)^{1/3})$. In the realizable case, with the choice of $p=\sqrt{\frac{\Cbad\log Z}{\Cinfo N}}$ and $\eta=1$, the regret of the algorithm is bounded as $O(\sqrt{\Cinfo \Cbad N \log Z})$.
\end{theoremi}
\begin{proof}
Let function $f_n:[Z+1]\times [D]\rightarrow \R^{Z+1}$ be defined by
\[
f_n(k,X_{k,y})_i = \one\{k=\textsc{Exr}\} (C_{i,y}-\Cgood) \;.
\]
Therefore, $\sum_{k=1}^{Z+1}f_n(k,X_{k,y})_i = C_{i,y}-\Cgood$. With probability $p$, let $E_n=\textsc{Exr}$ and otherwise $E_n=\textsc{Ext}$. Let $C_n(i)=C_{i,Y_n}$. Define cost estimator
\[
\widehat C_{n,i} = \frac{f_n(E_n,X_{E_n,Y_n})_i}{p} = \frac{\one\{E_n=\textsc{Exr}\} (C_n(i)-\Cgood)}{p}\;.
\]
Let $Q_n$ be the weights of the EWA forecaster defined using the above costs. For any $i$, we have $\E(\widehat C_{n}(i))=C_n(i)-\Cgood$. Let $E_n$ be the learner's decision in round $n$, that is either \textsc{Exr} or a subset chosen by EWA, in which case it is denoted by $x_n$. We have
\[
\E(C_n(E_n)) = p \Cinfo + (1-p) \E(C_n(S_n)) \;.
\]
Let $S^*$ be the optimal subset. By the regret bound of EWA \citep{PLG},
\[
\sum_{n=1}^N \widehat C_n(S_n) - \sum_{n=1}^N \widehat C_n(S^*) \le \frac{\log Z}{\eta} + \frac{\eta}{2}\sum_{n=1}^N \|\widehat C_n\|_\infty^2 \;.
\]
Thus,
\begin{align*}
\sum_{n=1}^N \E(C_n(E_n)) - \sum_{n=1}^N \E(C_n(S^*)) &\le \Cinfo \sum_{n=1}^N p  + \frac{\log Z}{\eta} + \frac{\eta}{2}\sum_{n=1}^N \E(\|\widehat C_n\|_\infty^2) \\
&\le \Cinfo \sum_{n=1}^N p + \frac{\log Z}{\eta} + \frac{\eta \Cbad^2}{2}\sum_{n=1}^N \frac{1}{p} \;.
\end{align*}
With the choice of $p=O((\Cbad/\Cinfo)^{2/3}(\log^{1/3} Z)/N^{1/3})$ and $\eta=O((\log^{2/3} Z)/(\Cbad^{2/3} \Cinfo^{1/3}N^{2/3}))$, the regret of the partial monitoring game is bounded as $O(\Cbad^{2/3} \Cinfo^{1/3}N^{2/3} \log^{1/3} Z)$. The regret scales logarithmically with the number of experts, and is independent of the number of adversary choices.

Next, we show a fast $O(\sqrt{N})$ rate when the optimal expert always has small cost. More specifically, we assume that $C_n(S^*)=\Cgood$ for the optimal expert $S^*$. The fast rate holds independently of the relative values of $\Cgood,\Cinfo$, and $\Cbad$. The algorithm can also be implemented efficiently. 

Let $\widehat \ell_n = p \widehat C_n/\Cbad$, which is guaranteed to be in $[0,1]$. Notice that $\sum_{n=1}^N \widehat \ell_n(S^*)=0$ as $C_n(S^*)=\Cgood$ by assumption. In this case, the regret of EWA is known to be logarithmic:
\[
\sum_{n=1}^N \widehat \ell_n(S_n) - \sum_{n=1}^N \widehat \ell_n(S^*) = O(\log Z) \;.
\]
Thus,
\[
\sum_{n=1}^N \E(C_n(E_n)) - \sum_{n=1}^N \E(C_n(S^*)) \le \Cinfo \sum_{n=1}^N p + \frac{\Cbad \log Z}{p} \;.
\]
Therefore, with the choice of $p=\sqrt{\frac{\Cbad\log Z}{\Cinfo N}}$,
\[
\sum_{n=1}^N \E(C_n(E_n)) - \sum_{n=1}^N \E(C_n(S^*)) \le O(\sqrt{\Cinfo \Cbad N \log Z}) \;.
\]
The meta-regret scales logarithmically with the number of experts, and is independent of the number of adversary choices. Given that the optimal expert is known to have small loss in all rounds, the learner can eliminate all other experts. Therefore, the EWA strategy reduces to a uniform distribution over the surviving experts. %

\end{proof}

\subsection{Proof of Theorem~\ref{thm:PM}}
\label{app:proof-thm-PM}

\pmml~is constructed as a special case of the EWA algorithm above, where the sampling distribution at each $\ext$ round is simply the uniform distribution over the surviving experts. The proof of Theorem~\ref{thm:PM} is therefore a direct consequence of the more general analysis done for the EWA forecaster in Theorem~\ref{thm:PM-app} above.

\begin{proof}
The BAI algorithm might return a number of extra arms in addition to the optimal arm. However, since with high probability the optimal arm is always in the surviving set, the cost estimate for the optimal subset is always zero, and costs of all other subsets are under-estimated. Therefore, if $S_n$ is the expert (subset) selected in task $n$ and $S^*$ is the optimal subset, by fast rates of the previous section,
\[
\sum_{n=1}^N \E(C_n(S_n)) - \sum_{n=1}^N \E(C_n(S^*)) \le O(\sqrt{\Cinfo \Cbad N \log Z}) \;.
\]
Given that with high probability the optimal arm is always in the surviving set and therefore $C_n(S^*)=\Cgood$,
\begin{align*}
\reg_T &=  \sum_{n=1}^N \E\left(  \tau r_n(a_n^*) - \sum_{t=1}^\tau r_n(A_{n,t}) \right)\le \sum_{n=1}^N \E(C_n(S_n)) \le N \Cgood + O(\sqrt{\Cinfo \Cbad N \log Z}) \\
&= N \sqrt{M \tau} + O(\sqrt{\Cinfo \Cbad N \log Z})\\
& = N\sqrt{M \tau} + O(\tau^{3/4}K^{1/4}\sqrt{NM\log(K)}\,,
\end{align*}
where the first inequality holds by the fact that $\E(C_n(S_n))$ is an upper bound on the regret for task $n$. 
\end{proof}

\section{Further experimental details and results}\label{sec:FurtherExperiments}

This section consists of further experimental details and results. 
We use the code in the following repository: \url{https://anonymous.4open.science/r/meta-bandit-760E/README.md}. We used a server machine with the following configuration:
OS: Ubuntu 18.04 bionic, Kernel: x86\_64 Linux 4.15.0-176-generic, CPU: Intel Core i9-10900K @ 20x 5.3GHz, GPU: GeForce RTX 2080 Ti, RAM: 128825 MiB, DISK: 500 GB SSD.

\subsection{Setup}
In each experiment, the adversary first samples the size $M$ set of optimal arms, $S^*:=\cup_n S^*_n$, uniformly at random (without replacement) from $[K]$. The mean reward of task $n$, $r_n \in \mathcal{R}=[0,1]^K$, is then generated according to the experiment setup as described in the following.

\paragraph{The optimal arm:} We categorize the experiments into three settings based on how the optimal arm is generated: i) \emph{non-oblivious adversarial}, ii) \emph{oblivious adversarial}, and iii) \emph{stochastic}.

i) In the adversarial setting with a \emph{non-oblivious} adversary, the adversary peeks into the learner's set of discovered arms, $S_n$, at the end of each task. With probability $q_n$ (see \cref{eq:qn-}), the adversary chooses a new optimal arm uniformly at random from $[K]\backslash S_n$. Otherwise, the next optimal arm is chosen uniformly at random from $S_n$. 

ii) The \emph{oblivious} adversary is applicable against any learner even if the learner does not maintain a set of
discovered arms. Here the adversary simulates an imaginary \greedy algorithm with a minimax optimal $p_n$ (see \cref{eq:qn-}). Then it samples new optimal arms and generates the reward sequence with respect to this imaginary learner. This is the same as the non-oblivious adversary except here the adversary plays against an imaginary learner. 

iii) In the \emph{stochastic} setting, for each task $n$, the environment samples the optimal arm uniformly at random from the optimal set, i.e., $a^*_n\sim \text{Uniform}(S^*)$.

Note that in the \emph{non-oblivious} setting, the rewards are generated at the start of each task, according to the learner's discovered arms. In the other settings, however, rewards of all the tasks could be generated at the very beginning, independently of the learner.
\paragraph{The sub-optimal arms (min gap):} Based on the discussion after \cref{assumption:identification}, the minimum gap for the assumption to hold is 
\[
r_n(a^*) - \max_{a \neq a^*}r_n(a) \geq \Delta\;,
\]
where $\Delta = \Theta(\sqrt{K\log(N/\delta)/\tau})$. After generating the optimal arm, depending on the setting, the rewards of other arms are generated in two ways: 1) with a minimum gap condition uniformly at random in $[0,r_n(a^*_n)-\Delta)$ and 2) without a minimum gap condition uniformly at random in $[0,r_n(a^*_n)]$. In the second case, the mean reward is generated such that the gap condition is violated by at least 1 sub-optimal arm.

\paragraph{Task length and PE:} As we know, task length plays an important role in regard to PE performance. In the case where \cref{assumption:identification} holds, we set the phase length based on $\Delta$ 
and make sure $\tau$ is longer than the length of the first phase of PE. For more details, see the analysis of PE \citep{AO-2010} in exercise 6.8 (elimination algorithm) of \citet{lattimore-Bandit}.

\paragraph{\cref{assumption:identification}:}
We have two types of experiments considering \cref{assumption:identification}: i) In the experiments where \cref{assumption:identification} is supposed to hold, we make sure the task length is longer than the first phase of PE and the minimum gap condition holds (case 1 in the discussion on the min gap). 
ii) In the tasks where \cref{assumption:identification} is supposed to be violated, we use case 2 in the discussion on the min gap above with a small $\tau$ so that PE fails.

\subsection{Further experiments} 
Next, we report the experimental results under different conditions. Error bars are $\pm1$ standard deviation, computed over $5$ independent runs.

\cref{fig:Gap} shows the result when \cref{assumption:identification} holds, where \greedy almost matches the \OptMoss, outperforming the other algorithms. 
\cref{fig:NoIden} shows the results when \cref{assumption:identification} does not hold. 
In this case, we observe that \BOG outperforms the other algorithms and is close to \OptMoss. Here \greedy is less effective and sometimes has large variance due to the failure of PE.

\begin{figure*}[htb!]
    \centering
    \begin{minipage}{.24\textwidth}
        \includegraphics[width=\textwidth]{img/Adversarial_task_exp.png}
    \end{minipage}
    \begin{minipage}{.24\textwidth}
        \includegraphics[width=\textwidth]{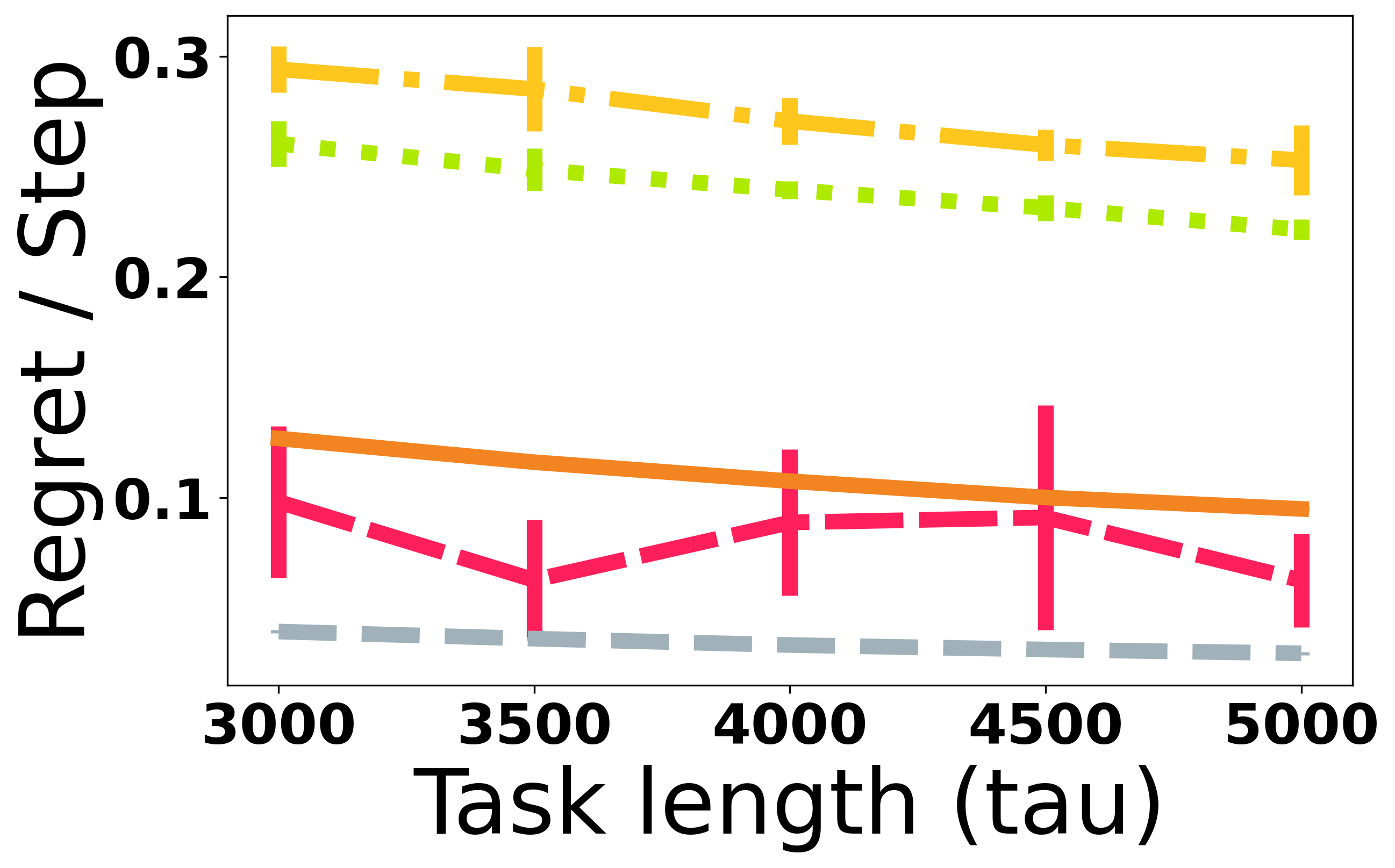}
    \end{minipage}
    \begin{minipage}{.24\textwidth}
        \includegraphics[width=\textwidth]{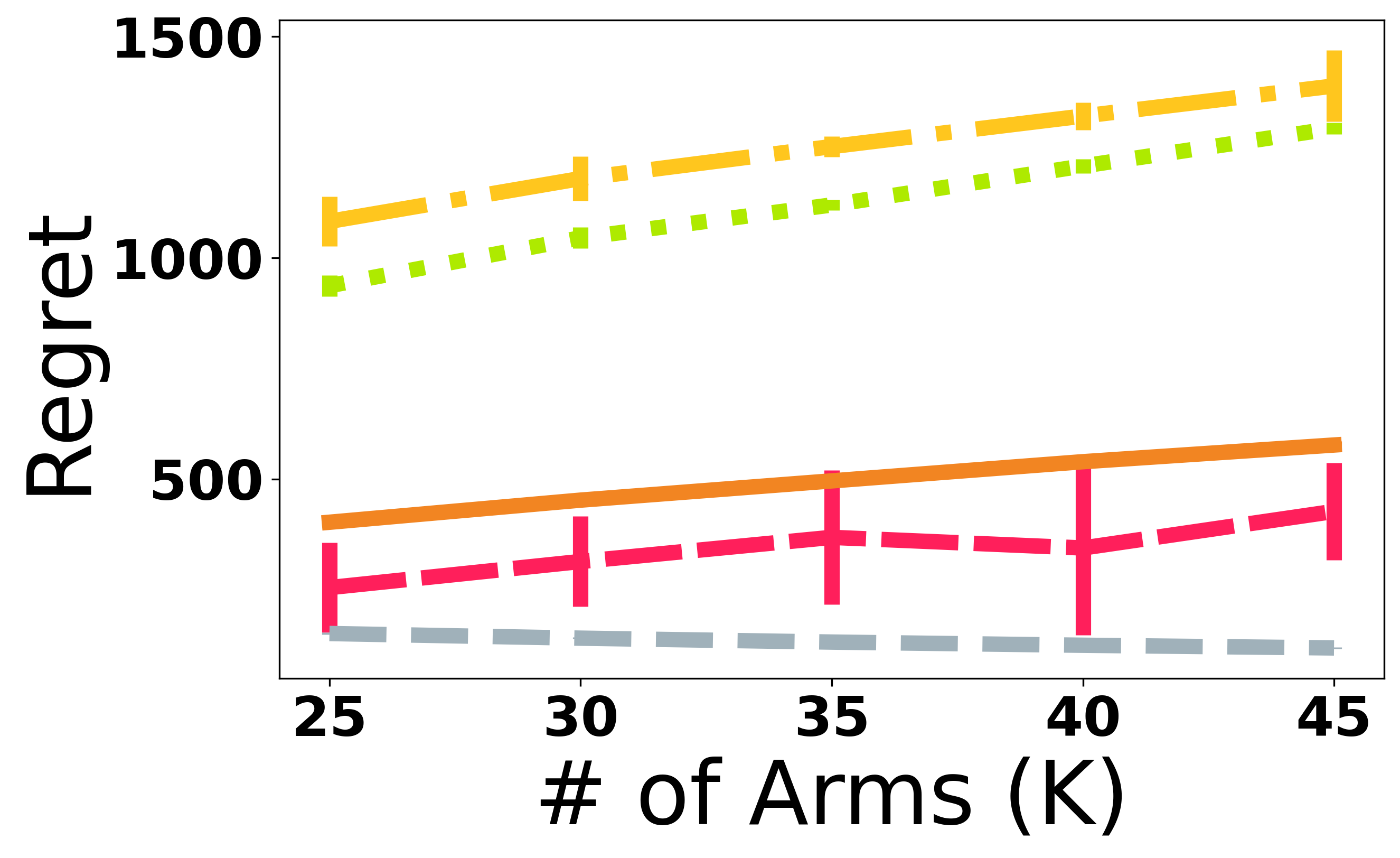}
    \end{minipage}
    \begin{minipage}{.24\textwidth}
    \includegraphics[width=\textwidth]{img/Adversarial_subset_exp.png}
    \end{minipage}
    \includegraphics[width=\textwidth]{img/legend5.png}
    \vspace{-0.8cm}
    \caption{Oblivious adversarial setting with \cref{assumption:identification}.
    Default setting: $\setupOne$. 
    \greedy is near-optimal on all tasks.
    Left to Right: Regret as a function of $N$, $\tau$, $K$, and $M$.
    }
    \label{fig:Gap}
\end{figure*}

\begin{figure}[htb!]
    \centering
    \begin{minipage}{.24\textwidth}
        \includegraphics[width=\textwidth]{img/no_gap_small_tau_Adversarial_task_exp.png}
    \end{minipage}
    \begin{minipage}{.24\textwidth}
        \includegraphics[width=\textwidth]{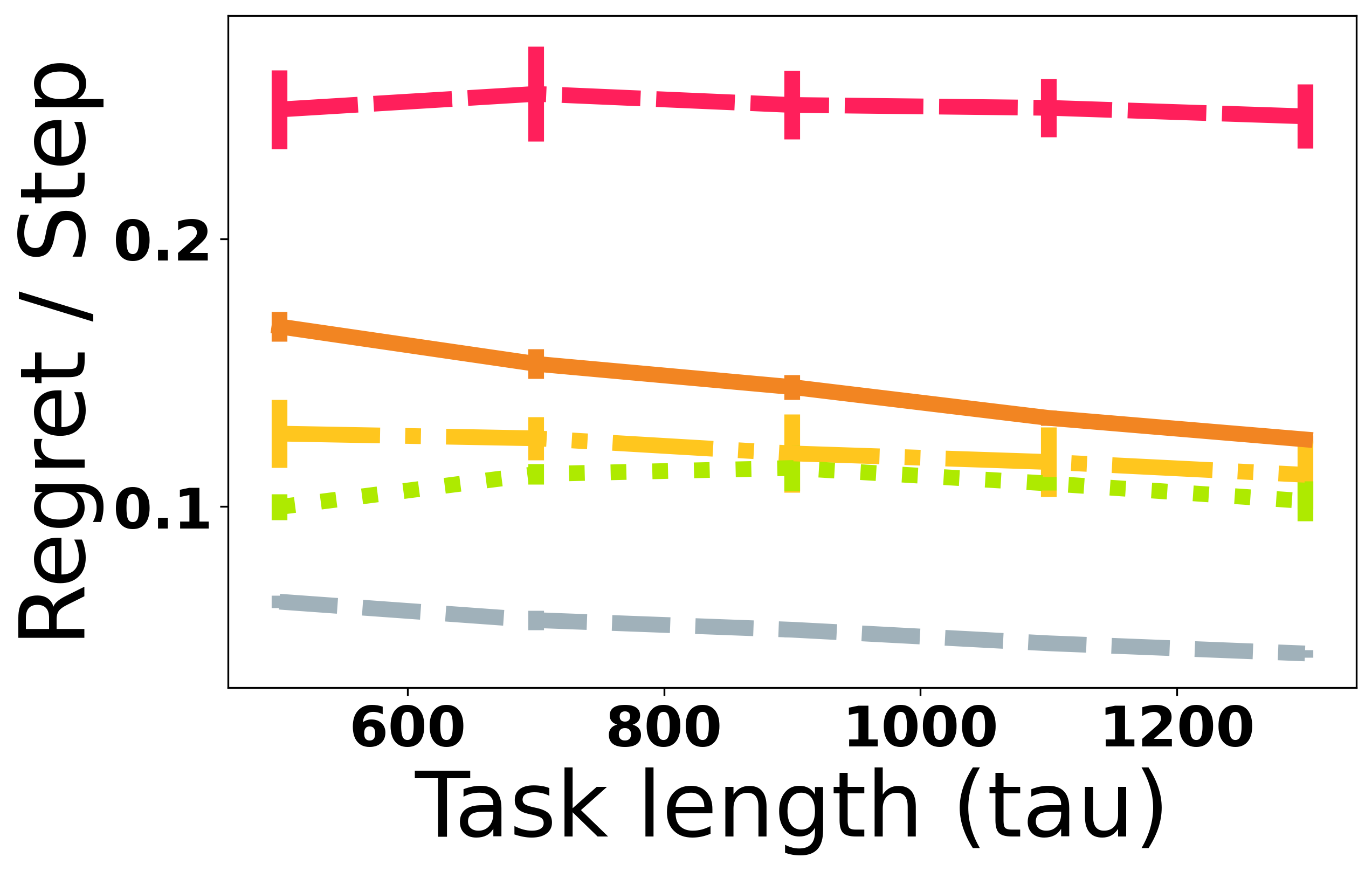}
    \end{minipage}
    \begin{minipage}{.24\textwidth}
        \includegraphics[width=\textwidth]{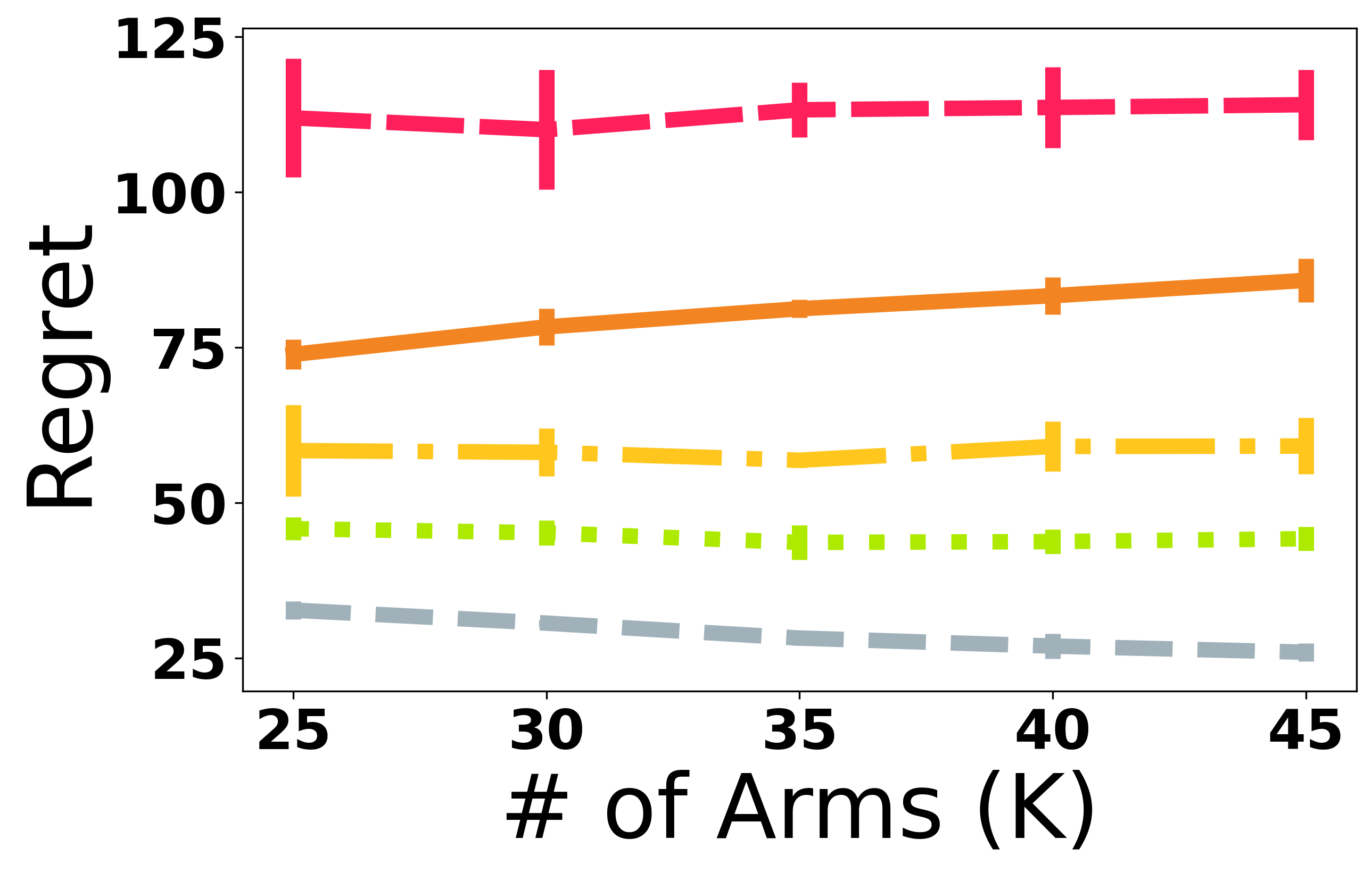}
    \end{minipage}
    \begin{minipage}{.24\textwidth}
    \includegraphics[width=\textwidth]{img/no_gap_small_tau_Adversarial_subset_exp.png}
    \end{minipage}
    \includegraphics[width=\textwidth]{img/legend5.png}
    \vspace{-0.8cm}
    \caption{Oblivious adversarial setting \NoGapSmallt. 
    Default setting: $\setupTwo$. 
    $\BOG$ is near-optimal on all tasks and outperforms \OGO. 
    Left to Right: Regret as a function of $N$, $\tau$, $K, M$.
    }
    \label{fig:NoIden}
\end{figure}

\cref{fig:E-BASS} demonstrates the performance of \pmml when \cref{assumption:identification} holds. We can see that \pmml outperforms all other baselines. For large $M$, \greedy seems to be more effective than the others. \cref{fig:E-BASS_no_gap} compares \pmml to the other algorithms when \cref{assumption:identification} does not hold. \BOG is competitive with \pmml and outperforms it for larger $M$. Comparing \cref{fig:E-BASS} and \cref{fig:E-BASS_no_gap}, we can see that \greedy and \pmml perform better if \cref{assumption:identification} holds.

\begin{figure}[htb!]
    \centering
    \begin{minipage}{.24\textwidth}
        \includegraphics[width=\textwidth]{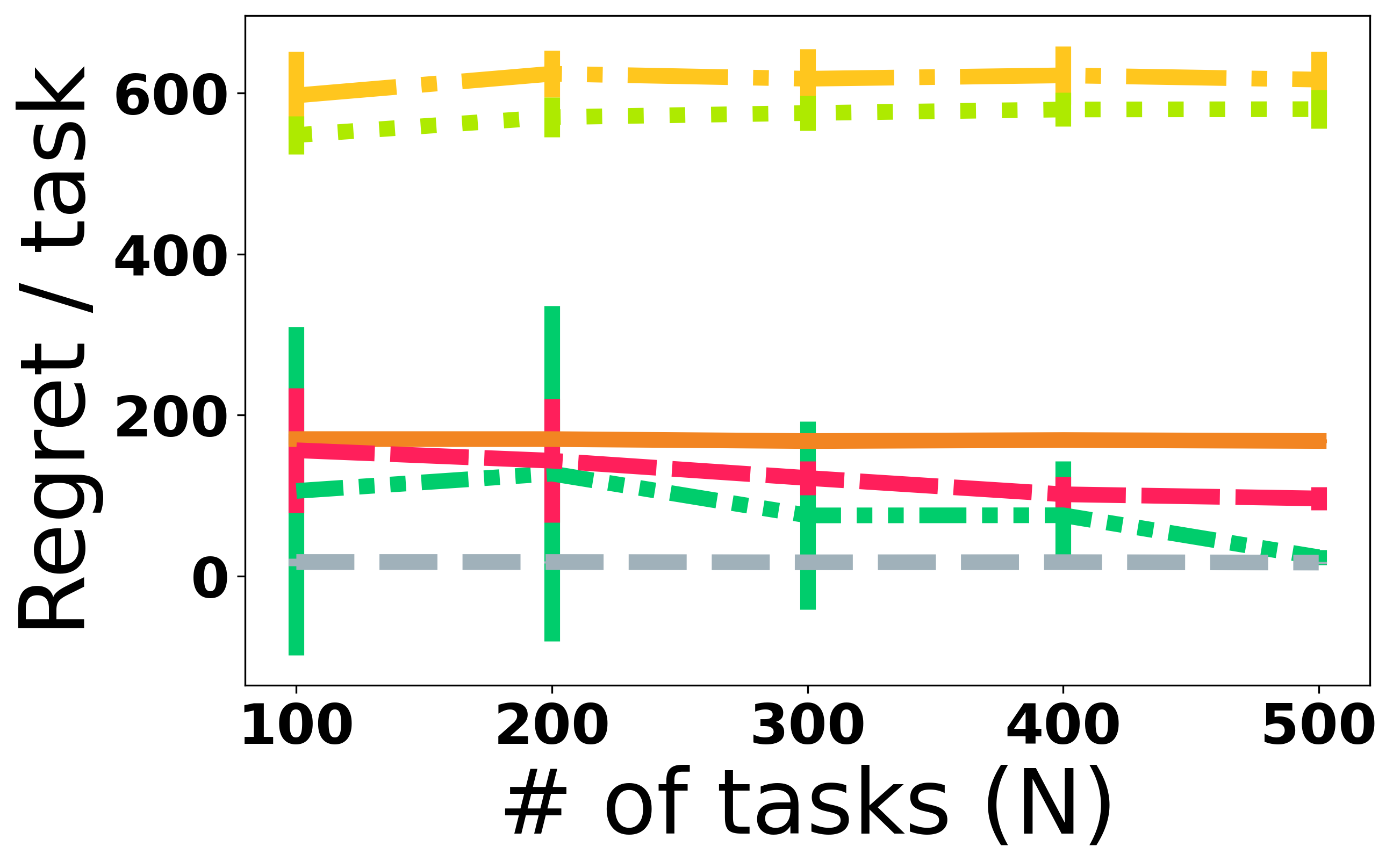}
    \end{minipage}
    \begin{minipage}{.24\textwidth}
        \includegraphics[width=\textwidth]{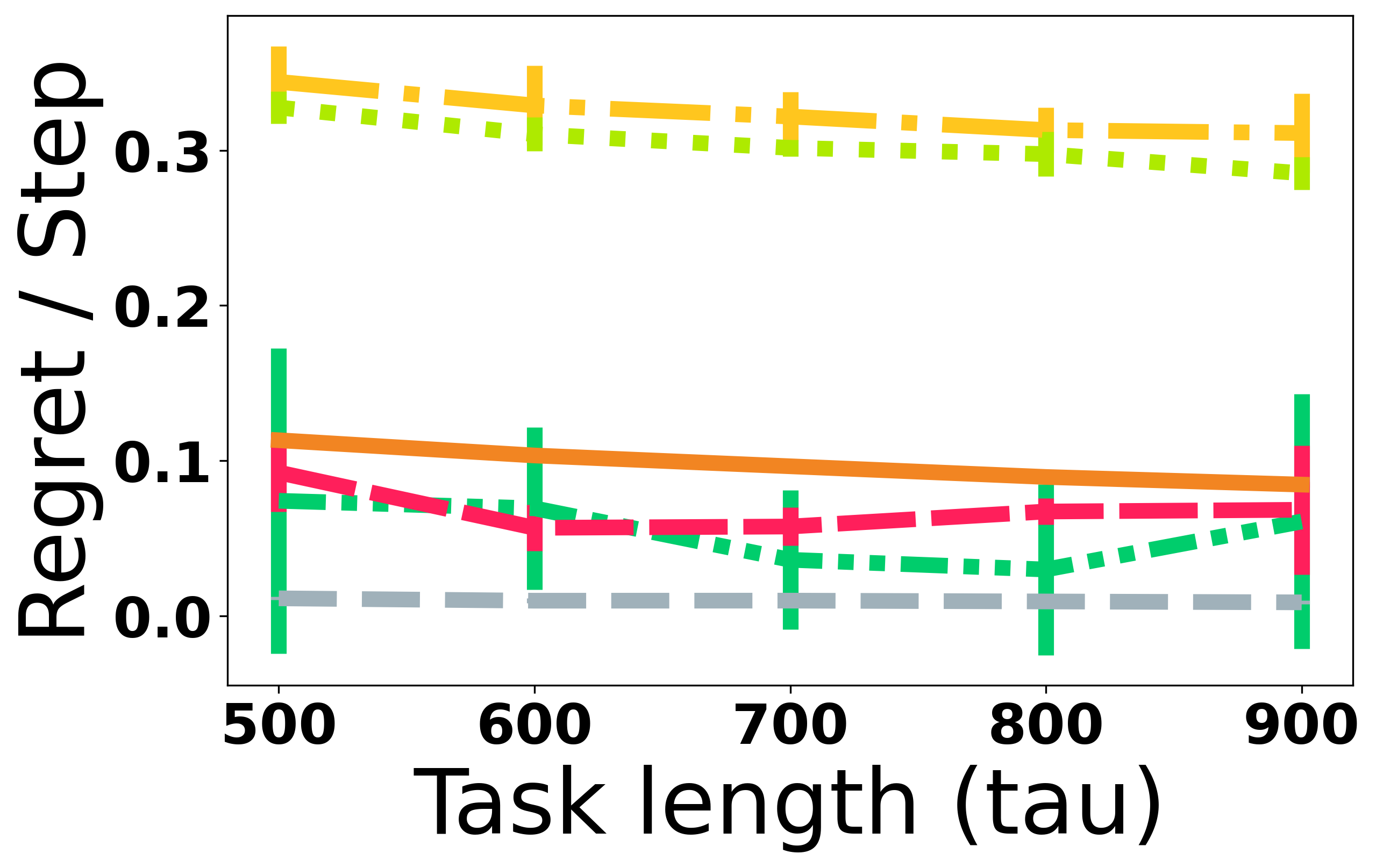}
    \end{minipage}
    \begin{minipage}{.24\textwidth}
        \includegraphics[width=\textwidth]{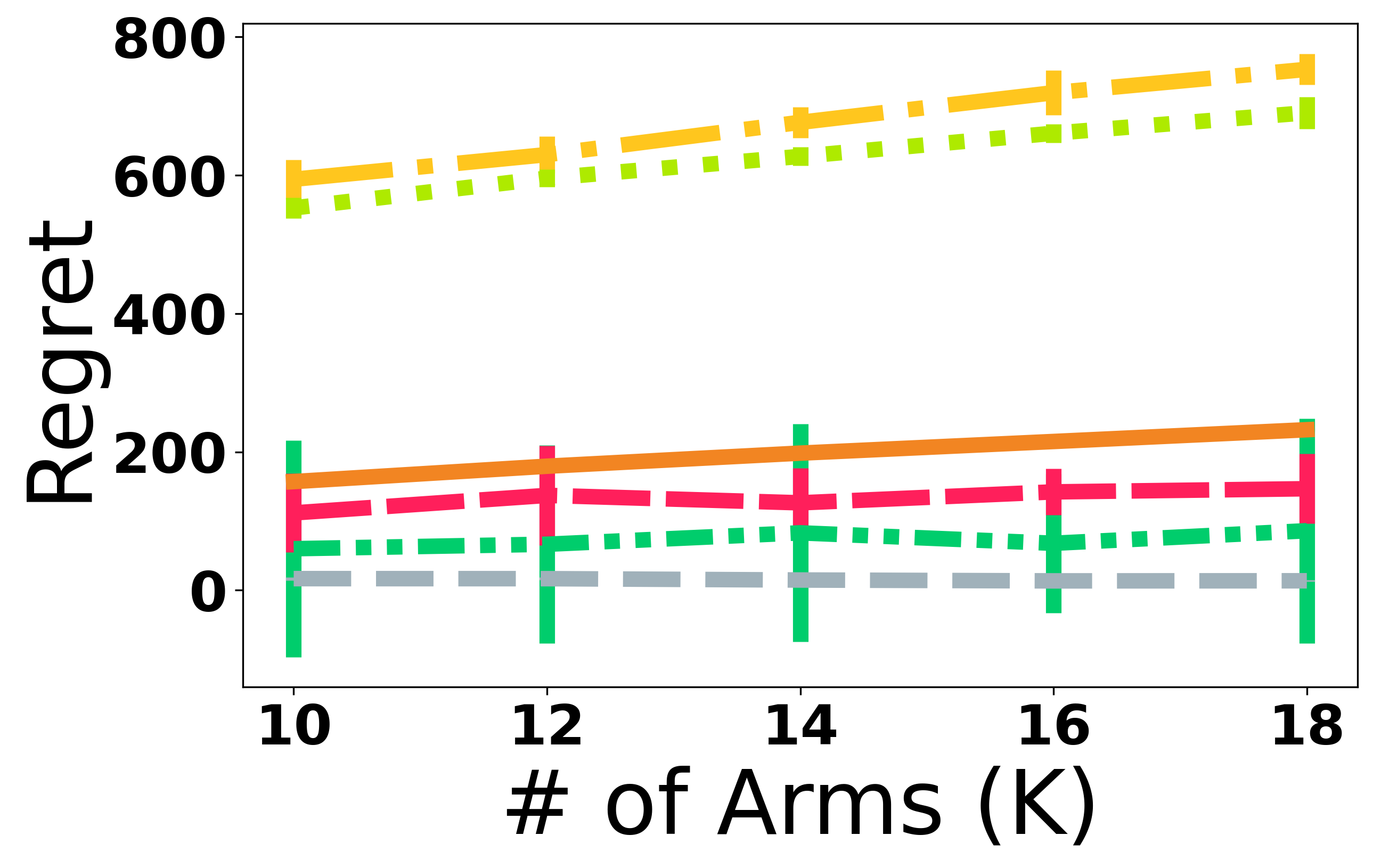}
    \end{minipage}
    \begin{minipage}{.24\textwidth}
    \includegraphics[width=\textwidth]{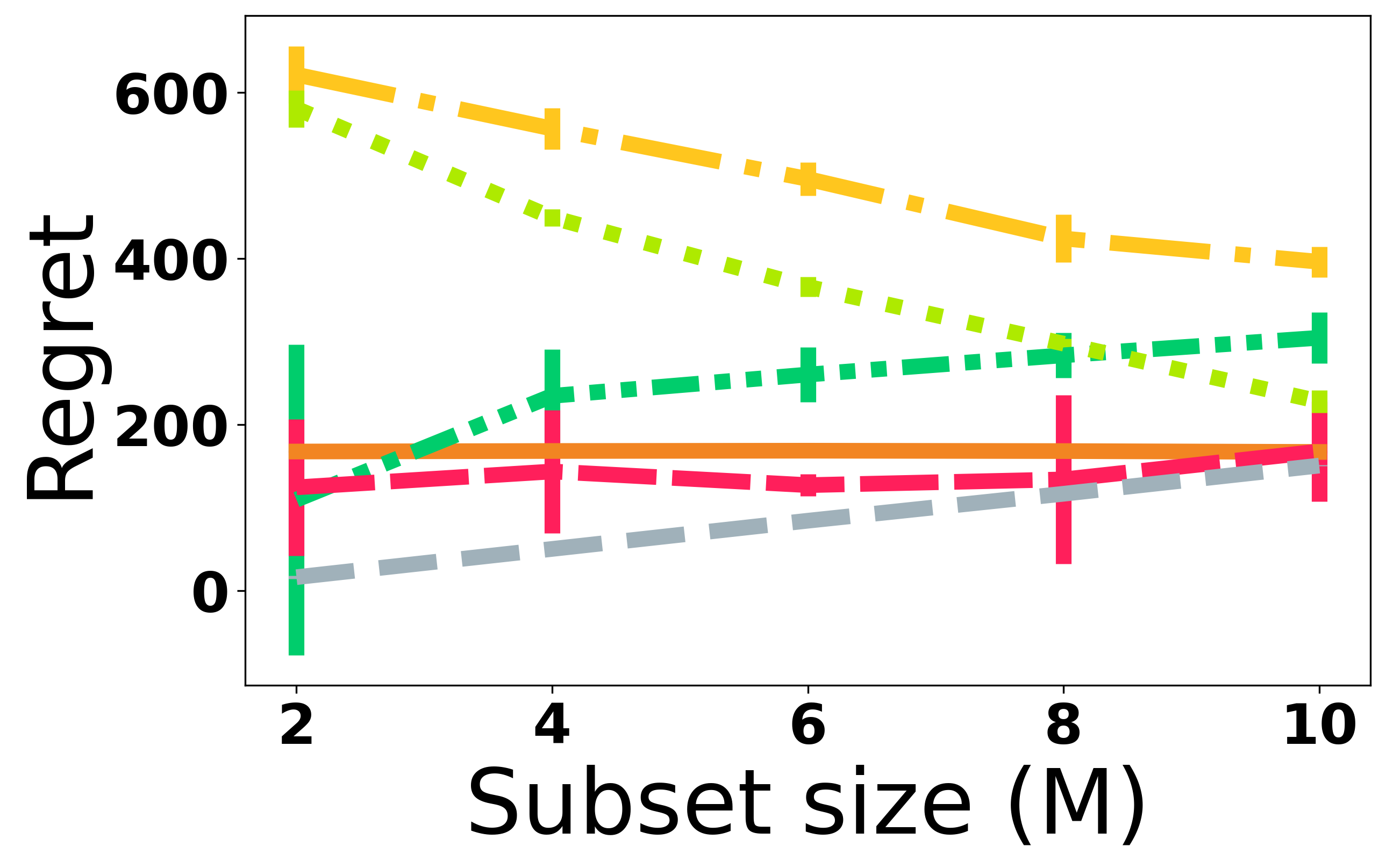}
    \end{minipage}
    \includegraphics[width=\textwidth]{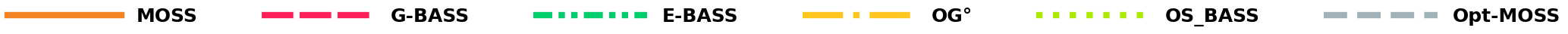}
    \vspace{-0.8cm}
    \caption{\pmml's performance in the oblivious adversarial setting with \cref{assumption:identification}.
    Default setting: ($N$, $\tau$, $K, M$) = (400, 2000, 11, 2). 
    \pmml outperforms other algorithms. 
    Left to Right: Regret as a function of $N$, $\tau$, $K, M$.
    }
    \label{fig:E-BASS}
\end{figure}

\begin{figure}[htb!]
    \centering
    \begin{minipage}{.24\textwidth}
        \includegraphics[width=\textwidth]{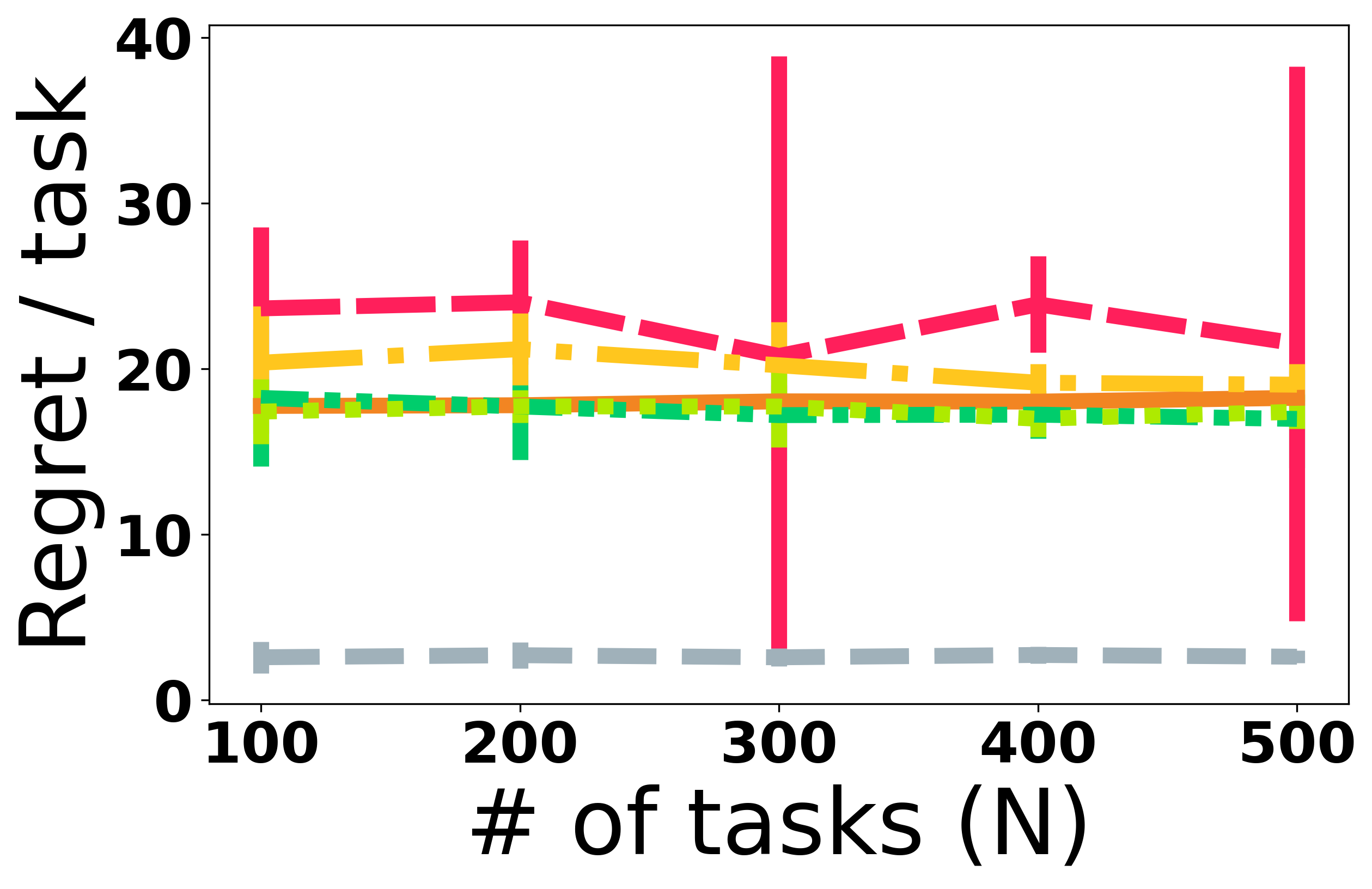}
    \end{minipage}
    \begin{minipage}{.24\textwidth}
        \includegraphics[width=\textwidth]{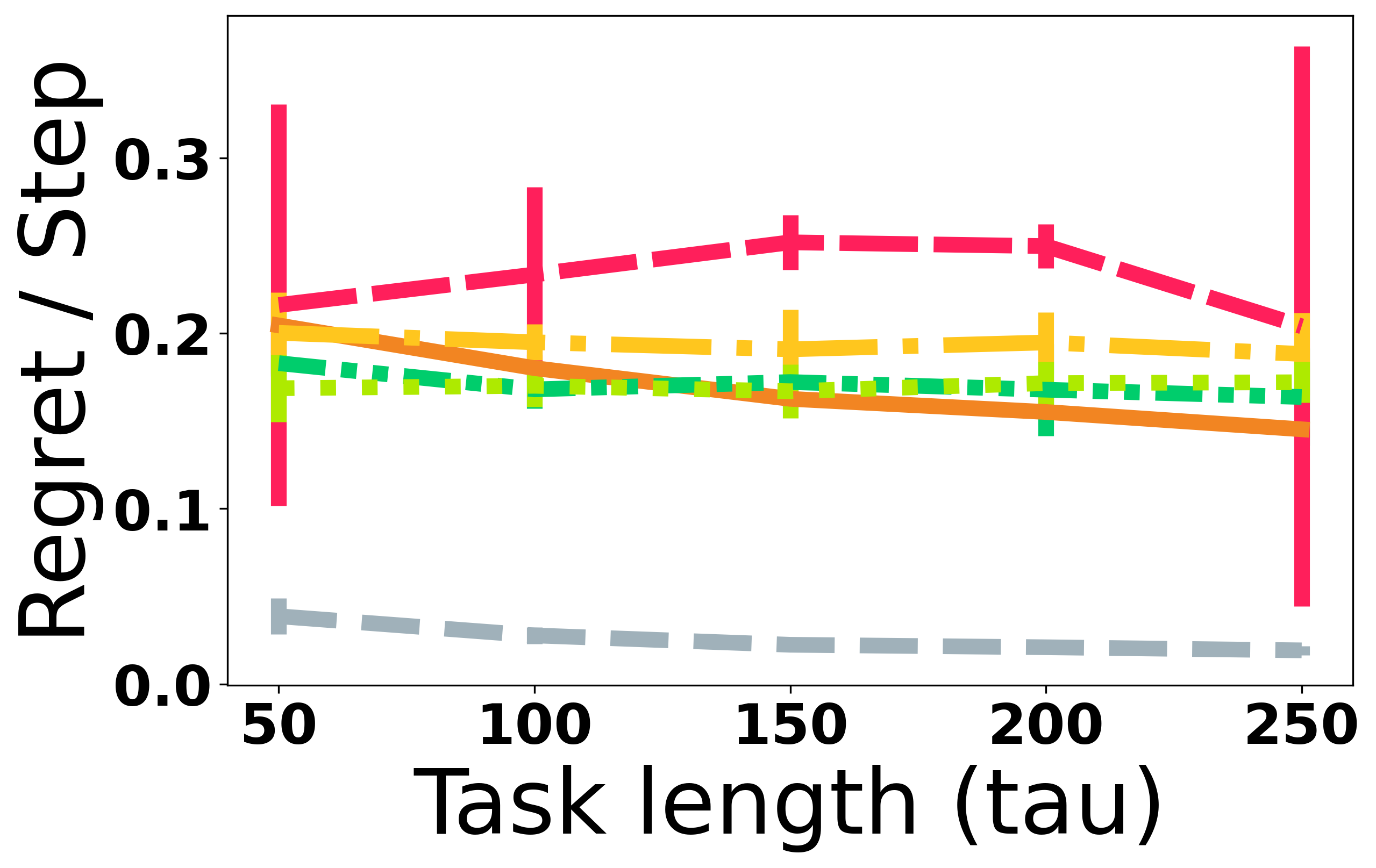}
    \end{minipage}
    \begin{minipage}{.24\textwidth}
        \includegraphics[width=\textwidth]{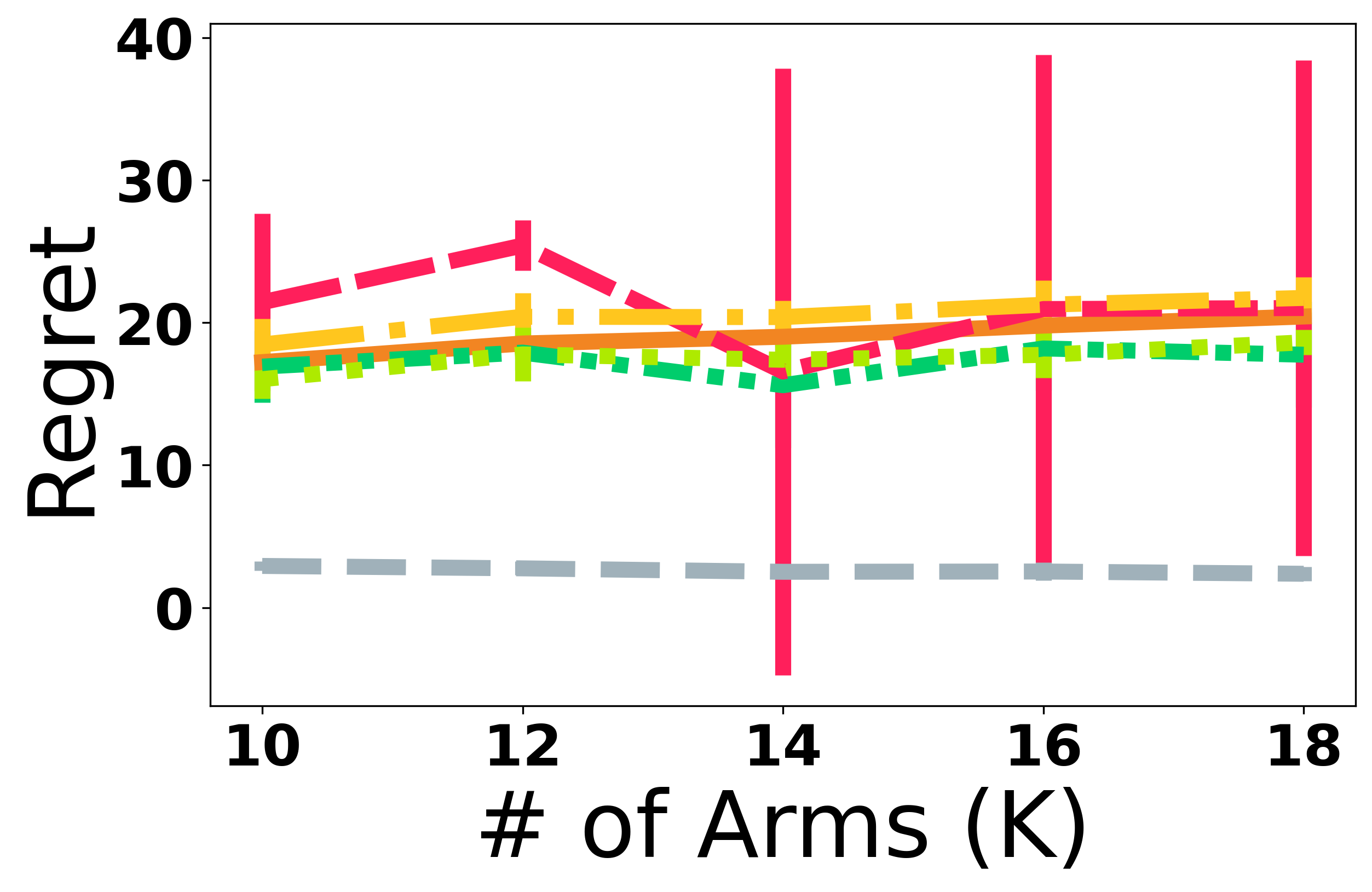}
    \end{minipage}
    \begin{minipage}{.24\textwidth}
    \includegraphics[width=\textwidth]{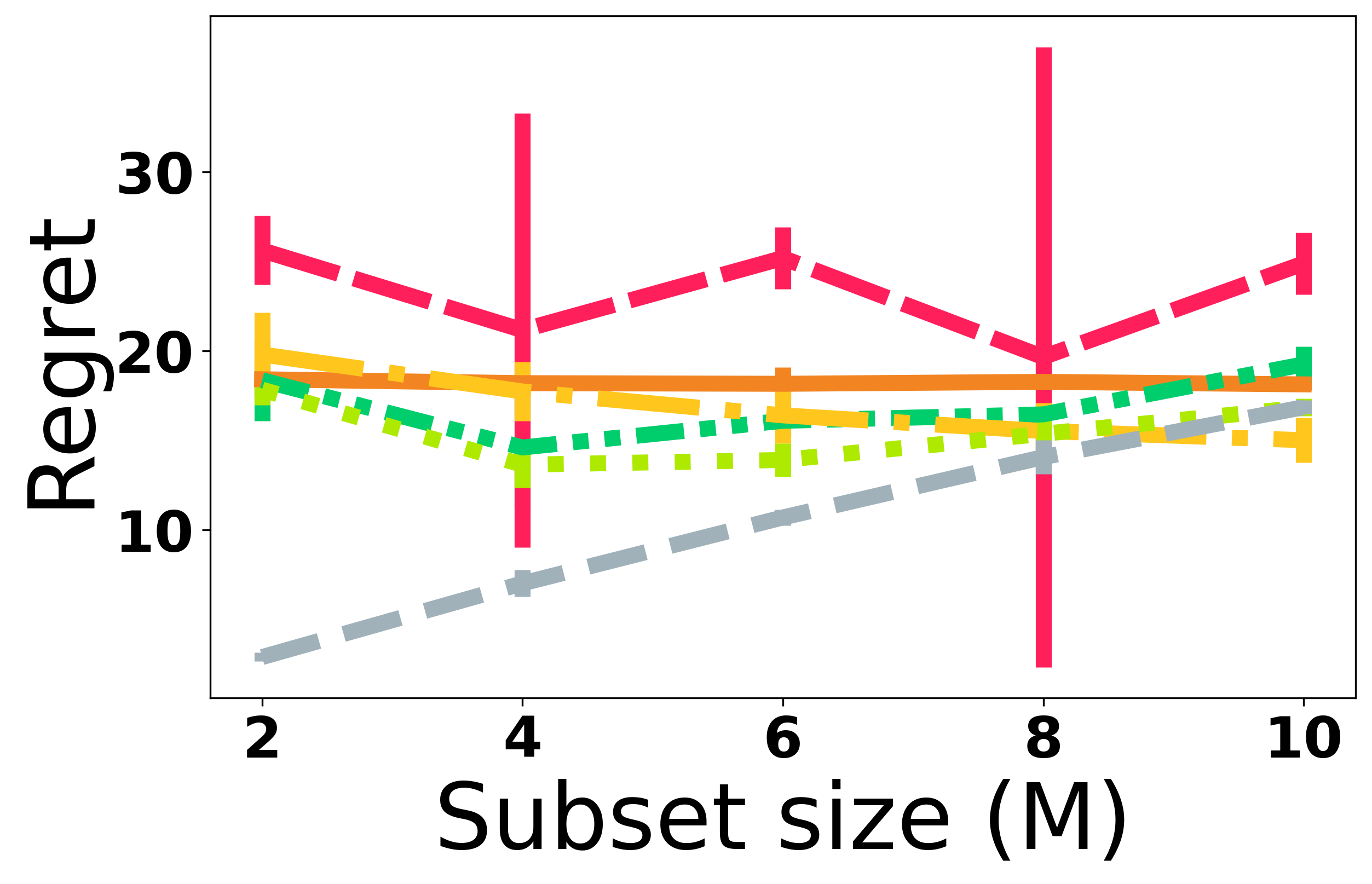}
    \end{minipage}
    \includegraphics[width=\textwidth]{img/legend6.png}
    \vspace{-0.8cm}
    \caption{\pmml's performance in the oblivious adversarial setting \NoGapSmallt. 
    Default setting: ($N$, $\tau$, $K, M$) = (400, 100, 11, 2). 
    \pmml and \BOG win all the settings, while \MOSS is competitive.
    Left to Right: Regret as a function of $N$, $\tau$, $K, M$.
    }
    \label{fig:E-BASS_no_gap}
\end{figure}

\cref{fig:Non-oblivious,fig:Non-oblivious_no_gap} show the experimental results with a non-oblivious adversary. We observe similar trends as in the previous experiments.

\begin{figure}[htb!]
    \centering
    \begin{minipage}{.24\textwidth}
        \includegraphics[width=\textwidth]{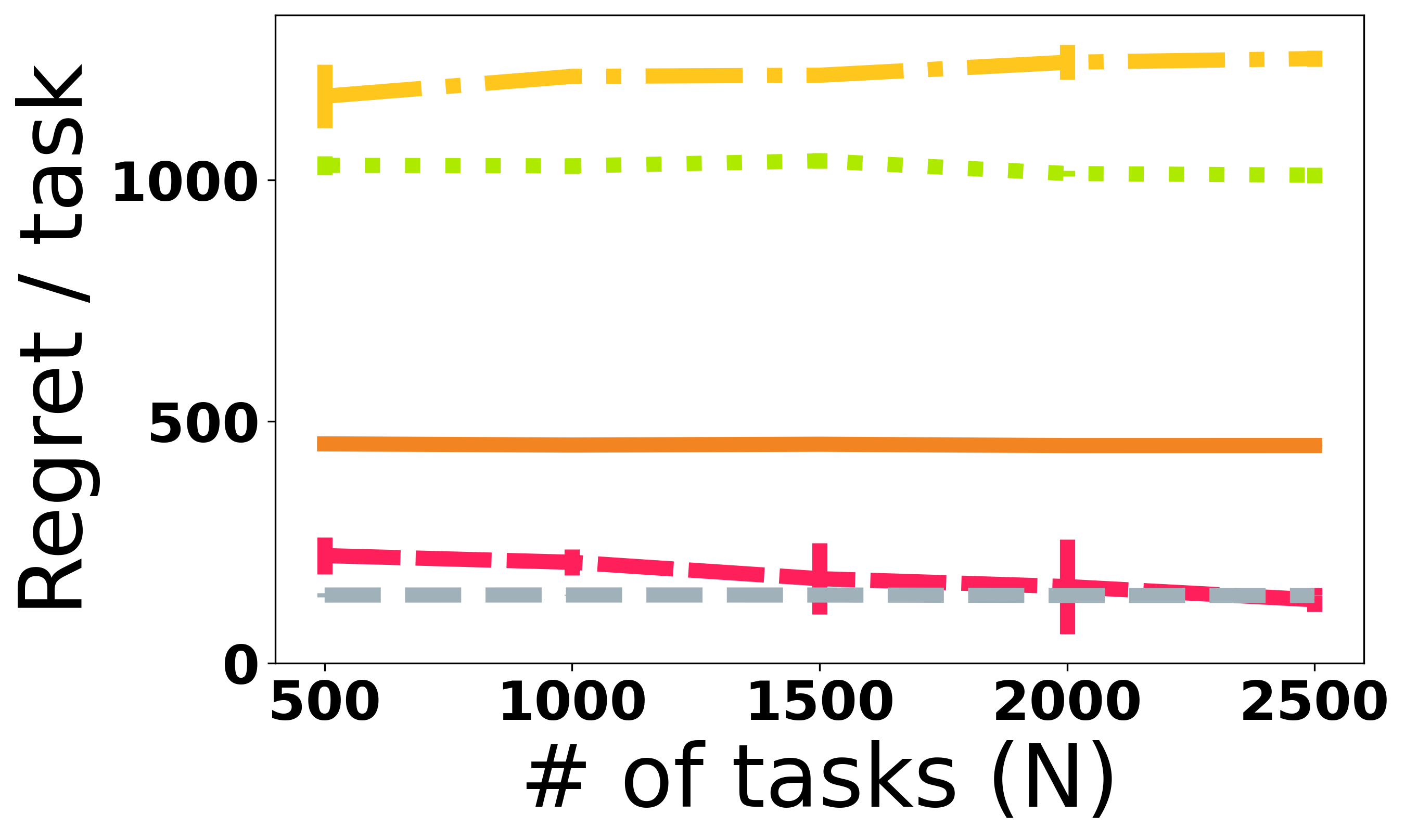}
    \end{minipage}
    \begin{minipage}{.24\textwidth}
        \includegraphics[width=\textwidth]{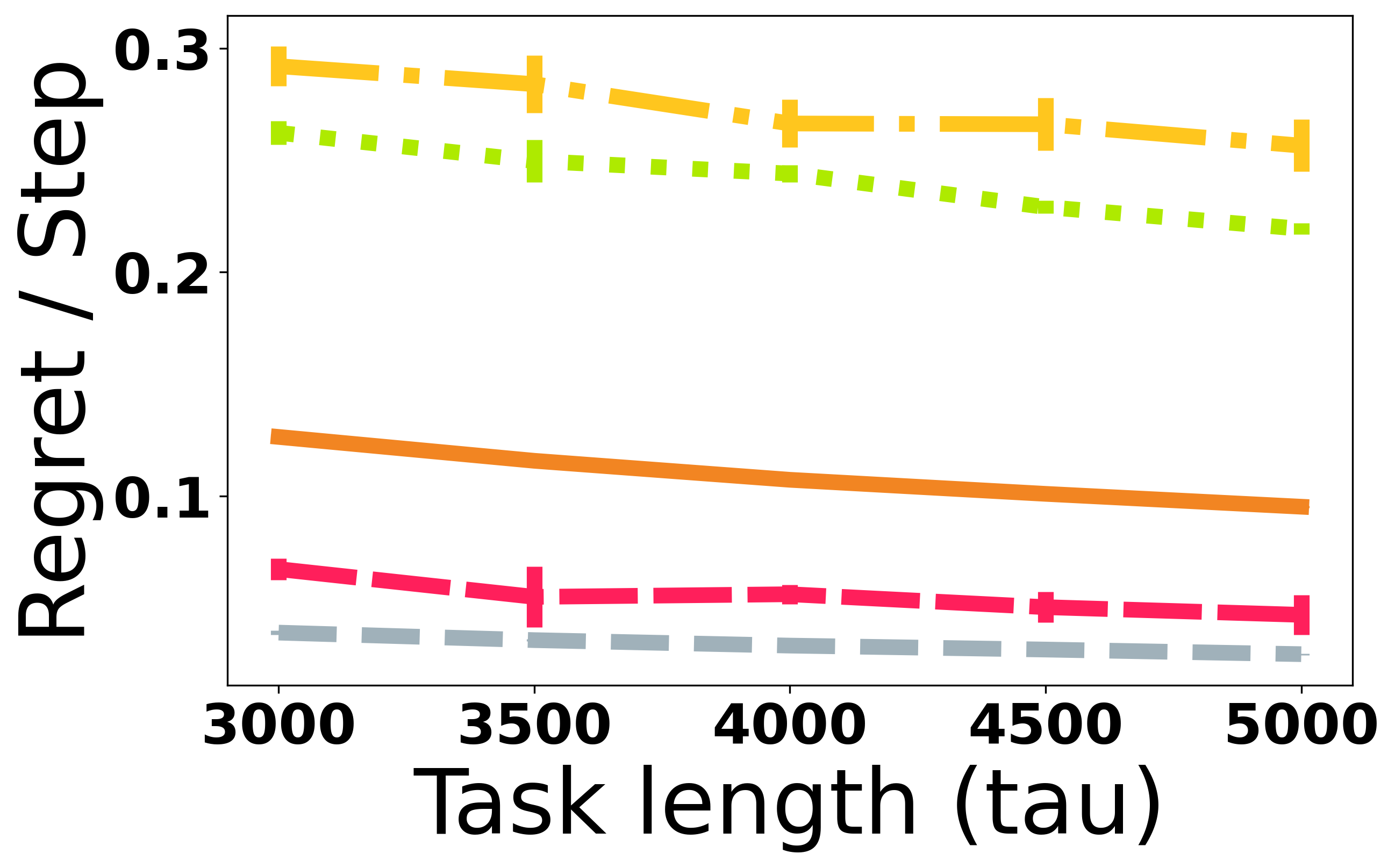}
    \end{minipage}
    \begin{minipage}{.24\textwidth}
        \includegraphics[width=\textwidth]{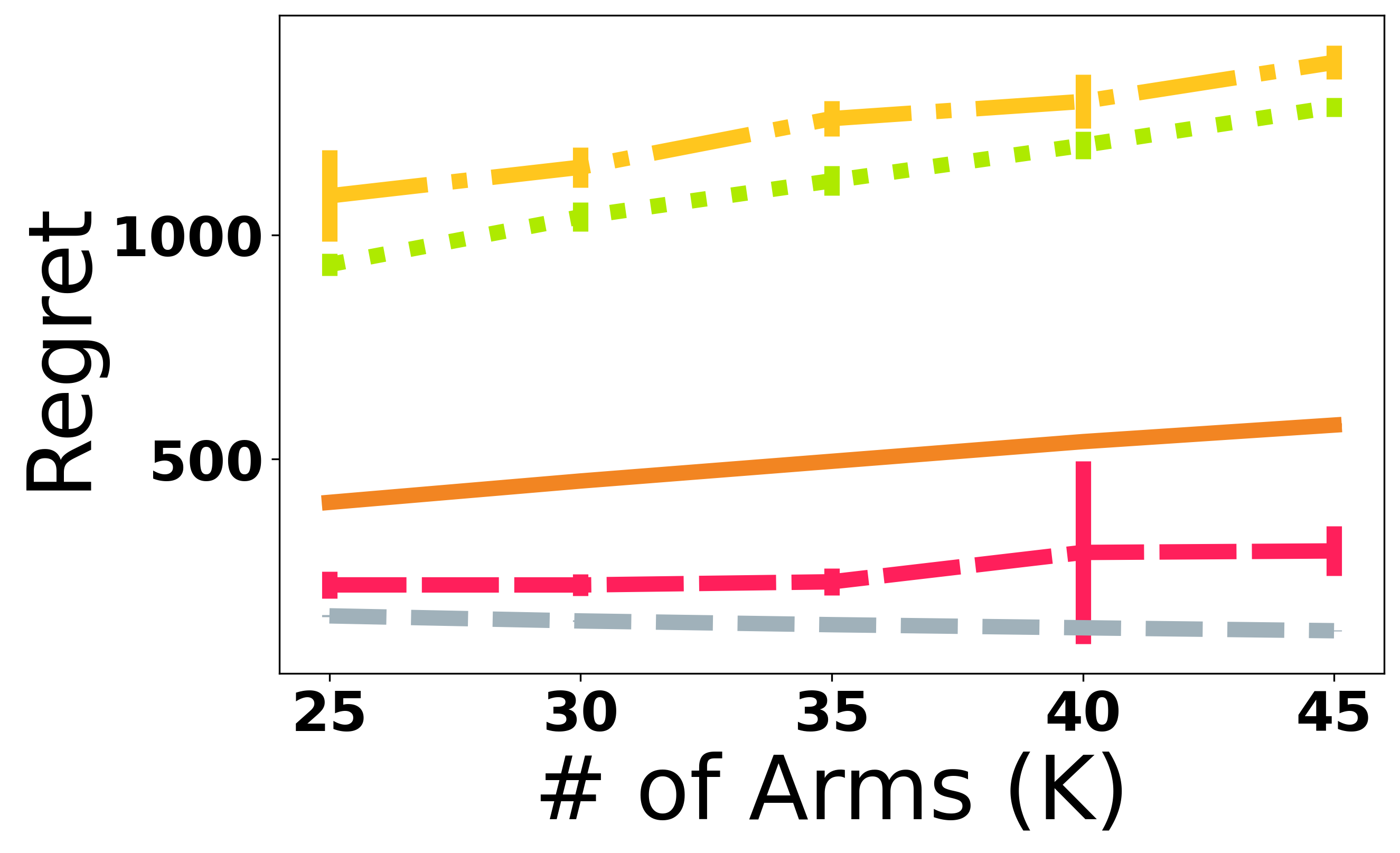}
    \end{minipage}
    \begin{minipage}{.24\textwidth}
    \includegraphics[width=\textwidth]{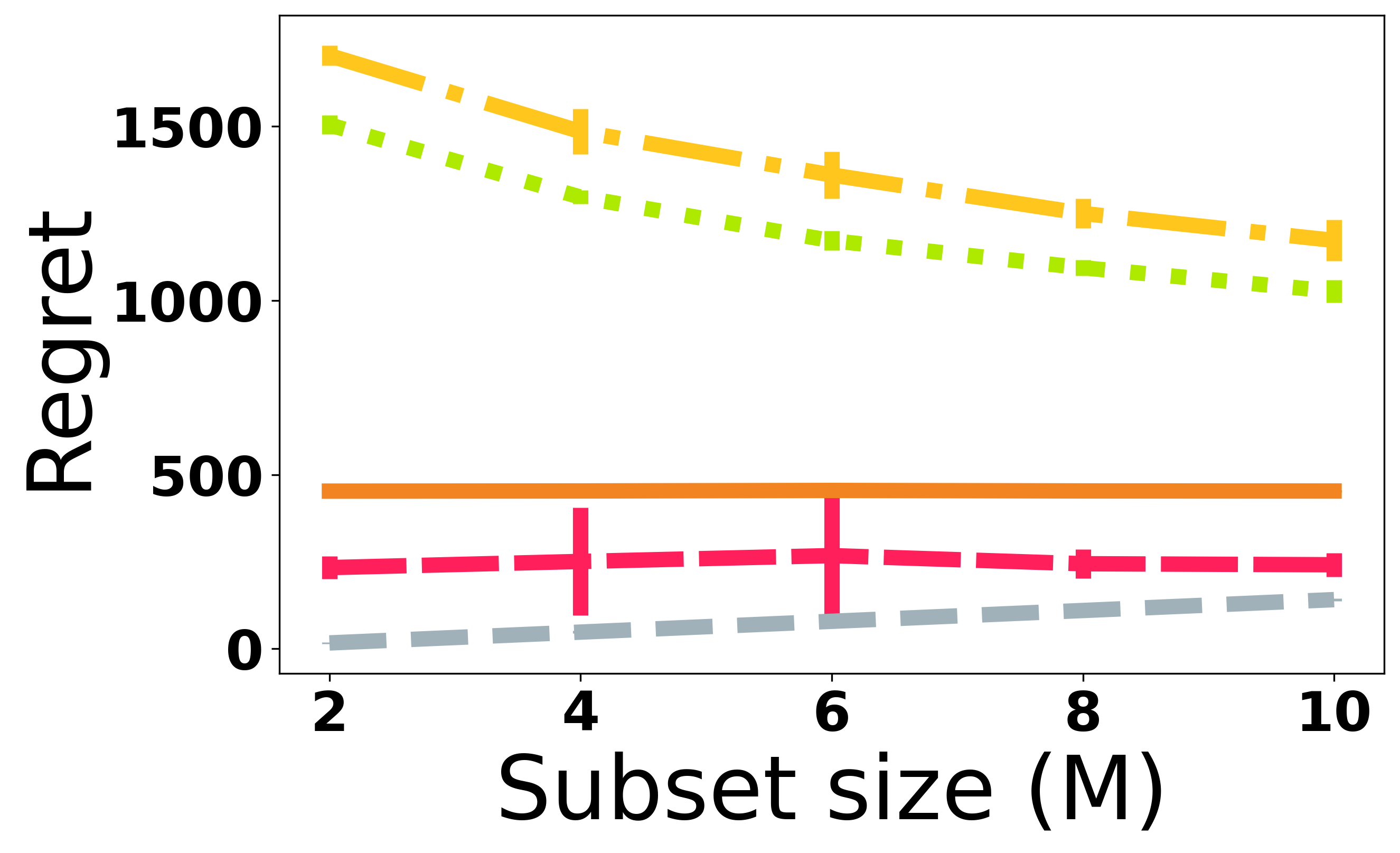}
    \end{minipage}
    \includegraphics[width=\textwidth]{img/legend5.png}
    \vspace{-0.8cm}
    \caption{The non-oblivious adversarial setting, where \cref{assumption:identification} holds.
    Default setting: $\setupOne$. 
    \greedy is near-optimal on all tasks. 
    Left to Right: Regret as a function of $N$, $\tau$, $K, M$.
    }
    \label{fig:Non-oblivious}
\end{figure}

\begin{figure}[htb!]
    \centering
    \begin{minipage}{.24\textwidth}
        \includegraphics[width=\textwidth]{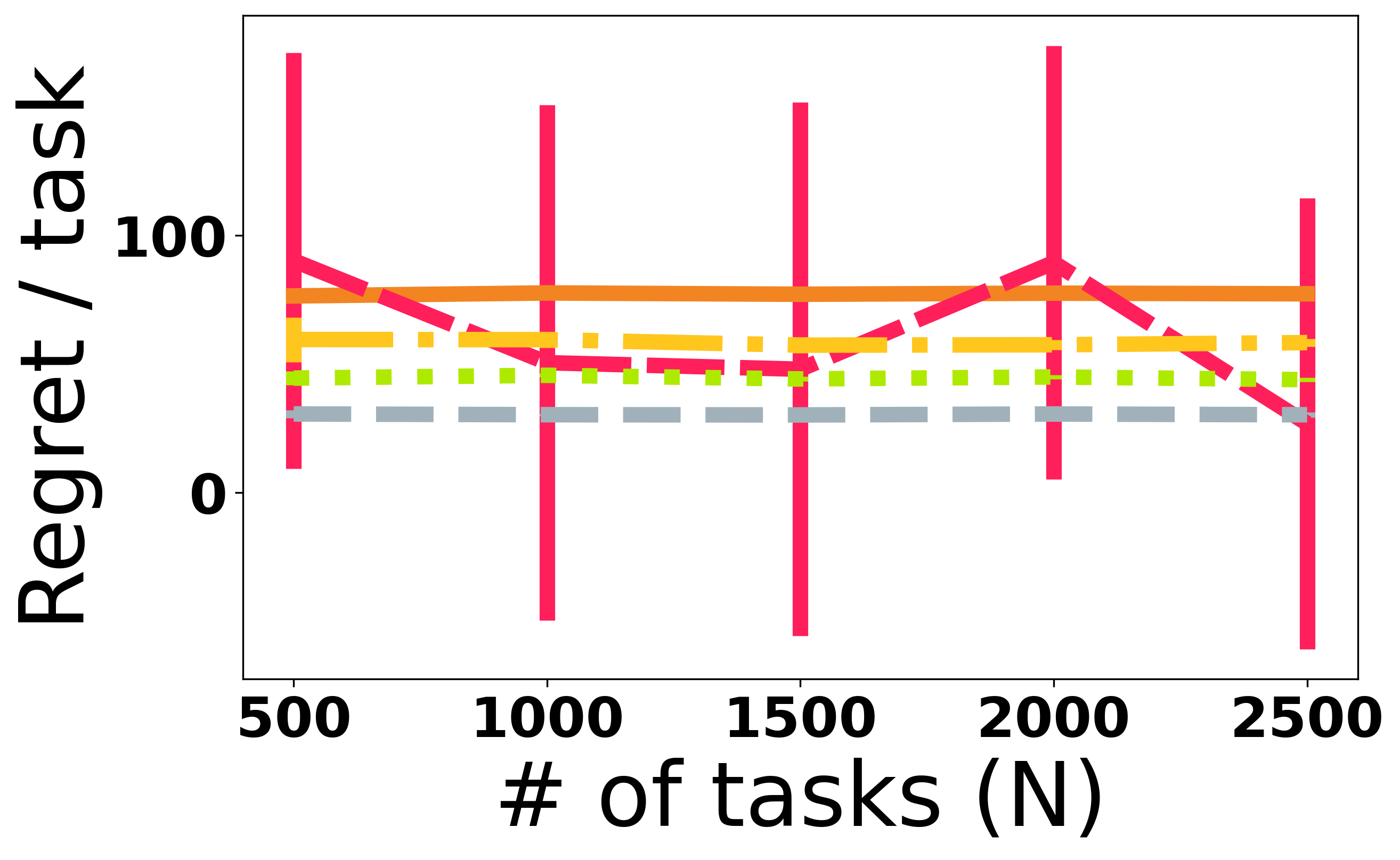}
    \end{minipage}
    \begin{minipage}{.24\textwidth}
        \includegraphics[width=\textwidth]{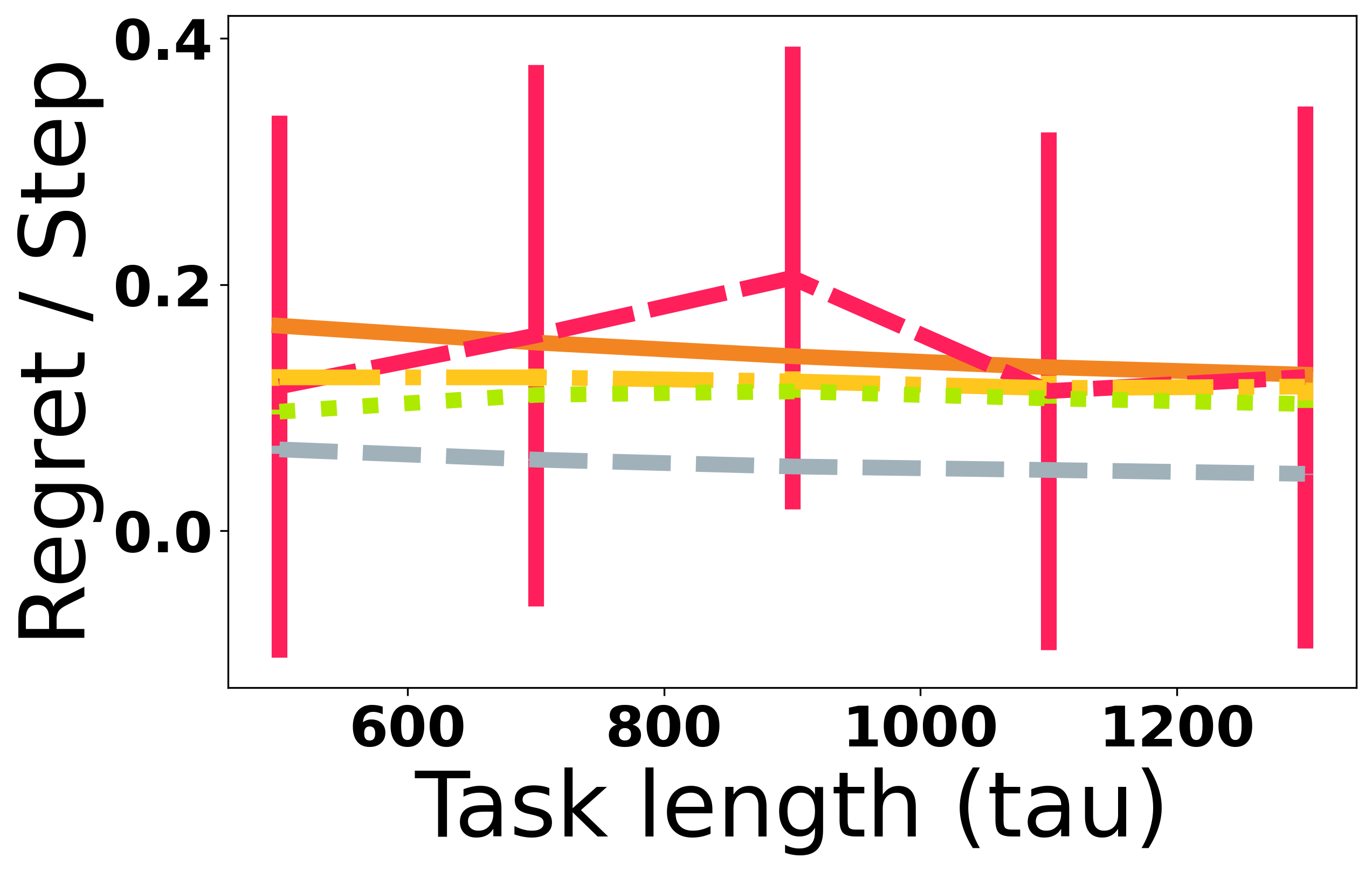}
    \end{minipage}
    \begin{minipage}{.24\textwidth}
        \includegraphics[width=\textwidth]{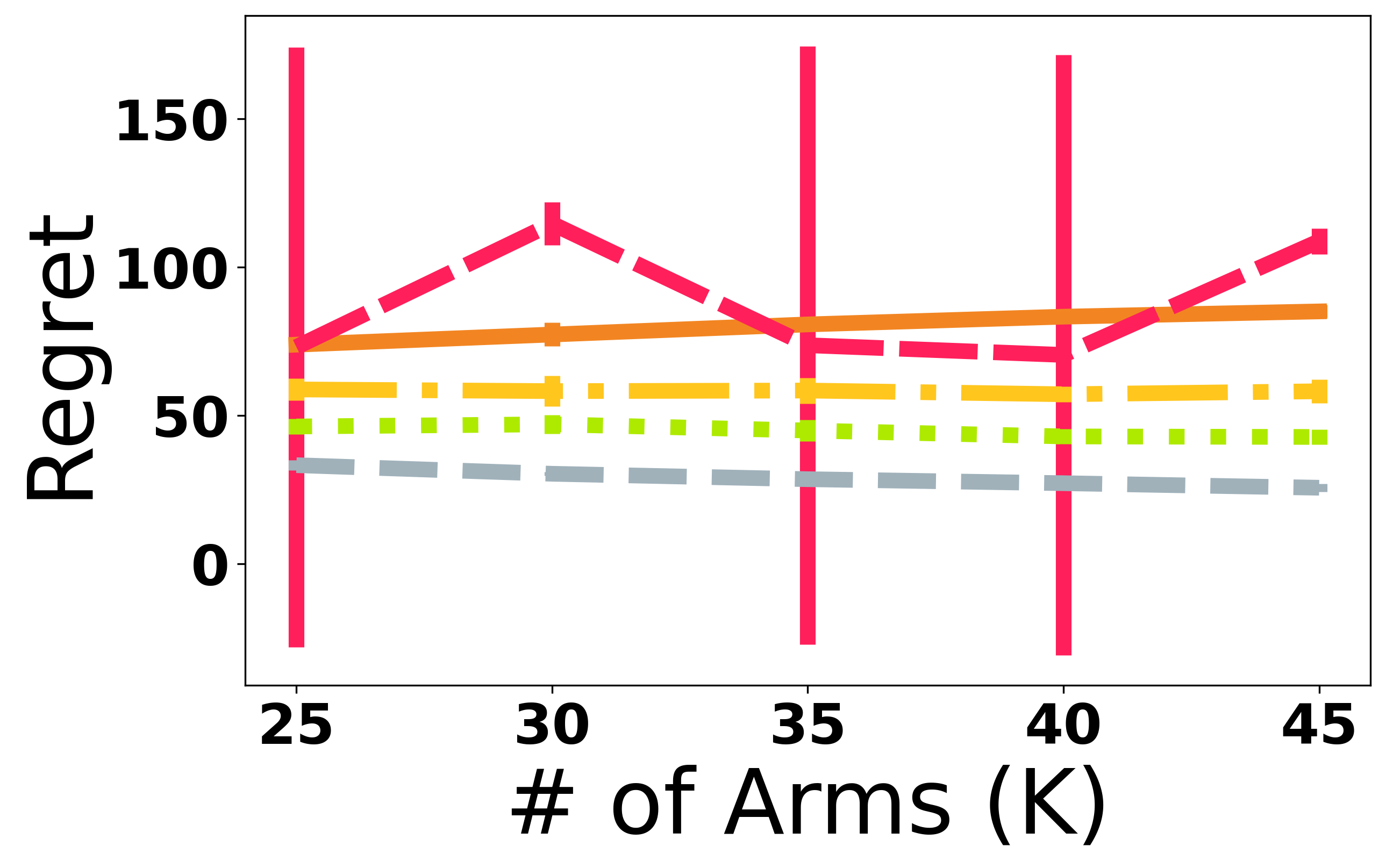}
    \end{minipage}
    \begin{minipage}{.24\textwidth}
    \includegraphics[width=\textwidth]{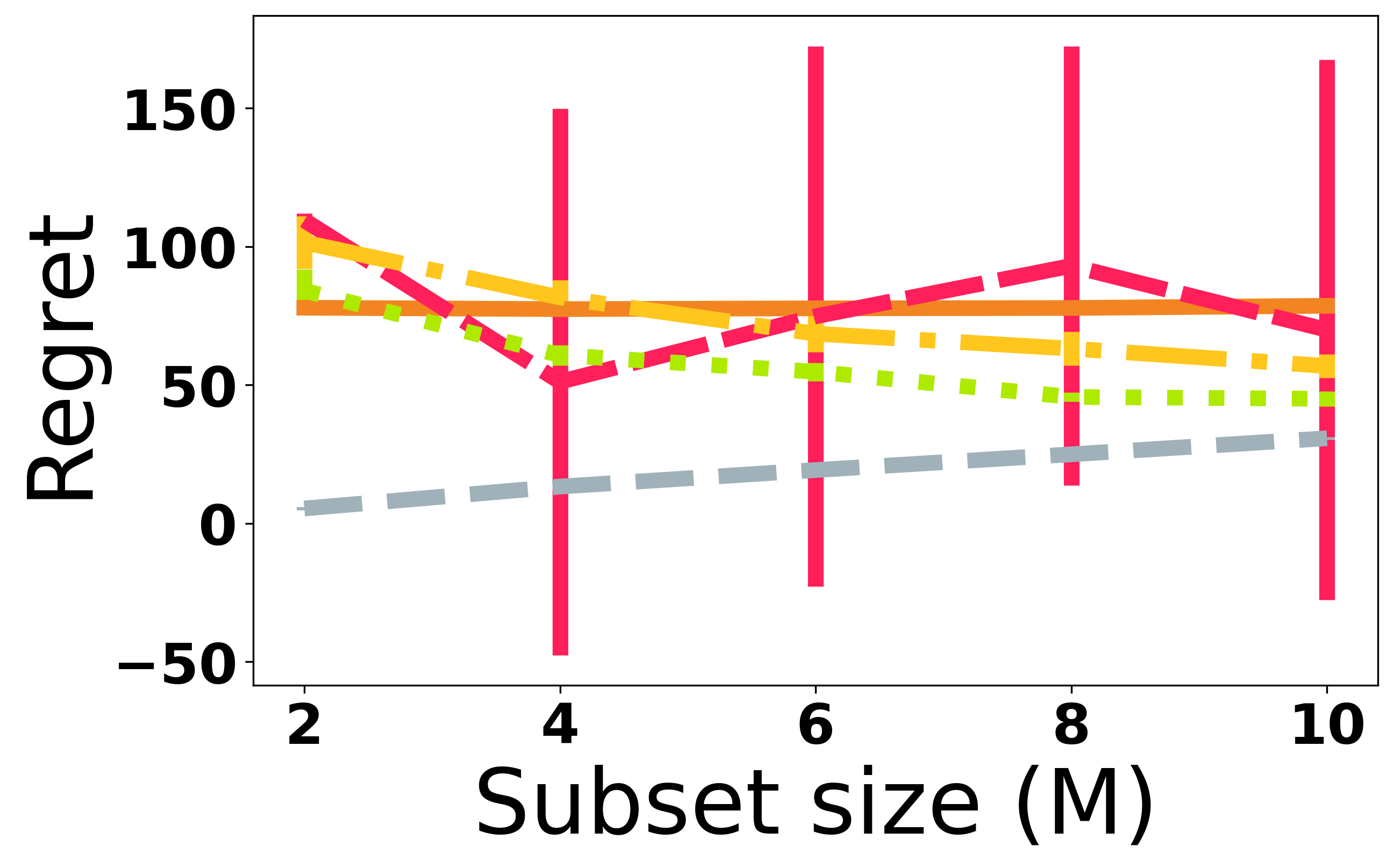}
    \end{minipage}
    \includegraphics[width=\textwidth]{img/legend5.png}
    \vspace{-0.8cm}
    \caption{Non-oblivious adversarial setting \NoGapSmallt.
    Default setting: $\setupTwo$. 
    \BOG mostly outperforms the other method. \greedy has a high variance as PE fails in this experiment.  
    Left to Right: Regret as a function of $N$, $\tau$, $K, M$.
    }
    \label{fig:Non-oblivious_no_gap}
\end{figure}

The results for the stochastic setting are shown in \cref{fig:Stoch,fig:Stoch_no_gap}. In \cref{fig:Stoch} it seems that \greedy performs the best while \MOSS is closer to the oracle \OptMoss. However, in \cref{fig:Stoch_no_gap} \BOG outperforms \greedy and the other algorithms and gets closer to the oracle baseline. We can see in the stochastic setting the variance is higher than the adversarial setting.

\begin{figure}[htb!]
    \centering
    \begin{minipage}{.24\textwidth}
        \includegraphics[width=\textwidth]{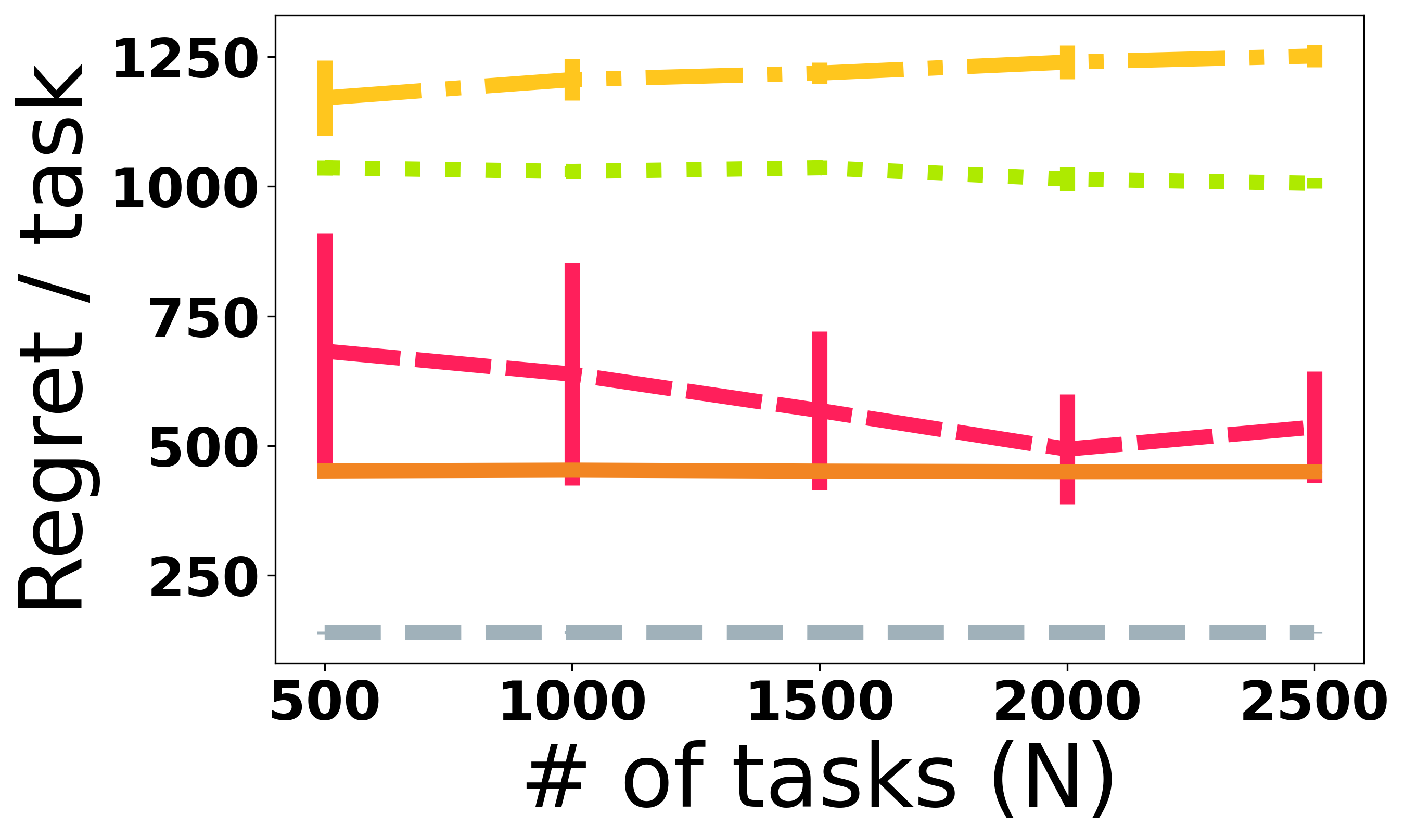}
    \end{minipage}
    \begin{minipage}{.24\textwidth}
        \includegraphics[width=\textwidth]{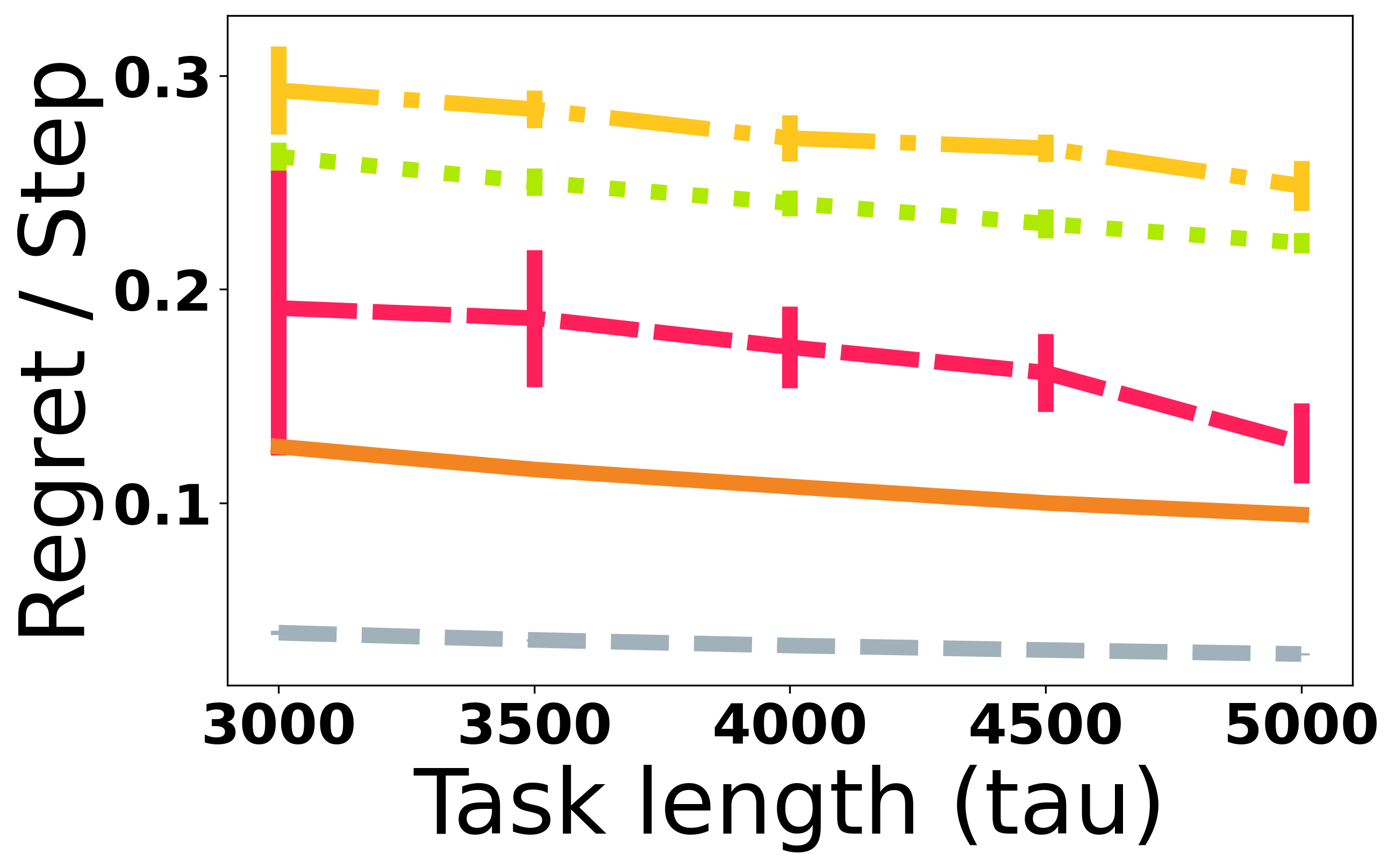}
    \end{minipage}
    \begin{minipage}{.24\textwidth}
        \includegraphics[width=\textwidth]{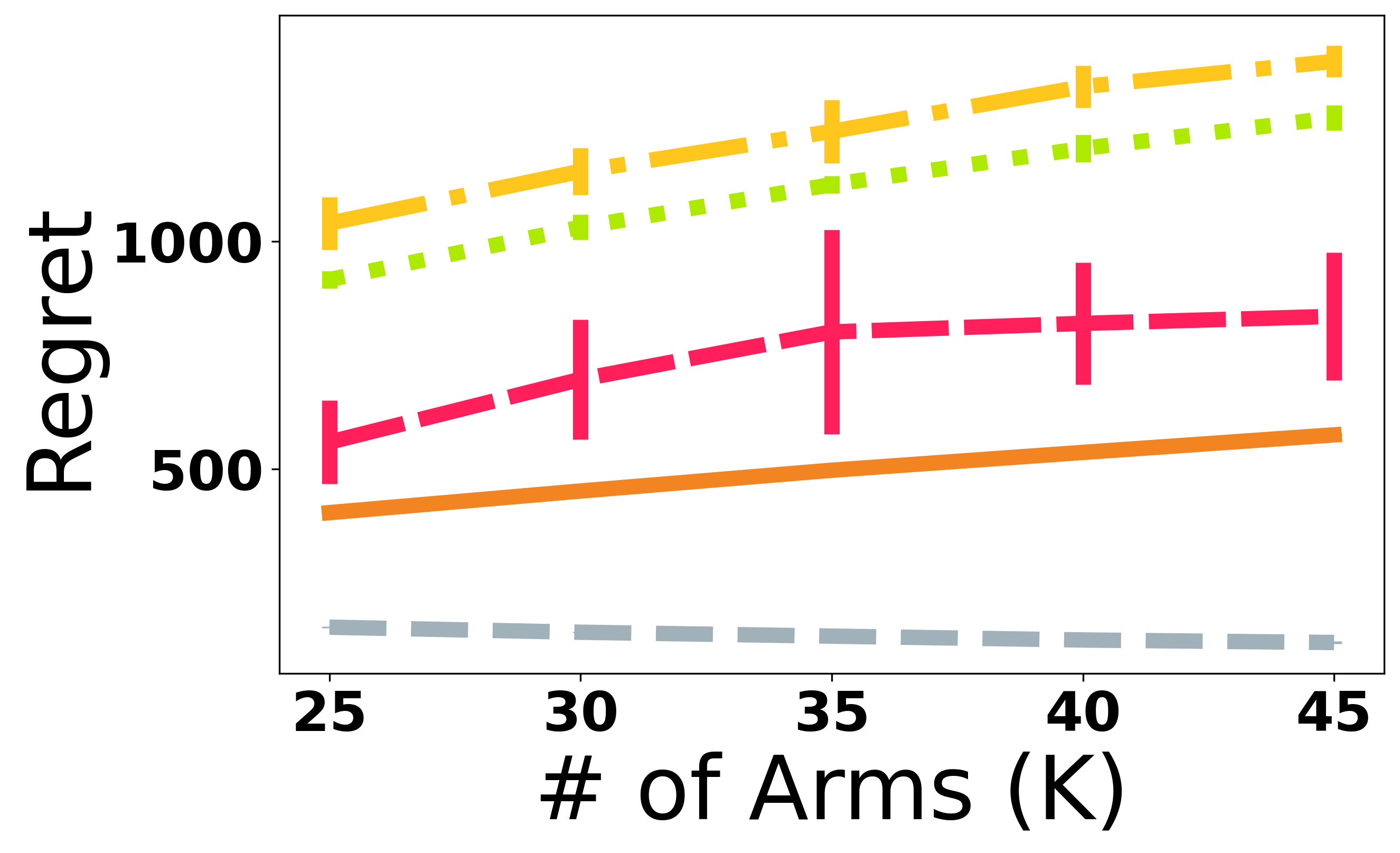}
    \end{minipage}
    \begin{minipage}{.24\textwidth}
    \includegraphics[width=\textwidth]{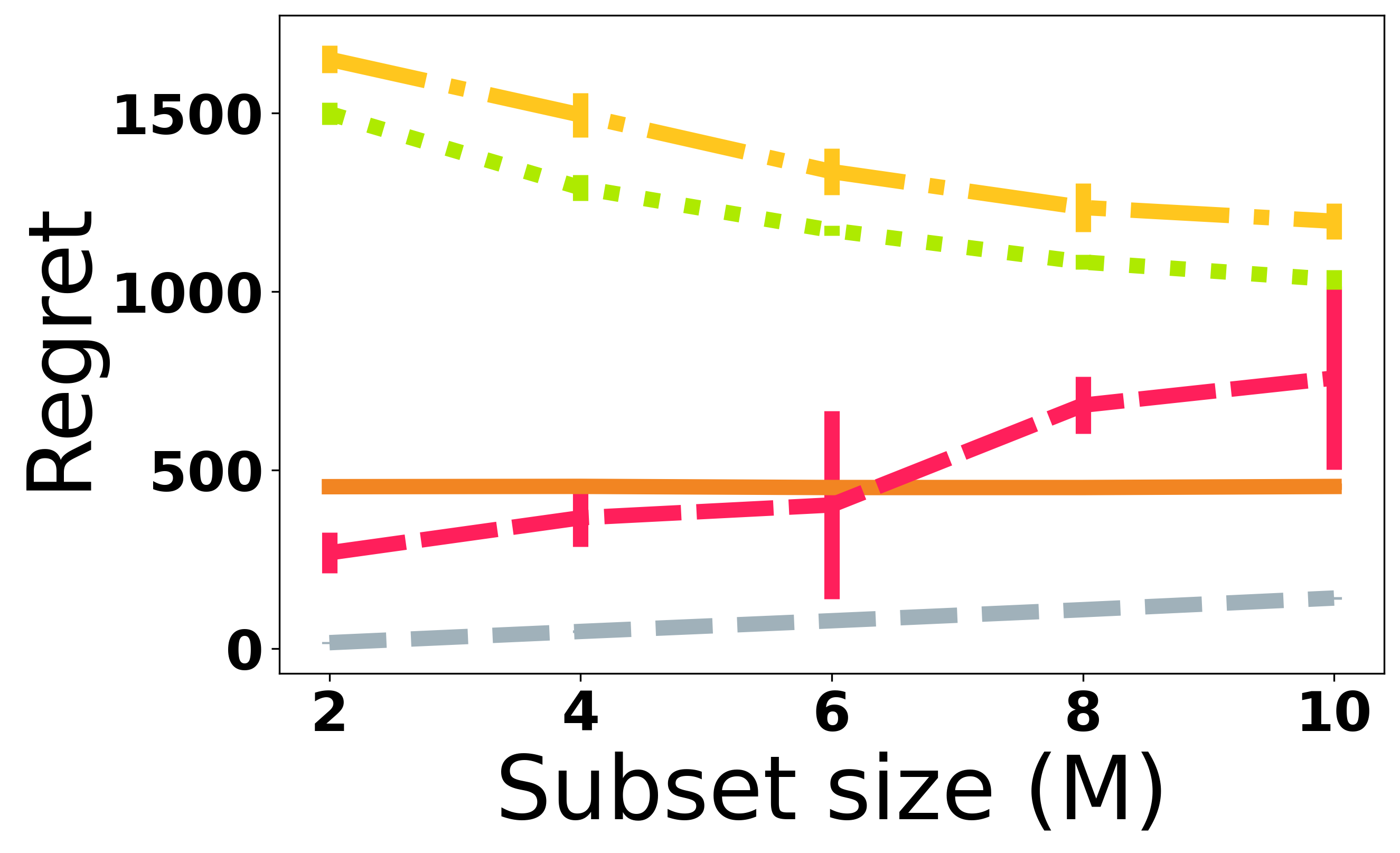}
    \end{minipage}
    \includegraphics[width=\textwidth]{img/legend5.png}
    \vspace{-0.8cm}
    \caption{Stochastic setting, where \cref{assumption:identification} holds.
    Default setting: $\setupOne$. \greedy and \MOSS have the best performance in all the experiments.
    Left to Right: Regret as a function of $N$, $\tau$, $K, M$.
    }
    \label{fig:Stoch}
\end{figure}

\begin{figure}[htb!]
    \centering
    \begin{minipage}{.24\textwidth}
        \includegraphics[width=\textwidth]{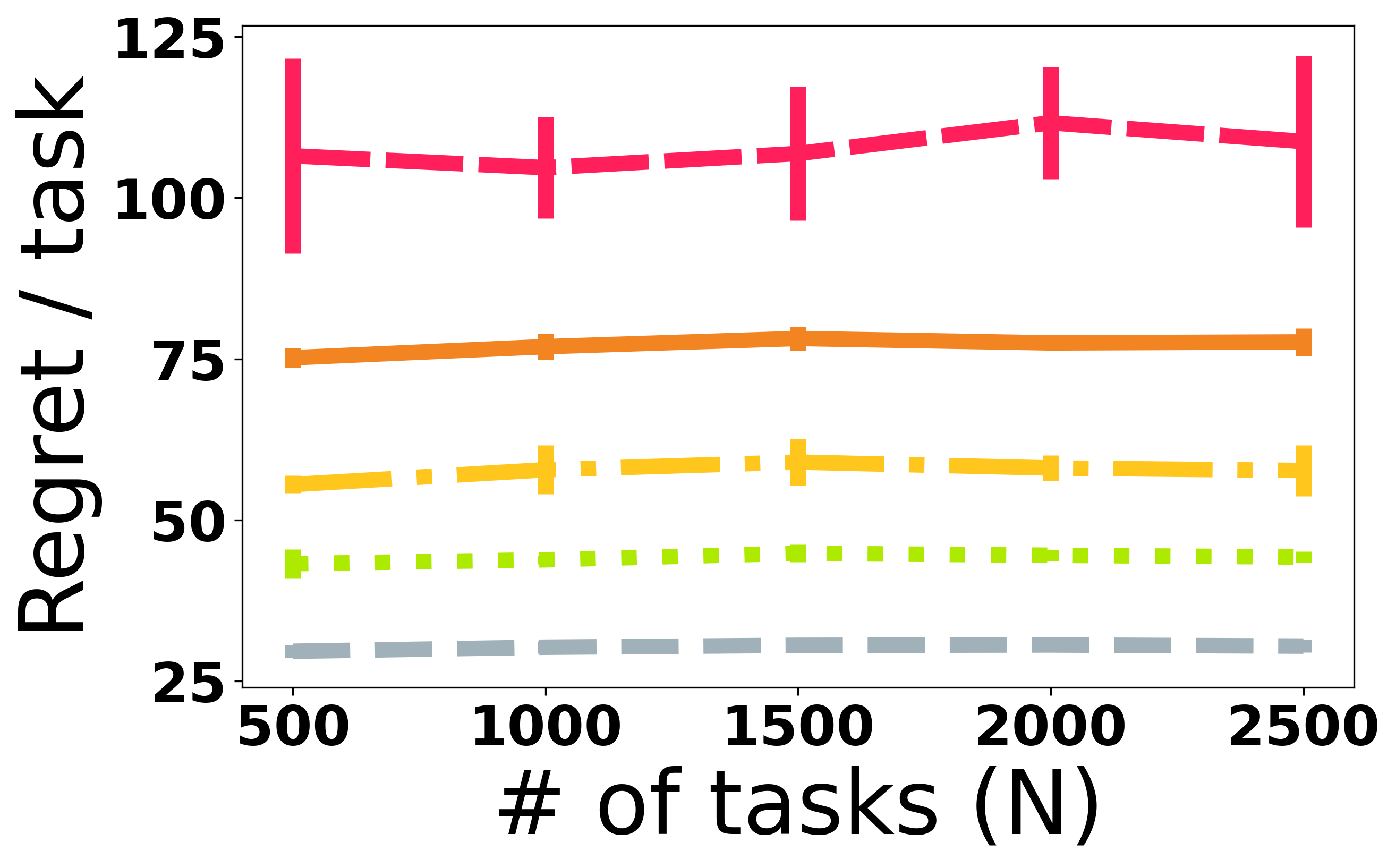}
    \end{minipage}
    \begin{minipage}{.24\textwidth}
        \includegraphics[width=\textwidth]{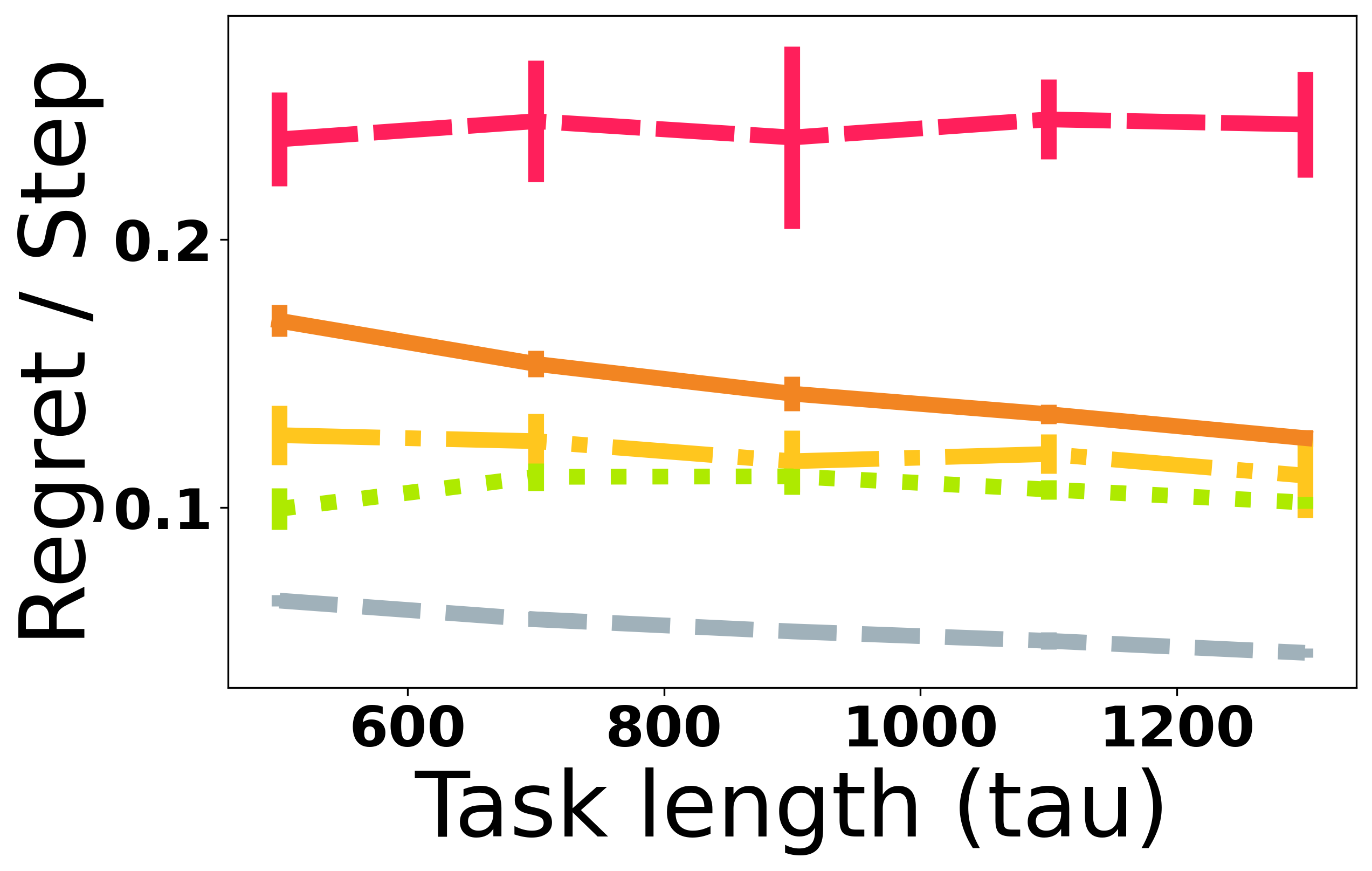}
    \end{minipage}
    \begin{minipage}{.24\textwidth}
        \includegraphics[width=\textwidth]{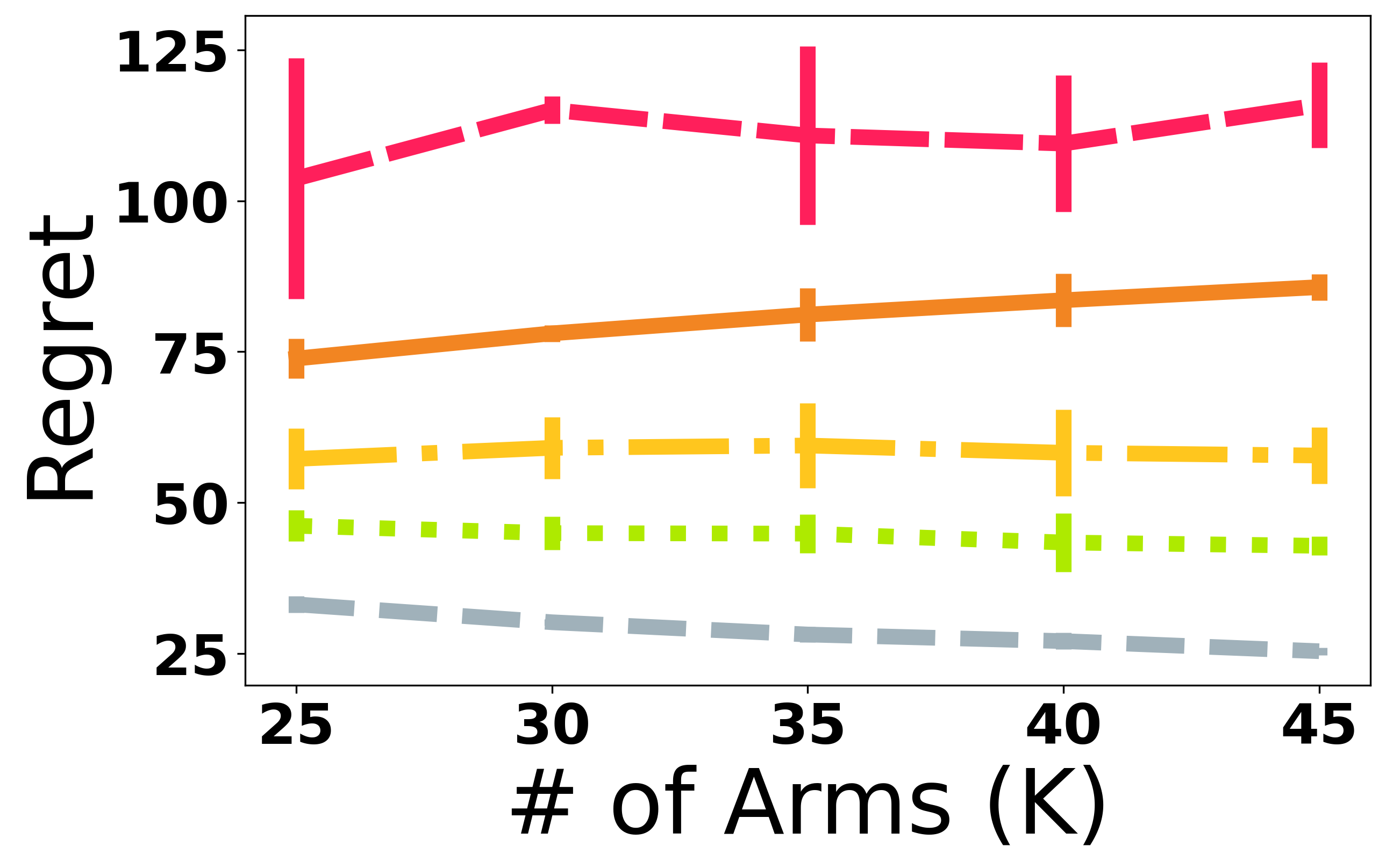}
    \end{minipage}
    \begin{minipage}{.24\textwidth}
    \includegraphics[width=\textwidth]{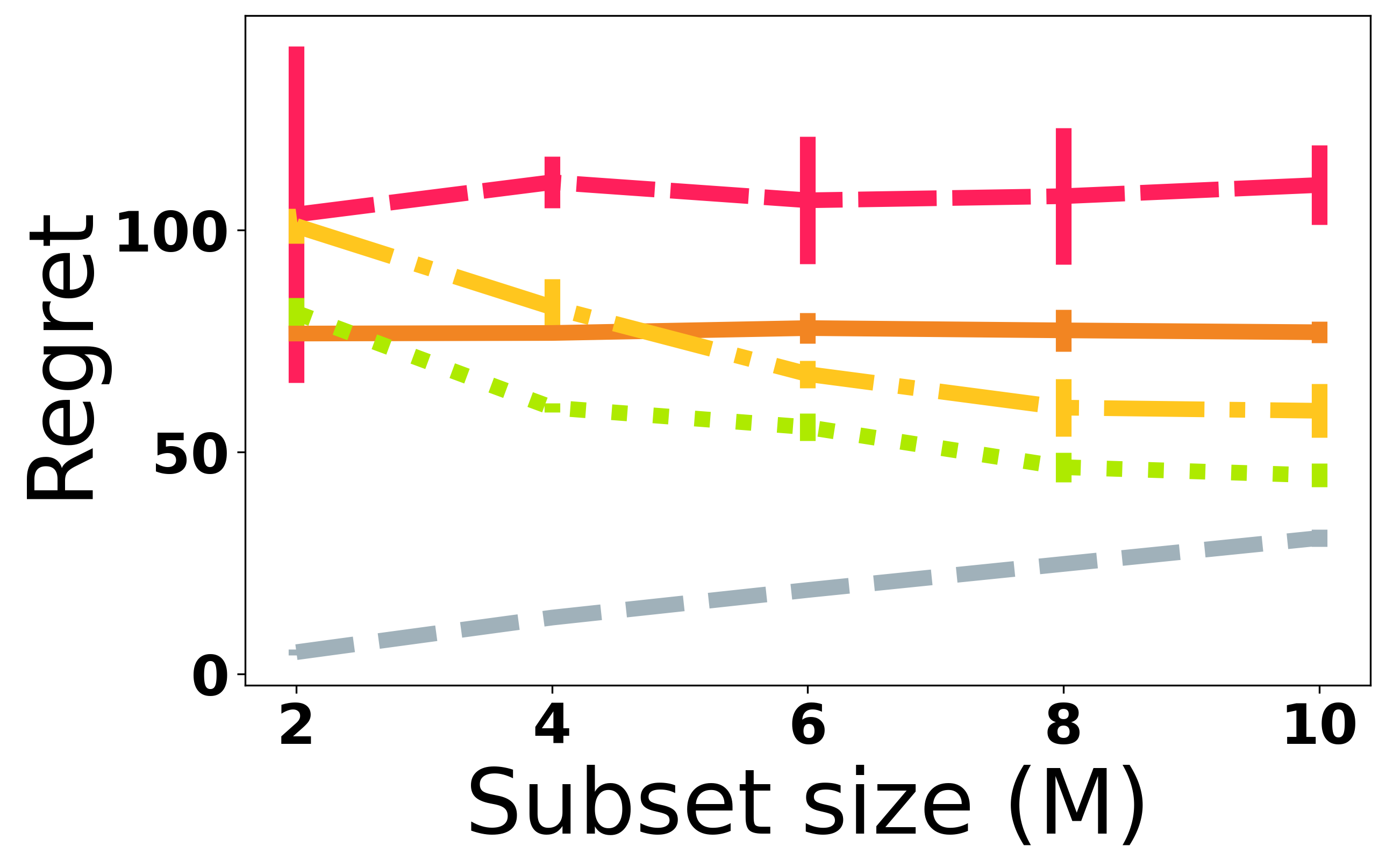}
    \end{minipage}
    \includegraphics[width=\textwidth]{img/legend5.png}
    \vspace{-0.8cm}
    \caption{Stochastic setting \NoGapSmallt.
    Default setting: $\setupTwo$.
    \BOG has the best performance in all experiments.
    Left to Right: Regret as a function of $N$, $\tau$, $K, M$.
    }
    \label{fig:Stoch_no_gap}
\end{figure}

In all the experiments, \BOG outperforms \OGO which confirms the choice of $\gamma$ and $\tau$ in our analysis for \BOG.

\end{document}